%% file: main.tex
\documentclass[conference]{IEEEtran}
\usepackage{times}

\usepackage{listings}
\usepackage{tcolorbox}
\usepackage{xcolor}
\usepackage{xcolor-material}
\usepackage{framed}
\usepackage[colorinlistoftodos]{todonotes}
\usepackage{amsmath}
\usepackage{amssymb}
\usepackage{cancel}
\usepackage{amsthm}
\newtheorem*{prop}{Proposition}
\usepackage{algorithm}
\usepackage{algorithmic}
\usepackage{booktabs}
\usepackage{caption}

\usepackage{etoolbox}
\usepackage{etoc}

\usepackage{wasysym}
\usepackage{fontawesome5}
\usepackage{textcomp}
\usepackage{multirow}
\usepackage[table]{xcolor}
\usepackage{makecell}

\usepackage[column=S]{cellspace} % Fixes the gap/spacing issues
\usepackage{listings}

\lstnewenvironment{code}{%
  \lstset{
    basicstyle=\ttfamily\small, % Use the nice font (zi4)
    breaklines=true,            % WRAP LINES (Key feature)
    columns=fullflexible,       % Better spacing for variable-width fonts
    keepspaces=true,            % Preserves indentation
    aboveskip=0pt,              % No extra space above
    belowskip=0pt,              % No extra space below
    frame=none,                 % No extra frame (since 'prompt' handles it)
    extendedchars=true,
    literate={-}{-}1,           % Fixes hyphen rendering in some fonts
  }
}{}

\newtoggle{anonymous}
\togglefalse{anonymous}

\newtoggle{supplemental}
\toggletrue{supplemental}

\newtoggle{appendix_only}
\togglefalse{appendix_only}

\usepackage[framemethod=TikZ]{mdframed}
\mdfdefinestyle{ThmFrame}{linecolor=white,
    outerlinewidth=0pt,
    roundcorner=2pt,
    leftmargin=-3pt,
    rightmargin=-2pt,
    innertopmargin=5pt, innerbottommargin=5pt, innerrightmargin=5pt,
    innerleftmargin=5pt,
    backgroundcolor=MaterialBlueGrey100!40}

\DeclareMathOperator*{\E}{\mathbb{E}}

\usepackage[numbers, sort&compress]{natbib}
\usepackage{multicol}
\usepackage[hidelinks]{hyperref}
\usepackage[utf8]{inputenc}

\usepackage{xspace}

\makeatletter
\DeclareRobustCommand\onedot{\futurelet\@let@token\@onedot}
\def\@onedot{\ifx\@let@token.\else.\null\fi\xspace}
\newcommand{\eg}{e.g.,\xspace}
\newcommand{\ie}{i.e.,\xspace}

\newcommand{\etc}{etc\onedot}

\newcommand{\iid}{i.i.d\onedot}
\makeatother

\newcommand{\algname}{\texttt{R\&B-EnCoRe}\xspace}

\usepackage[strict]{changepage}
\usepackage{framed}
\definecolor{formalshade}{rgb}{0.90, 0.95, 0.90}%{0.85,1,0.85}
\definecolor{green}{rgb}{0.0,0.6,0.30}
\newenvironment{formal}{%
  \MakeFramed{\advance\hsize-\width\FrameRestore}%
  \noindent\hspace{-4.55pt}% disable indenting first paragraph
  \begin{adjustwidth}{}{7pt}%
}
{%
  \end{adjustwidth}\endMakeFramed%
}

\definecolor{hyposhade}{rgb}{0.92, 0.95, 0.98}
\definecolor{hypoline}{rgb}{0.15, 0.35, 0.65}

\newenvironment{hypothesis}{%
  \MakeFramed{\advance\hsize-\width\FrameRestore}%
  \noindent\hspace{-4.55pt}% 
  \begin{adjustwidth}{}{7pt}%
}
{%
  \end{adjustwidth}\endMakeFramed%
  \vspace{-6pt}%
}

\definecolor{promptshade}{rgb}{0.95, 0.92, 0.96}
\definecolor{promptline}{rgb}{0.40, 0.30, 0.50}
\newenvironment{prompt}{%
  \MakeFramed{\advance\hsize-\width\FrameRestore}%
  \noindent\hspace{-4.55pt}% 
  \begin{adjustwidth}{}{7pt}%
}
{%
  \end{adjustwidth}\endMakeFramed%
  \vspace{-6pt}%
}

\begin{document}

\input{headers/title}

\input{headers/author}

\maketitle
\iftoggle{anonymous}{}{
\begingroup
\renewcommand\thefootnote{\textsuperscript{\faMusic}}
\footnotetext{Equal contribution.}
\renewcommand\thefootnote{\textsuperscript{}}
\endgroup
}

\input{headers/abstract}

\IEEEpeerreviewmaketitle

% this is just lists of plans
% \input{draft ideas/planning}

\input{sections/introduction}

\input{sections/related_works}

\input{sections/preliminaries}

\input{sections/proposed_approaches}

\input{sections/experiments}

\input{sections/discussion_conclusion}

\input{sections/acknowledgments}

\newpage
\bibliographystyle{unsrtnat}  %{plainnat}
\bibliography{references}

\iftoggle{supplemental}{
%\newpage

\input{appendix/appendix}
}{}

\end{document}

%% file: headers/title.tex
\title{Self-Supervised Bootstrapping of Action-Predictive Embodied Reasoning}

%% file: headers/author.tex
% You will get a Paper-ID when submitting a pdf file to the conference system

\newcommand{\e}{\textsuperscript{\textsuperscript{\faMusic}}}
\newcommand{\s}{\textsuperscript{$\sharp$}}
\newcommand{\B}{\textsuperscript{$\flat$}}
\newcommand{\N}{\textsuperscript{$\natural$}}

\author{
\authorblockN{Milan Ganai\s\e,
Katie Luo\s\e,
Jonas Frey\s\B, 
Clark Barrett\s, and
Marco Pavone\s\N}
\authorblockA{\s Stanford, \B UC Berkeley, \N NVIDIA\\
\texttt{\{mganai, katieluo, jonfrey, barrettc, pavone\}@stanford.edu}}
}

%% file: headers/abstract.tex
\begin{abstract}

Embodied Chain-of-Thought (CoT) reasoning has significantly enhanced Vision-Language-Action (VLA) models, yet current methods rely on rigid templates to specify reasoning primitives (\eg objects in the scene, high-level plans, structural affordances). These templates can force policies to process irrelevant information that distracts from critical action-prediction signals.
This creates a bottleneck: without successful policies, we cannot verify reasoning quality; without quality reasoning, we cannot build robust policies. 
We introduce \algname, which enables models to bootstrap embodied reasoning from internet-scale knowledge through self-supervised refinement. 
By treating reasoning as a latent variable within importance-weighted variational inference, models can generate and distill a refined reasoning training dataset of embodiment-specific strategies without external rewards, verifiers, or human annotation.
We validate \algname across manipulation (Franka Panda in simulation, WidowX in hardware), legged navigation (bipedal, wheeled, bicycle, quadruped), and autonomous driving embodiments using various VLA architectures with 1B, 4B, 7B, and 30B parameters.
Our approach achieves 28\% gains in manipulation success, 101\% improvement in navigation scores, and 21\% reduction in collision-rate metric over models that indiscriminately reason about all available primitives.
\algname enables models to distill reasoning that is predictive of successful control, bypassing manual annotation engineering while grounding internet-scale knowledge in physical execution.

\end{abstract}

%% file: sections/introduction.tex
\section{Introduction}

Vision-Language-Action models (VLAs) have begun to successfully harness the momentum of internet-scale foundation models~\cite{bommasani2021opportunities} to serve as powerful generalist robot policies~\cite{kawaharazuka2025vision, zhong2025survey}. These models inherit vast semantic and visual knowledge from Vision-Language Models (VLMs) pretrained on web-scale image-text corpora~\cite{zhang2024vision}, and adapt to embodied tasks through continued training on relatively data-scarce robotics demonstrations~\cite{kim2025openvla, belkhale2024minivla, black2025pi0, black2025pi05, intelligence2025pi}. To transform reactive policies that directly map tasks to actions into models with deeper understanding of the physical world, recent approaches leverage structured Chain-of-Thought (CoT) reasoning that bridges the gap between abstract visual and semantic knowledge and embodied control~\cite{belkhale2024rt, zawalski2025Robotic, chen2025training}. Specifically, these VLAs learn to output intermediate textual reasoning to elucidate the thinking process of converting high-level intent into low-level control~\cite{team2025gemini}.

\input{floats/figures/teaser_fig}

However, generating effective robotic reasoning traces remains a major bottleneck. Industry currently invests months in manual annotation guidelines to structure visual and spatial details~\cite{azzolini2025cosmos}. To scale this labeling process, foundation models are increasingly used to synthetically generate content in visual question answering (VQA) style for reasoning primitives (\eg objects in the scene, what the planned steps should be)~\cite{touvron2023llama, deitke2025molmo, bai2025qwen3vltechnicalreport, sermanet2024robovqa}. Yet, indiscriminately applying these rich, internet-scale priors for all tasks risks overwhelming embodied reasoning VLAs with irrelevant or distracting information~\cite{liu2025simpact, liu2025evovla, sun2025emma}. Without principled mechanisms to identify which reasoning primitives predictively inform physical interaction, embodied CoT models often exhibit verbose reasoning that fails to ground internet-scale knowledge in robotics contexts~\cite{man2025argus}. Thus, a critical challenge remains: \textit{how do we distill visual and semantic knowledge to construct action-predictive embodied reasoning data, 
% \katie{eliminating months} 
reducing the burden
of manual CoT engineering}~\cite{zhai2025igniting, sumers2023distilling, zhou2025chatvla}\emph{?}

Applying CoT to robotics is hindered by a fundamental embodiment-to-reasoning grounding gap: unlike VLMs pretrained on massive image-text datasets, robotics lacks internet-scale corpora linking vision and language to action, making it difficult to verify what reasoning truly informs control. Consequently, the effective reasoning process for physical action remains an unobserved latent variable---there is no oracle to verify if a thought effectively informs a movement, creating a circular dependency where robust training requires high-quality reasoning, yet identifying such reasoning requires an already-successful policy. Current approaches struggle to bridge this gap, often resorting to rigid, ``one-size-fits-all'' templates that force embodied reasoning models to process irrelevant details---like social norms in an empty scene environment---thereby ignoring the heterogeneity of robotic tasks. To address this, we introduce \algname (\textbf{\underline{R}}efine and \textbf{\underline{B}}ootstrap \textbf{\underline{E}}mbodime\textbf{\underline{n}}t-specific \textbf{\underline{C}}hain-\textbf{\underline{o}}f-Thought \textbf{\underline{Re}}asoning), a framework that treats reasoning not as a fixed sequence, but as a latent strategy (see Fig.~\ref{fig:latent_bridge}) to be discovered and optimized.

Given a training dataset of robotics demonstrations, \algname leverages synthetic reasoning traces generated from foundation models and addresses the challenge of \textit{unverifiable quality} by formulating embodied reasoning in a variational inference framework. 
Instead of relying on external rewards, heuristics, or verifiers, our self-supervised approach uses VLA models' generative probabilities to weigh reasoning strategies based on information benefit in expert action prediction.
By importance sampling on these weights, \algname selects refined reasoning traces---pruning confounding or verbose information inherited from web-scale priors---to further bootstrap VLA training.
\algname helps mitigate data scarcity in robotics by distilling informative, high-quality embodied reasoning from internet-scale, multimodal priors, bypassing the need for expensive human annotations.

Through \algname, we find that VLAs can be distracted by irrelevant information and slowed by the verbosity of exhaustive reasoning---such as enumerating every visible object or always analyzing counterfactual paths---whereas selective reasoning primitives tailored to the embodiment improve both task success and efficiency. Empirically, across manipulation, legged locomotion, and autonomous driving benchmarks, \algname discovers distinct, interpretable reasoning distributions that bootstrap task success, yielding concise traces that reflect critical signals driving the model's decision-making. Ultimately, \algname creates potential for self-improving policies that learn not only to act, but to leverage priors to ponder the right questions before acting.
Our primary contributions are as follows:

\input{floats/figures/latent_bridge_reason}

\begin{enumerate}

\item We introduce \algname, a method that formulates embodied reasoning as a latent variable and provides a drop-in training recipe that self-improves from synthetic priors. By utilizing importance-weighted variational inference, our approach enables VLAs to refine and bootstrap reasoning strategies based on their information benefit for action prediction and eliminates ``one-size-fits-all'' heuristics and expensive human annotations.

\item We demonstrate that \algname effectively filters out ``distractor'' information (e.g., bounding boxes of task-irrelevant objects) while amplifying critical signals (e.g., structural affordances for legged robots). This process yields high-quality, interpretable reasoning traces that distinguish essential visual cues from redundant noise without external heuristics, rewards, or verifiers.

\item We validate \algname on manipulation, legged locomotion, and autonomous driving using various VLA architectures with 1B, 4B, 7B, and 30B parameters. \algname consistently produces models with action-predictive embodied reasoning, outperforming baseline models like those reasoning on all primitives, with 28\% gain in manipulation success, 101\% improvement in legged navigation scores, and 21\% reduction in autonomous driving collision-rate metric.

\end{enumerate}

%% file: floats/figures/teaser_fig.tex
\begin{figure}
    \centering
    \includegraphics[width=\linewidth]{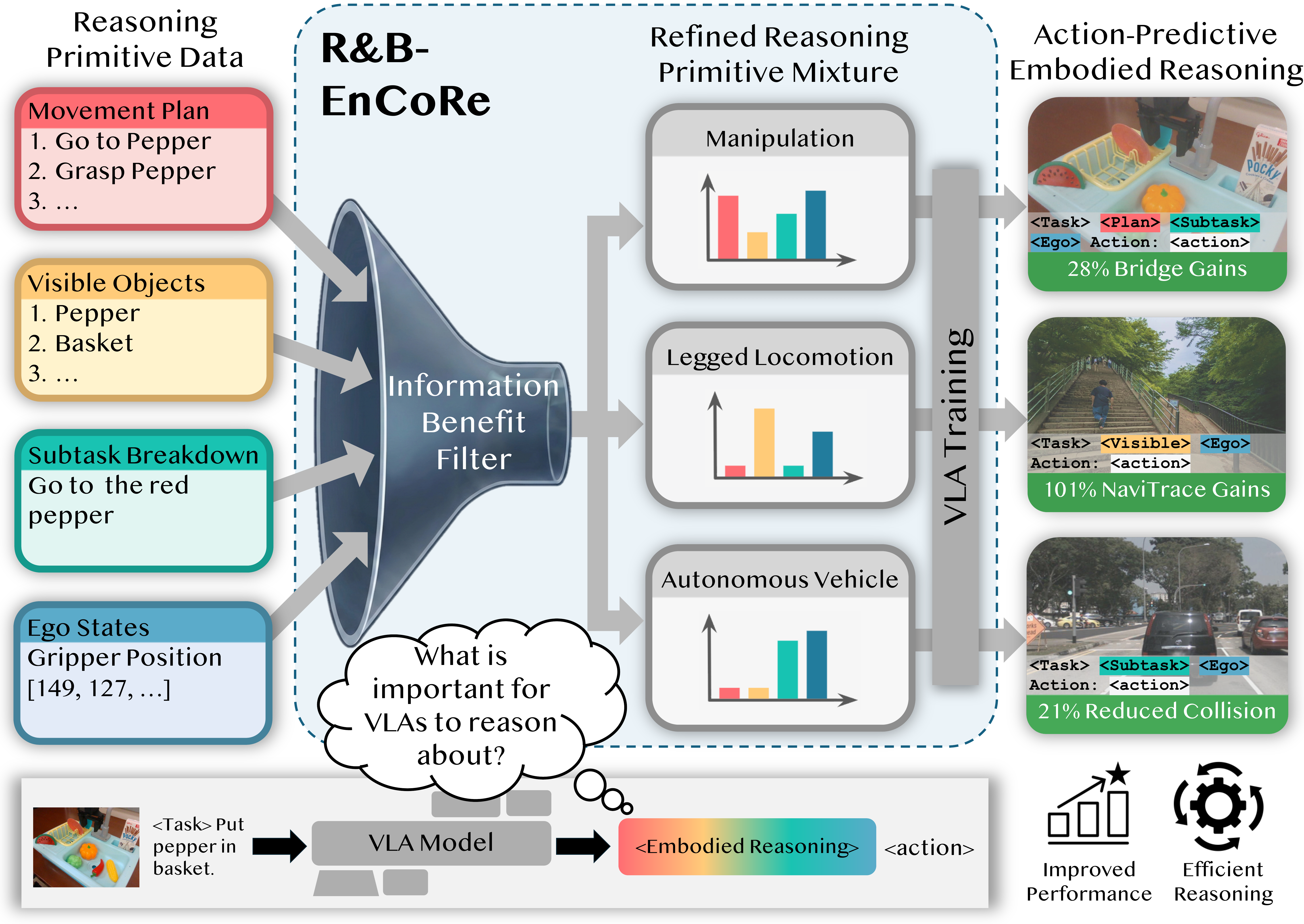}

    \vspace{-4pt}
    \caption{\small We generate diverse embodied reasoning primitives and refine them based on action-prediction information benefit. We bootstrap policy performance by retraining on these self-refined, high-quality reasoning traces, discovering embodiment-specific reasoning distributions that reveal effective strategies, significantly improving VLA task success while producing more efficient CoT traces.}
    \label{fig:teaser}
\iftoggle{anonymous}{\vspace{-15pt}
}{}
\end{figure}

%% file: floats/figures/latent_bridge_reason.tex
\begin{figure}[t]
    \centering
    \includegraphics[width=\linewidth]{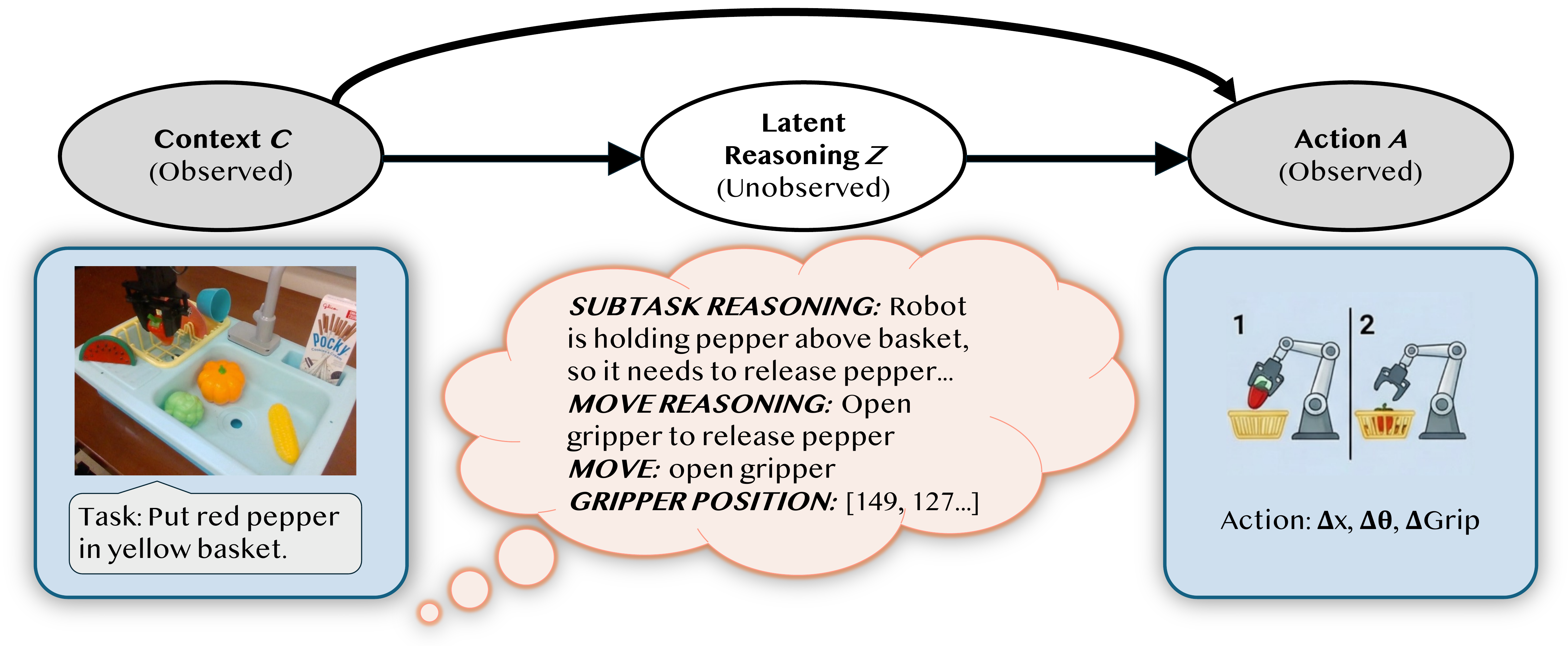}

\iftoggle{anonymous}{\vspace{-4pt}
}{}
    \caption{\small \textbf{Top:} Probabilistic Graphical Model relating the Task Context ($C$), Reasoning ($Z$), and Action ($A$). The latent reasoning $Z$ is induced from a set of primitives $\mathcal{R}$ (e.g., subtask reasoning, move reasoning). \textbf{Bottom:} An example reasoning trace on the Bridge setup.}
    \label{fig:latent_bridge}
%\vspace{-15pt}
\end{figure}

%% file: sections/related_works.tex
\iftoggle{anonymous}{\vspace{-5pt}
}{}
\section{Related Works}

\noindent \textbf{Vision-Language-Action Models.}  
We build on recent advancements in VLA architectures to leverage their inherent multimodal understanding across diverse embodiments.
Robot learning has begun to develop generalist policies that integrate visual and semantic representations to operate in unstructured environments~\cite{kawaharazuka2025vision, zhong2025survey}. Leveraging foundation models~\cite{bommasani2021opportunities} such as Vision-Language Models (VLMs)~\cite{bai2025qwen3vltechnicalreport, steiner2024paligemma, tschannen2025siglip} pre-trained on internet-scale corpora, these approaches adapt semantic and visual priors to embodied tasks~\cite{karamcheti2024prismatic}. This paradigm treats action prediction as a vision-language problem: models are trained via behavioral cloning on diverse datasets~\cite{oneill2024openx, walke2023bridge, khazatsky2024droid, caesar2020nuscenes} to map scene images and task descriptions to discrete action tokens representing expert control~\cite{pertsch2025fast, wang25vqvla, belkhale2024minivla}. This allows models to inherit the generalization capabilities of foundation models~\cite{firoozi2025foundation} and has been applied to manipulation~\cite{brohan2022rt1, zitkovich2023rt2, octo2023octo, kim2025openvla, belkhale2024minivla, black2025pi0, black2025pi05, intelligence2025pi, bjorck2025gr00t}, legged navigation~\cite{bjorck2025gr00t, cheng2025navila, ding2024quar, ding2025humanoid, jiang2025wholebodyvla}, and autonomous driving~\cite{zhou2025opendrive}.

\input{floats/figures/methods_diagram}

\noindent \textbf{Semantic and Visual Reasoning.} Chain-of-Thought (CoT) reasoning has enhanced LLM and VLM performance by generating intermediate logical steps before producing a final answer~\cite{wei2022chain, kojima2022large}. This computation increases expressivity and search capabilities~\cite{merrill2023expressive, li2024chain}, refining internal representations to better answer complex queries~\cite{feng2023towards, wang2024chain} in domains ranging from math and coding to visual question answering~\cite{wang2025survey, yang2025code, xu2025llava, zhou2025visualizing}. 
Beyond standard prompting, recent efforts explicitly integrate reasoning objectives during pre-training and post-training~\cite{abdin2024phi, li2023textbooks, luo2023wizardcoder, wei2023magicoder, liu2023tinygsm}, or they improve reasoning and instruction following via supervised finetuning~\cite{xu2023wizardlm, ge2024scaling, taori2023stanford} or reinforcement learning and self-play~\cite{trinh2024solving, guo2025deepseek, silver2016mastering, silver2017mastering}. Recent works~\cite{chen2024hiddenreasoners, hoffman2023training, hu2024amortizing, zhong2025brite} leverage Variational Inference to formulate reasoning as a latent variable for language tasks, using external verifiable rewards. Similarly, Expectation-Maximization algorithms have been used to iteratively bootstrap reasoning quality~\cite{singh2024beyond, ruan2025reasoning}. 
Our work adapts these variational inference techniques to the multimodal robotics domain, enabling VLAs to refine and bootstrap embodied reasoning without external supervision.

\noindent \textbf{Embodied Reasoning.} Recent work on generalist policies has extended CoT to robotics by generating intermediate reasoning traces before action prediction~\cite{team2025gemini}, finding that such ``thinking'' enhances generalization and robustness~\cite{zawalski2025Robotic}. Prior work employs diverse reasoning primitives: manipulation policies reason about task decomposition~\cite{sharma2022skill, wu2025you}, end-effector position~\cite{team2024gemini, lee2025molmoact}, bounding boxes~\cite{zhu2025objectvla}, language motions~\cite{belkhale2024rt}, future frames~\cite{zhao2025cot}, spatial relations~\cite{chen2024spatialvlm, zheng2024tracevla, driess2023plame}, and affordances~\cite{mees2023grounding, xu2025a0, borja2022affordance}. Navigation models similarly have reasoned about perception~\cite{mao2023gpt}, planning~\cite{maolanguage}, meta-actions~\cite{nvidia2025alpamayor1bridgingreasoningaction}, failure modes~\cite{ganai2025realtime}, and counterfactuals~\cite{peng2025counterfactual, glossop2025cast}. These methods typically rely on rigid, structured CoT annotations to ground visual and semantic knowledge~\cite{zawalski2025Robotic, chen2025training, van2025refinevla, clark2025action}. Determining the optimal combination of reasoning primitives for a specific robot and task remains an open challenge; 
% we address this by formulating reasoning as a latent variable, allowing the VLA to filter irrelevant information and bootstrap the strategies most critical for its specific embodiment.
our work formulates reasoning as a latent variable, allowing the VLA to filter irrelevant information and bootstrap the strategies most critical for its specific embodiment.

%% file: floats/figures/methods_diagram.tex
\begin{figure*}[ht]
    \centering
    \includegraphics[width=\textwidth]{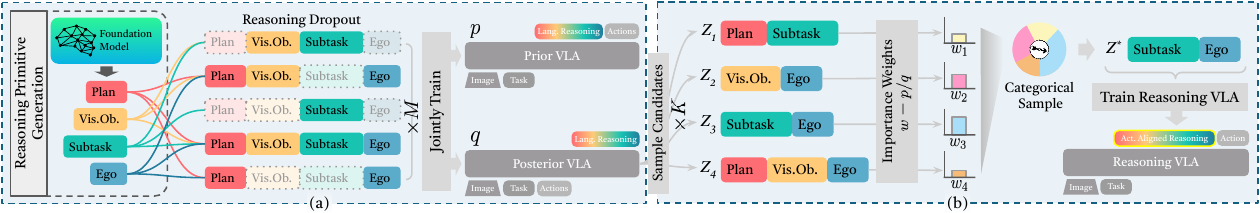}

\iftoggle{anonymous}{\vspace{-4pt}
}{}
    \caption{\small Overview of \algname. (a) We generate diverse reasoning primitives (e.g., Plan, Visible Objects) and combine them via dropout to warmstart model capturing prior and posterior distributions. (b) We sample candidates from posterior and apply importance weighting to filter for reasoning that maximizes action prediction power. These refined, high-quality reasoning traces are used to bootstrap the final VLA.}
    \label{fig:methods}
%\vspace{-15pt}

\end{figure*}

%% file: sections/preliminaries.tex
\section{Preliminaries on Variational Inference}
\label{sec:variationalinference}

In this work, we leverage variational inference to identify effective embodied reasoning for VLA action generation.
We briefly describe the latent variable model framework and the Importance Weighted Autoencoder (IWAE),
which will serve as the foundation for our algorithm \algname (see Fig.~\ref{fig:methods}).

\subsection{Latent Variable Models and the Variational Autoencoders}

Given a context $C$ (\eg observation and task), latent variable $Z$ (\eg textual reasoning) and observed variable $A$ (\eg action) as in Fig.~\ref{fig:latent_bridge}, the objective in Variational Inference is to find a latent distribution that maximizes marginal log-likelihood of the ground truth observed data.
For generative model $ p(Z,A\mid C) = p(A \mid C,Z)p(Z\mid C) $, which we term the prior distribution, the objective can be written as:
\begin{equation*}
\log p(A\mid C) = \log \int p(Z, A\mid C) \, dz.
\end{equation*}
In the rest of the paper, we drop the context $C$ for brevity when its usage as a conditional variable is implicit.

Typically for generative models, computing this integral is intractable. Variational Autoencoders (VAE) address this by introducing a posterior distribution estimate model 
% (i.e. encoder) 
$q(Z \mid A)$ to approximate the true posterior $p(Z \mid A)$. By applying Jensen’s Inequality to the evidence, the Evidence Lower Bound (ELBO) for VAE can be obtained:
\begin{equation*}
\log p(A) = \log\E_{Z\sim q}\left[  \frac{p(Z, A)}{q(Z \mid A)} \right] \geq \underbrace{\E_{Z\sim q}\left[ \log \frac{p(Z, A)}{q(Z \mid A)} \right]}_{\text{ELBO}_{\text{VAE}}} . %\doteq L_{VAE}(\theta, \phi)
\end{equation*}
This ELBO is composed of computationally tractable terms. Note, the difference between $\text{ELBO}_{\text{VAE}}$ and the true evidence is the Kullback-Leibler (KL) divergence between the approximate and true posterior: $ D_{KL}(q(Z \mid A) \parallel p(Z \mid A)) $~\cite{bishop2006pattern}.

\subsection{Importance Weighted Autoencoders (IWAE)}

The IWAE framework \cite{burda16importance,cremer2017reinterpreting} improves on VAE by creating an importance sampling estimate of the evidence via multiple samples. We can draw $K$ \iid samples $Z_1, \dots, Z_K$ from the posterior distribution $q(Z \mid A)$. The $K$-sample importance-weighted lower bound is defined as:
\begin{equation*}
\mathcal{L}_K \doteq \mathbb{E}_{Z_1, \dots, Z_K \sim q(Z \mid A)} \left[ \log \left( \frac{1}{K} \sum_{k=1}^K \frac{p(Z_k, A)}{q(Z_k \mid A)} \right) \right]    .
\end{equation*}

To see that this is a lower bound on $\log p(A)$, we observe that the expectation of the term inside the log is an unbiased estimator of the evidence:
\begin{align*}
\E_{Z_{1:K} \sim q} \left[ \frac{1}{K} \sum_{k=1}^K \frac{p(Z_k, A)}{q(Z_k \mid A)} \right] = \frac{1}{K} \sum_{k=1}^K \E_{Z_k \sim q} \left[ \frac{p(A, Z_k)}{q(Z_k \mid A)} \right]  \\
= \frac{1}{K} \sum_{k=1}^K \int  \frac{p(A, Z_k)}{q(Z_k \mid A)} q(Z_k \mid A) dZ_k = p(A).
\end{align*}
Using Jensen’s Inequality on the concave log function yields:
\begin{equation*}
\log p(A) = \log \E \left[ \frac{1}{K} \sum_{k=1}^K w_k \right] \geq \E \left[ \log \frac{1}{K} \sum_{k=1}^K w_k \right] = \mathcal{L}_K,
\end{equation*}
where $w_k = \frac{p(Z_k, A)} {q(Z_k \mid A)}$ are importance weights. For $K>1$, \cite{burda16importance} shows $\mathcal{L}_{K+1} \geq \mathcal{L}_K \geq \mathcal{L}_1=\text{ELBO}_{\text{VAE}}$, so by increasing sample count, IWAE theoretically improves evidence estimate.
\iftoggle{anonymous}{}{In the next section, we propose our algorithm which is based on a sampling-importance-resampling technique introduced in~\cite{cremer2017reinterpreting} that improves upon IWAE by estimating the posterior distribution using a categorical distribution of the importance weights. More details on the theory behind this technique can be found in Appendix~\ref{sec:iwae_theory}.}

%% file: sections/proposed_approaches.tex
\section{\algname Framework}
\label{sec:method}

\input{floats/alg/warmstarting}

We propose our approach, \algname, which treats embodied reasoning not as a fixed annotation, but as a latent variable $ Z $ that explains the relationship between observation and task context $ C $ and physical action $ A $ (Fig.~\ref{fig:latent_bridge}).

\subsection{Warmstarting Strategy Hypotheses via Reasoning Dropout}
\label{sec:warmstart}

While foundation models possess the capability to reason about diverse aspects of a physical scene, we lack ground-truth data determining which specific signals are useful for low-level robotic control. We hypothesize a set of $\rho$ potentially relevant reasoning primitives $\mathcal{R}=\{R_1, \dots, R_\rho\}$ 
% (e.g., high-level plans, Affordances, Terrain Analysis)
(\eg high-level plans, visual objects, subtask breakdowns, ego state)
and extract textual explanations for each by posing them as VQA tasks to Foundation Models or using human annotated datasets.
% (Fig.~\ref{fig:methods}a).

To discover the optimal combination of these primitives for specific tasks, we must expose the model to various strategies $\mathbf{R}\subseteq \mathcal{R}$ of combinations during training rather than a single fixed template. To this end, we construct a ``warmstart'' dataset $\mathcal{D}_{\text{warm}}$ by sampling from the powerset of reasoning primitives. Specifically, for each demonstration $(C,A)$, we generate $M$ synthetic traces $Z_j,~j\in\{1 \cdots M\}$ using Reasoning Dropout (refer to Alg.~\ref{alg:rbencore_train}). 
Each primitive $R_r \in \mathcal{R}$ is independently included or dropped with a fixed probability $d$. This mechanism exposes the model to various reasoning strategies---ranging from concise to verbose---providing diversity to our refining stage to identify signals that improve action prediction.

\iftoggle{anonymous}{\vspace{-5pt}
}{}

\subsection{Jointly Training Prior and Posterior}
\label{sec:joint_training}

We utilize these diverse traces to train a Vision-Language-Action (VLA) model that serves as the posterior and prior roles simultaneously, mirroring the encoder-decoder structure of variational autoencoders but adapted for embodied reasoning (Fig.~\ref{fig:methods}a).
The prior model represents the agent's ability to generate reasoning and actions at test time, while the posterior model represents the ability to explain actions in hindsight. 
Specifically, the prior is trained on the warmstart dataset $\mathcal{D}_{\text{warm}}$ to generate reasoning strategies with correct actions, and the posterior is trained on the same dataset to propose reasoning candidates conditioned on actions:

\vspace{0.1cm}
\noindent\textbf{Prior Model} $ p(Z, A \mid C) $: Trained on the sequences 
$ \{(Z_j, A \mid C)\}$ by conditioning on the context tokens. This represents the online policy: it observes context $ C $, generates reasoning $ Z_j $, and predicts action $ A $. This distribution captures the predictive power of reasoning for control.

\vspace{0.1cm}
\noindent\textbf{Posterior Model} $ q(Z \mid C, A) $: Trained on the sequences $ \{(Z_j \mid C, A) \}$ by conditioning on context \emph{and action} tokens. By conditioning on task $C$ and ground-truth action $A$, the posterior learns to generate diverse reasoning $ Z_j $, explaining how action $A$ is the correct response, acting as a proposal distribution from which we sample action-relevant reasoning.

\input{floats/alg/refine_bootstrap}

\iftoggle{anonymous}{\vspace{-5pt}
}{}

\subsection{Refining and Bootstrapping via Importance Sampling}
\label{sec:refine_bootstrap}

We refine VLA reasoning through a two-step process: first, we use our trained prior and posterior model to generate and refine reasoning traces through importance sampling; second, we retrain a VLA with these more action-predictive embodied reasoning traces.

\input{floats/figures/reasoning_distrib}

\noindent\textbf{Stage 1: Information-Beneficial Reasoning (Refine).}
We refine the reasoning-enriched data by identifying strategies that improve action prediction through importance weighting (Alg.~\ref{alg:rbencore_refine}).
For each demonstration $(C, A)$, we sample $K$ candidate reasoning traces $\{Z_k\}_{k=1}^K \sim q(\cdot \mid C,A)$ from the trained posterior. 
For each trace $Z_k$, we compute importance weight:
\begin{equation*}
w(Z_k) = \frac{p(Z_k, A \mid C)}{q(Z_k \mid C, A)}.
\end{equation*}
Then, we resample a single trace $Z^*$ from these candidates according to a categorical distribution proportional to the importance weights $\text{Cat}(w(Z_k))$. Intuitively, this process yields reasoning traces with a natural interpretation---selecting a reasoning primitive for inclusion in the final strategy indicates that it provides information about the correct action.
To formalize this, we define the \emph{information benefit} of a reasoning strategy $\mathbf{R}$ as how much it reduces the divergence between our model's action distribution and the expert's distribution $p_{data}(A|C)$:
$$\Delta \mathcal{I}_\mathbf{R} \doteq D_{\text{KL}}(p_{\text{data}} \| p(A|C,\mathcal{Z}_{\cancel{\mathbf{R}}})) - D_{\text{KL}}(p_{\text{data}} \| p(A|C,\mathcal{Z}_\mathbf{R})),$$
where $\mathcal{Z}_\mathbf{R}$ and $\mathcal{Z}_{\cancel{\mathbf{R}}}$ denote the set of traces with and without strategy $\mathbf{R}$. Our importance weighting estimates this quantity:

%\vspace{-5pt}
\begin{formal}
\iftoggle{anonymous}{\vspace{-5pt}
}{}
\begin{prop}[Importance Weight Ratios Estimate Information Benefit] 
\label{prop:info_benefit}
Under the training (Alg.~\ref{alg:rbencore_train}) and sampling (Alg.~\ref{alg:rbencore_refine}) procedures, the expected log-ratio of importance weights equals the information benefit:
\begin{equation*}
\mathbb{E}_{A \sim p_{\text{data}}} \left[ \log \frac{\mathbb{E}_{Z \sim q}[w(Z_\mathbf{R}) \mid Z_\mathbf{R} \in \mathcal{Z}_\mathbf{R}]}{\mathbb{E}_{Z \sim q}[w(Z_{\cancel{\mathbf{R}}}) \mid Z_{\cancel{\mathbf{R}}} \in \mathcal{Z}_{\cancel{\mathbf{R}}}]} \right] = \Delta \mathcal{I}_\mathbf{R}. 
\end{equation*}
\text{Proof in Appendix~\ref{sec:prop}.}
\end{prop}
\end{formal}
%\vspace{-5pt}
\noindent
More intuitively, \algname automatically amplifies reasoning strategies that improve action prediction ($\Delta \mathcal{I}_{\mathbf{R}} > 0$) while suppressing distracting ones ($\Delta \mathcal{I}_{\mathbf{R}} < 0$) in a self-supervised manner. The result is a refined dataset of concise, action-predictive reasoning traces tailored to the specific embodiment and task.

\vspace{0.1cm}
\noindent\textbf{Stage 2: Training Action-Aligned Reasoning (Bootstrap).}
% \noindent\textbf{Bootstrap: Training on Action-Aligned Reasoning.}
We bootstrap VLA performance by retraining the model on the filtered, high-quality reasoning dataset $ \{(Z^*, A \mid C)\} $ (Fig.~\ref{fig:methods}b). This yields a policy that reasons about information beneficial for control. By training on reasoning traces selected for their information benefit, the final VLA policy learns from examples where the reasoning is aligned with task success. This mitigates noise introduced by training on all reasoning primitives, pruning those that distract from predicting expert action. Specifically, the model learns from the refined data robust internal representations that more effectively map context to action via task-relevant intermediate reasoning.

%% file: floats/alg/warmstarting.tex
\begin{algorithm}[t]
\caption{\algname: Warmstarting}
\label{alg:rbencore_train}
\begin{algorithmic}[1]
\REQUIRE Dataset $\mathcal{D} = \{(C^i, A^i)\}_{i=1}^N$, Reasoning Primitives $\mathcal{R}$, Dropout rate $d$, VLM $\mathcal{M}$, Foundation Model $\text{FM}$

\STATE $\mathcal{D}_{\text{warm}} \leftarrow \emptyset$
\FOR{each $(C^i, A^i) \in \mathcal{D}$}
    %\STATE \COMMENT{\textit{Query FM to generate text $z_R, \forall R \in \mathcal{R}$}}
    \STATE $\{z_R\}^i_{R \in \mathcal{R}} \leftarrow \text{Query}(\text{FM}, (C^i,A^i), \mathcal{R})$ 
    \FOR{$j=1$ \TO $M$}
        \STATE Construct reasoning trace $Z^i_j \subseteq \{z_R\}^i_{R \in \mathcal{R}} $ by including each $z \in \{z_R\}^i_{R \in \mathcal{R}} $ with probability $1-d$
        \STATE $\text{w.p. 0.5 : } \mathcal{D}_{\text{warm}} \leftarrow\mathcal{D}_{\text{warm}} \cup \{(Z^i_j, A^i \mid C^i)\} $ \\
        $\text{o.w. : } \mathcal{D}_{\text{warm}} \leftarrow
        \mathcal{D}_{\text{warm}} \cup  \{(Z^i_j \mid C^i, A^i)\}$
    \ENDFOR
\ENDFOR

\STATE $\mathcal{M}_{pq}\leftarrow\text{Train}(\mathcal{M}, \mathcal{D}_{\text{warm}})$

\RETURN $\mathcal{M}_{pq}$
\end{algorithmic}

\end{algorithm}

%% file: floats/alg/refine_bootstrap.tex
\begin{algorithm}[t]
\caption{\algname: Refinement \& Bootstrapping}
\label{alg:rbencore_refine}
\begin{algorithmic}[1]
\REQUIRE Dataset $\mathcal{D} = \{(C^i, A^i)\}_{i=1}^N$,  $\mathcal{M}_{pq}$, VLM $\mathcal{M}$
%\ENSURE Final Policy $\pi_{\text{final}}$

\STATE $\mathcal{D}_{\text{refined}} \leftarrow \emptyset$ 
\FOR{each $(C^i, A^i) \in \mathcal{D}$}
    \STATE Sample $K$ candidates $\{Z^i_k\}_{k=1}^K \sim q(\cdot | C^i, A^i)$ 
    \FOR{$k=1$ to $K$}
        \STATE Compute weight $w(Z^i_k) \leftarrow \frac{p(Z^i_k, A^i | C^i) }{q(Z^i_k | C^i, A^i)}$ 
    \ENDFOR
    % \STATE Resample index $j \sim \text{Cat}(w(Z^i_k))$
    \STATE Sample reasoning trace $Z^{i*} \sim \text{Cat}(w(Z^i_k))$
    \STATE $\mathcal{D}_{\text{refined}} \leftarrow \mathcal{D}_{\text{refined}} \cup \{( Z^{i*}, A^i \mid C^i)\}$ 
\ENDFOR

%\STATE Train final policy $\pi_{\text{final}}(Z,A\mid C)$ on $\mathcal{D}_{refined}$ %to minimize $\mathcal{L}_{autoregressive}$ 
\STATE $\mathcal{M}_{\text{VLA}} \leftarrow \text{Train}(\mathcal{M}, \mathcal{D}_{\text{refined}})$

\RETURN $\mathcal{M}_{\text{VLA}}$

\end{algorithmic}

\end{algorithm}

%% file: floats/figures/reasoning_distrib.tex
\begin{figure*}[ht]
    \centering
    \includegraphics[width=\textwidth]{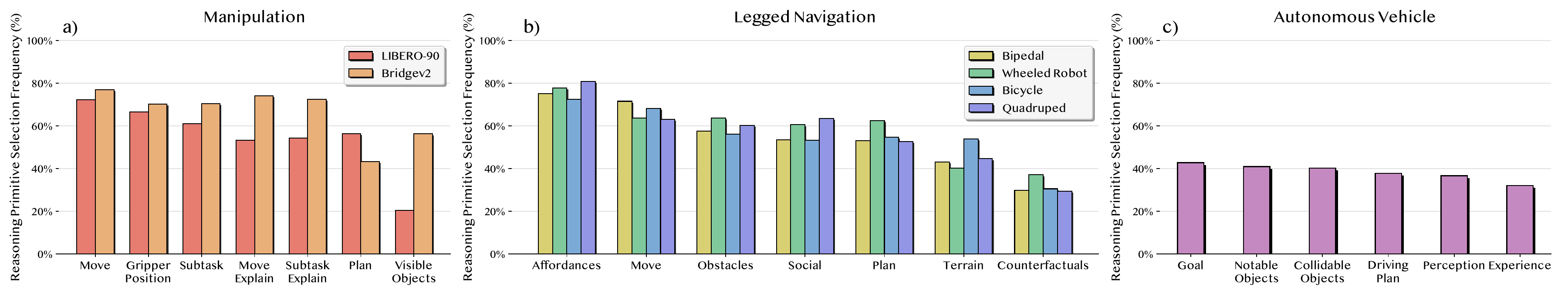}

    %\vspace{-8pt}
    \caption{\small This plot shows the reasoning primitives distributions that are generated from \algname refining warmstarting diverse reasoning strategy data. In a) the distribution for manipulation shows differences between reasoning for Franka Panda in simulation versus WidowX hardware in real-world data, notably for Visible Object, Move Explain, and Subtask Explain reasoning primitives. In b) we observe that the four-legged locomotion embodiments we investigate benefit in similar frequencies across reasoning types, with structural affordances being critical. For autonomous vehicles, we find in c) that reasoning focuses on goals and constraints.}
    \label{fig:mainresult}
%\vspace{-10pt}
\end{figure*}

%% file: sections/experiments.tex
\section{Experiments}

We report the performance of \algname across robotic manipulation (LIBERO-90 and real-world WidowX) in Sections~\ref{sec:Libero90Results} and \ref{sec:BridgeResults}, legged navigation (bipedal, wheeled robot, bicycle, quadruped) in Section~\ref{sec:NavigationResults}, and autonomous driving in Section~\ref{sec:AVResults}. Our evaluation shows that by treating embodied reasoning as a latent variable, VLAs can refine informative reasoning traces with high action-predictive power from internet-scale priors, without external heuristics, rewards, or verifiers. For each embodiment, our primary baselines are VLAs with no reasoning, VLAs trained with all reasoning primitives, and the prior VLAs with random reasoning primitives. We evaluate in the following experiments how \algname guides VLAs to self-refine their internal decision-making process to align more closely with expert control and ultimately bootstrap performance.
We detail our empirical results, structured to evaluate \algname's ability to refine reasoning distributions and improve task success in a self-supervised manner.

\input{sections/experiment_content/manipulation}

\input{sections/experiment_content/quadruped}

\input{sections/experiment_content/driving}

%% file: sections/experiment_content/manipulation.tex
\input{sections/experiment_content/manipulation/libero}

\input{sections/experiment_content/manipulation/bridge}

%% file: sections/experiment_content/manipulation/libero.tex
\subsection{LIBERO-90 Franka Panda Manipulation}
\label{sec:Libero90Results}

We evaluate our approach on the LIBERO-90~\cite{liu2023libero} manipulation benchmark by training a 1B-parameter MiniVLA~\cite{belkhale2024minivla}, which combines a 0.5B Qwen2.5 LLM~\cite{qwen2025qwen25technicalreport} backbone with DINOv2~\cite{oquab2024dinov} and SigLIP~\cite{zhai2023sigmoid} vision encoders. Applying \algname, we use a reasoning dropout rate of $d=0.2$. The context $C$ includes the scene image and task description. Following~\cite{chen2025training}, the latent variable $Z$ includes at most seven reasoning primitives generated by Llama 2~\cite{touvron2023llama} and Molmo~\cite{deitke2025molmo}: Plan, Visible Objects, Subtask, Subtask Explain, Move, Move Explain, and Gripper Position (see Appendix~\ref{sec:app_libero_details} for details). Action $A$ is seven discrete tokens encoding a 10-step action chunk~\cite{zhao2023learning} via VQ-VAE~\cite{van2017neural}. More details on LIBERO experiments are in Appendix~\ref{sec:app_libero_details}.

\begin{hypothesis}
\textbf{Q1} How effective is \algname at identifying task-salient perception reasoning in manipulation, \ie discerning action-critical objects, without external supervision?
\end{hypothesis}

\noindent
We validate whether \algname can discern relevant perceptual signals in LIBERO-90. We first apply our framework exclusively on the Visible Objects reasoning primitive. Using the architecture described in Section~\ref{sec:Libero90Results}, we restrict the reasoning to subsets of visible objects. We construct warmstart data by applying a dropout rate of $d=0.2$ to the ground-truth object list, generating traces with random subsets of the scene's objects. This forces the model to explore strategies ranging from sparse to exhaustive object enumeration.

As shown in Table~\ref{tab:object_list_stats}, refining visible objects significantly improves performance, outperforming baselines that reason with the full object list, random subsets, or no perception reasoning. This suggests that while perception is useful, exhaustive enumeration of every visible object can introduce noise that distracts from the immediate control task. To analyze refinement quality, we measure the \emph{Object Criticality Rate}: the percentage of traces where every listed object is task-salient, \ie explicitly mentioned in the a separate dataset of reasoning primitives (\eg Plan, Move). Crucially, these reference primitives act as a separate validation source and were \emph{not} used to train the models in this experiment, serving as an independent proxy for functional relevance. We find that the baseline model (listing all objects) almost never produces traces ($0.03\%$) with only task-salient objects. In contrast, \algname refines embodied CoT to prioritize task-relevant signals, increasing criticality rate to $>25\%$ and filtering irrelevant objects without external guidance (Fig.~\ref{fig:libero_90_VO}).

\begin{hypothesis}
\textbf{Q2} How does \algname refine verbose reasoning primitives in manipulation to improve task success?
\end{hypothesis}

\input{floats/figures/libero_VO}

\noindent
Applying \algname to refine a wider set of reasoning primitives in LIBERO-90, we see in Table~\ref{tab:liber90_res} that \algname achieves higher success over other reasoning strategies while reducing average token count by half compared to reasoning with all primitives. In Fig.~\ref{fig:mainresult}a, we observe \algname yields a refined training distribution of traces that prioritizes Move, Gripper Position, and Subtask reasoning primitives (Fig.~\ref{fig:mainresult}a). Notably, the model assigns high importance to Move primitives---short meta-action language descriptions. This confirms research findings~\cite{belkhale2024rt, nvidia2025alpamayor1bridgingreasoningaction} that generating these meta-action descriptions provides critical, nonredundant information that improves expert action prediction. We observe that \algname favors the concise Subtask/Move over the more verbose Subtask/Move Explain, as it prunes redundant justifications lacking information benefit in action prediction. Additionally, the Visible Objects primitive appears with low frequency ($20\%$) in the final distribution. This aligns with our findings in Q1: listing all objects is generally ineffective.

%% file: floats/figures/libero_VO.tex
% \begin{figure}[t]
%     \centering
%     \includegraphics[width=\linewidth]{images/LIBERO_Visible_Objects.pdf}
%     \caption{Libero VO}
%     \label{fig:libero_90_VO}
% \end{figure}
\begin{figure}[t]
\vspace{-10pt}
    \centering
    \includegraphics[width=\linewidth, trim=0 1.6cm 0 0.5cm, clip]{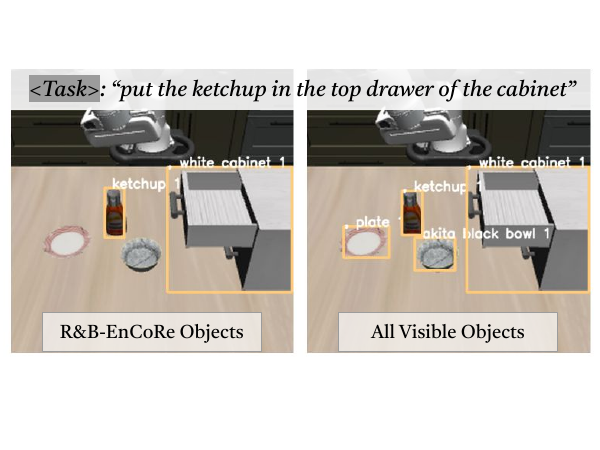}

    %\vspace{-4pt}
    \caption{\small Visible Objects generated in LIBERO-90 by \algname's model and a model producing a full list. The latter model attends to task-irrelevant objects like plate and bowl, while our model emits reasoning focused on task-critical objects (refer to Appendix~\ref{sec:supplementary_qualitative} for full qualitative comparisons).}
    \label{fig:libero_90_VO}
%\vspace{-8pt}
\end{figure}

%% file: sections/experiment_content/manipulation/bridge.tex
\subsection{Bridge WidowX Hardware}
\label{sec:BridgeResults}

We evaluate our approach on WidowX hardware by training with the Bridge v2 Dataset~\cite{walke2023bridge} a 7B-parameter OpenVLA~\cite{belkhale2024minivla}, which combines a Llama 2~\cite{qwen2025qwen25technicalreport} backbone with DINOv2~\cite{oquab2024dinov} and SigLIP~\cite{zhai2023sigmoid} vision encoders. With \algname, we use a reasoning dropout rate of $d=0.2$, with the same $7$ primitives as in Section~\ref{sec:Libero90Results}, where the reasoning content was generated in~\cite{zawalski2025Robotic} by the Gemini 1.0~\cite{team2024gemini} model. The action $A$ consists of seven discrete tokens based on the tokenization of~\cite{zitkovich2023rt2, kim2025openvla}. To support high-throughput repeated sampling when generating our synthetic reasoning candidate for refinement, we extended SGLang-VLA~\cite{kwok2025robomonkey,zheng2024sglang} to support rapid inferencing of reasoning-based~\cite{zawalski2025Robotic,chen2025training} Prismatic Models~\cite{karamcheti2024prismatic,kim2025openvla}. More details on Bridge experiments are in Appendix~\ref{sec:app_bridge_details}.

\input{floats/tables/visible_object_only}

\input{floats/tables/libero90}

\begin{hypothesis}
\textbf{Q3} How effective can \algname enable manipulation VLAs to generalize when test-time reasoning is suppressed for speed, compared to reasoning with all primitives?

\end{hypothesis}

\noindent
Since generating textual reasoning traces at test-time can drastically increase latency (often taking several seconds per step), we first investigate whether VLAs can benefit from reasoning during training while bypassing the computational cost of generation during deployment. To this end, we employ ``Action Forcing,'' where VLAs are prompted with an Action tag to immediately generate action tokens, effectively suppressing the reasoning output. For models trained with \algname or random primitives, this capability is emergent: because some of their training data has no reasoning content due to dropout, these models can zero-shot generalize to the Action Forcing prompt. In contrast, the VLA trained with all reasoning primitives does not have reasoning-free training data; hence, we post-train this model for some epochs with no reasoning to acquire this ability, following the findings in~\cite{chen2025training}.

\input{floats/tables/hardware_af}

\begin{figure}[t]
    \centering
        \includegraphics[width=\linewidth]{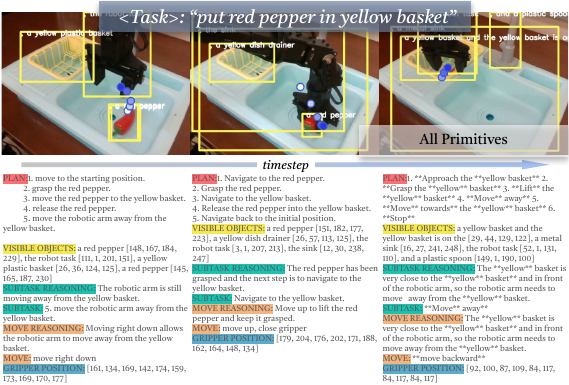}

        \vspace{0.3em}
        \includegraphics[width=\linewidth]{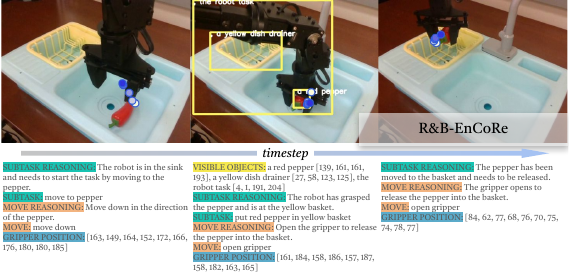}

        \vspace{-4px}
        
        \caption{\small Sample reasoning traces on WidowX hardware. \algname produces more concise and effective reasoning compared to reasoning on all primitives.}
    % Second subfigure with 3 stacked figures
    \label{fig:bridge-reasoning-qual}
    \vspace{-10pt}
\end{figure}

We evaluated four models across 9 tasks totaling 468 trials, categorized into: In-Distribution, OOD Target (novel target objects), and OOD Scene/Distractions (cluttered environments with distracting objects). Task details can be found in the Appendix~\ref{sec:app_hardware_details}. As seen in our results in Table~\ref{tab:hw_af}, for In-Distribution tasks, both \algname and all reasoning primitive models outperform no reasoning and random reasoning primitives. This result confirms that training on valid reasoning traces improves the underlying policy representation, even when that reasoning is not explicitly generated at test-time. For both OOD task groups, reasoning with all primitives suffers a performance drop, suggesting that rigid, ``one-size-fits-all'' templates force the baseline model to attend to specific tasks or scene features (\eg OOD distracting objects with the same color as the target) that may be irrelevant in OOD settings, confounding the policy. \algname's model is robust to these shifts, having learned to prune irrelevant signals, and maintains its success rate. Overall, our approach applied in manipulation VLAs on real-world hardware consistently demonstrates that refining the reasoning distribution during training aligns the latent representation with successful task performance without needing to sacrifice test-time speed.

\input{floats/figures/slow_hw_results}

%\iftoggle{anonymous}{\input{floats/figures/Navitrace_Waypoints} }{}

%\iftoggle{anonymous}{\input{floats/figures/navitrace_nogarbage_perf}}{}

\input{floats/figures/Navitrace_Waypoints}

\input{floats/figures/navitrace_nogarbage_perf}

%\subsubsection{Ablation: WidowX Bridge with Test-Time Reasoning}

\begin{hypothesis}
\textbf{Q4} How does \algname improve task performance and reduce test-time reasoning latency compared to baseline reasoning on all primitives for WidowX hardware?
\end{hypothesis}

\noindent
We perform an ablation study evaluating the performance and latency of explicit test-time reasoning on the WidowX robot, comparing \algname against the baseline trained on all primitives. We provide traces from a trajectory generated by models reasoning with all primitives and our approach in Fig.~\ref{fig:bridge-reasoning-qual}. Results in Fig.~\ref{fig:slowhwresult}a show that while in-distribution performance is comparable, \algname significantly outperforms the baseline in OOD settings with distractions, where the latter's performance degrades by $31\%$. This suggests \algname effectively filters irrelevant strategies that otherwise yield confounding traces in novel environments. Furthermore, Fig.~\ref{fig:slowhwresult}b demonstrates that \algname reduces inference time from $>5$ seconds to $\approx 3$ seconds per step. We observed that slow generation creates control lag, often causing grasped objects to slip during the trajectory. By pruning verbose reasoning, our approach generates concise traces with lower test-time latency and improved task success.

%% file: floats/tables/visible_object_only.tex
\begin{table}[t]
    \centering
    \caption{\small \algname autonomously identifies critical objects in LIBERO-90 data that are most predictive of expert actions, pruning irrelevant objects and improving task success.}

    \vspace{-4pt}
    \label{tab:object_list_stats}
    \setlength{\tabcolsep}{4pt}
    \begin{tabular}{lcc}
        \toprule
        \textbf{Training Strategy} & \textbf{Success Rate} & \textbf{Object Criticality Rate} \\
        \midrule
        No Reasoning & 75.9\% & N/A\\
        List of All Objects & 76.1\% & 0.03\% \\
        List of Random Objects & 77.0\% & 3.43\% \\
        \algname on Object List & \textbf{80.3\%} & \textbf{25.02\%} \\
        \bottomrule
    \end{tabular}

\end{table}

%% file: floats/tables/libero90.tex
\begin{table}[t]
    \centering
    \caption{\small Results on analyzing reasoning primitives in LIBERO-90: Plan, Visible Objects, Subtask, Subtask Explain, Move, Move Explain, and Gripper Position. \algname refines action-predictive reasoning, resulting in shorter traces and improved success rate. }

    \vspace{-4pt}
    \label{tab:liber90_res}
    \begin{tabular}{lccc}
        \toprule
        \textbf{Training Strategy} & \textbf{Success Rate} & \textbf{Avg. \# Gen. Token}\\
        \midrule
        No Reasoning & 75.9\% & 10.0 \\
        All Primitives & 78.6\%  & 256.8\\
        Random Primitives & 76.5\%  & 256.6 \\
        \algname & \textbf{79.5\%} & 129.3 \\
        \bottomrule
    \end{tabular}

%\vspace{-5pt}

\end{table}

%% file: floats/tables/hardware_af.tex
\begin{table*}[h]
\centering
\caption{\small Success rates on WidowX hardware Bridgev2 setup. Reasoning Models are prompted and/or trained with Action forcing for reduced latency~\cite{chen2025training}. In total for the experiments presented in this table, we had $4\text{ models}\times9\text{ tasks}\times13\text{ trials}=468\text{ total trials}$.}
\vspace{-4pt}
\label{tab:hw_af}
\setlength{\arrayrulewidth}{0.4pt}
\renewcommand{\arraystretch}{1.2} % Reduced stretch since cellspace handles it now
\resizebox{\textwidth}{!}{%
\begin{tabular}{Sl Sl Sc Sc Sc Sc} % Added S prefix for consistent spacing
\hline
\textbf{Category} & \textbf{Task} & \textbf{No Reason} & \textbf{All Primitives} & \textbf{Random} & \textbf{R\&B-EnCoRe} \\ \hline
\multirow{3}{*}{In Distrib.} & put red pepper in yellow basket & 69.2\% & \textbf{100.0\%} & 92.3\% & \textbf{100.0\%} \\ 
 & \cellcolor{gray!20}put red pepper on black stove & \cellcolor{gray!20}53.8\% & \cellcolor{gray!20}\textbf{92.3\%} & \cellcolor{gray!20}46.2\% & \cellcolor{gray!20}84.6\% \\ 
 & put orange carrot on pink plate & 61.5\% & 76.9\% & \textbf{84.6\%} & \textbf{84.6\%} \\ \hline
\multirow{3}{*}{OOD Target Object} & \cellcolor{gray!20}put blue peacock in sink & \cellcolor{gray!20}46.2\% & \cellcolor{gray!20}38.5\% & \cellcolor{gray!20}69.2\% & \cellcolor{gray!20}\textbf{76.9\%} \\ 
 & \makecell[l]{put orange tape on green towel \\ \emph{(include orange carrot)}} & 38.5\% & 38.5\% & 38.5\% & \textbf{76.9\%} \\
 & \cellcolor{gray!20}put pink pepto on green towel & \cellcolor{gray!20}46.2\% & \cellcolor{gray!20}30.8\% & \cellcolor{gray!20}53.8\% & \cellcolor{gray!20}\textbf{69.2\%} \\ \hline
\multirow{6}{*}{\makecell[l]{OOD Scene with \\ Distracting Objects}} & \makecell[l]{put yellow corn on blue plate \\ \emph{(include pink plate, carrot)}} & 53.8\% & 69.2\% & 38.5\% & \textbf{76.9\%} \\ 
 & \cellcolor{gray!20}\makecell[l]{put orange carrot in yellow basket \\ \emph{(include distraction objects in basket, sink)}} & \cellcolor{gray!20}61.5\% & \cellcolor{gray!20}\textbf{92.3\%} & \cellcolor{gray!20}84.6\% & \cellcolor{gray!20}\textbf{92.3\%} \\ 
 & \makecell[l]{first put yellow corn in yellow basket then put \\ red pepper in yellow basket \emph{(human takes corn)}} & 46.2\% & 30.8\% & 46.2\% & \textbf{61.5\%} \\ \hline
\end{tabular}
}
\vspace{-10pt}
\end{table*}

%% file: floats/figures/slow_hw_results.tex
\begin{figure}[t]
    \centering
    \includegraphics[width=\linewidth]{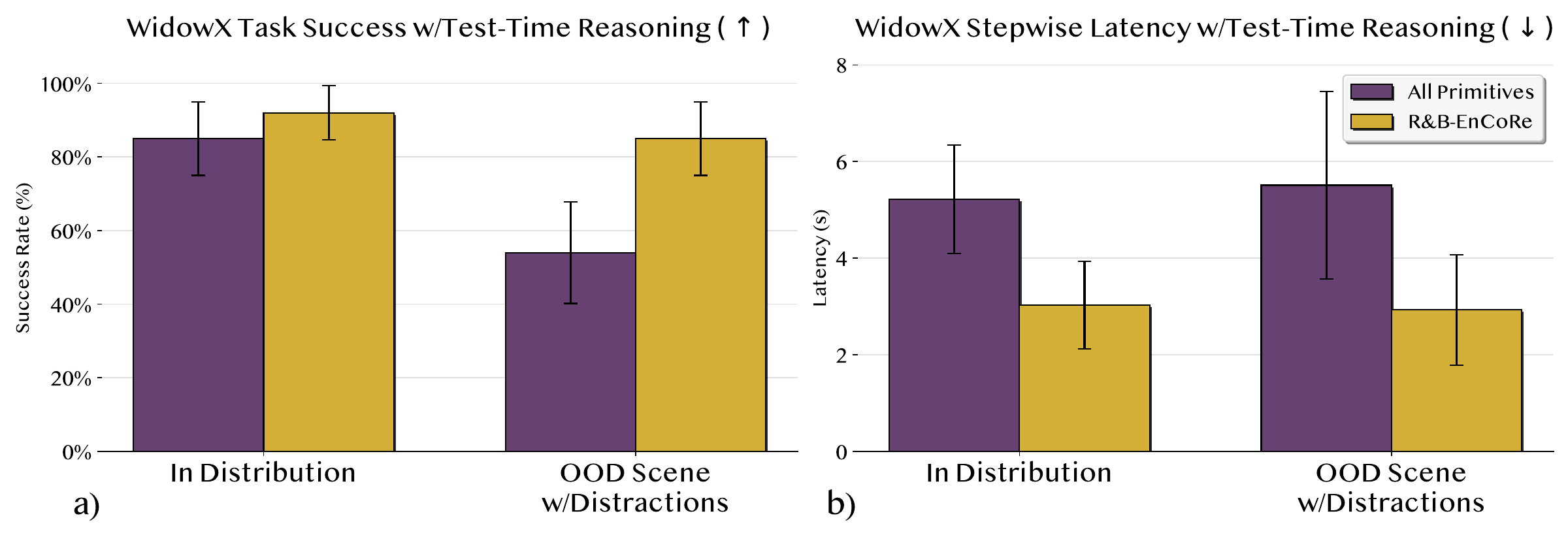}

    %\vspace{-8pt}
    \caption{\small Success rates and latency of test-time reasoning on WidowX hardware. \algname produces performant reasoning VLAs with shorter reasoning traces (so faster inference). Reasoning on all primitives degrades performance for cluttered scenes with OOD objects.}
    \label{fig:slowhwresult}
\vspace{-10pt}
\end{figure}

%% file: floats/figures/Navitrace_Waypoints.tex
\begin{figure*}[ht]
    \centering
    \includegraphics[width=\textwidth]{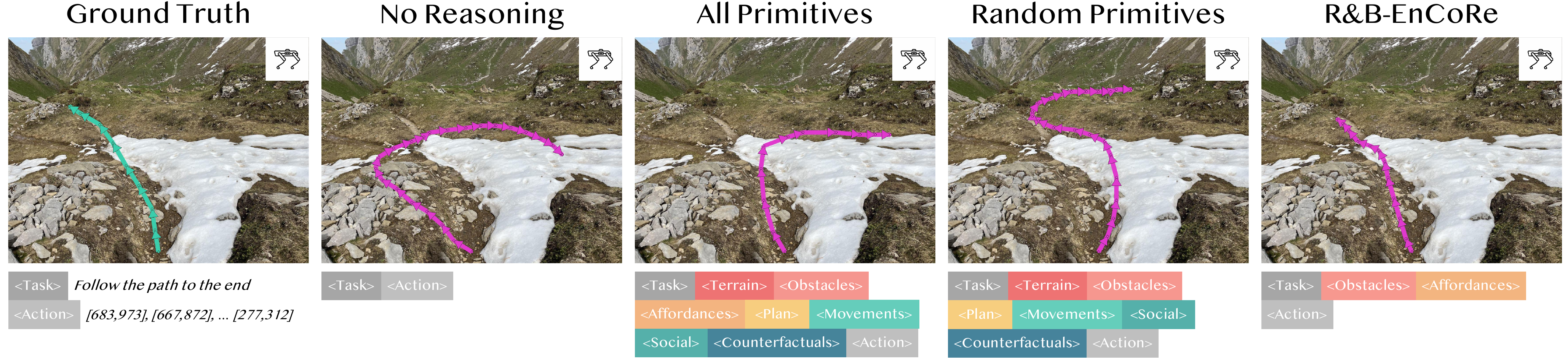}

    %\vspace{-8pt}
    \caption{\small Quadruped Navigation Waypoint Trajectories. The quadruped robot must follow the trail while avoiding slippery ice. No Reason navigation VLA ignores terrain hazards and traverses the ice. Reasoning with all primitives is confounded by irrelevant signals; while it has reduced ice contact (perhaps due to affordance reasoning), it fails to follow the path. Random Primitives tracks some of the path but likely due to lack of affordance reasoning, it traces through a lot of the slippery ice. \algname identifies the effective reasoning strategy, with minimal ice contact while maintaining the path, matching the Ground Truth. This example was taken from holdout task set.}

\vspace{-12pt}
\end{figure*}

%% file: floats/figures/navitrace_nogarbage_perf.tex
\begin{figure}[t]
    \centering
    \includegraphics[width=\linewidth]{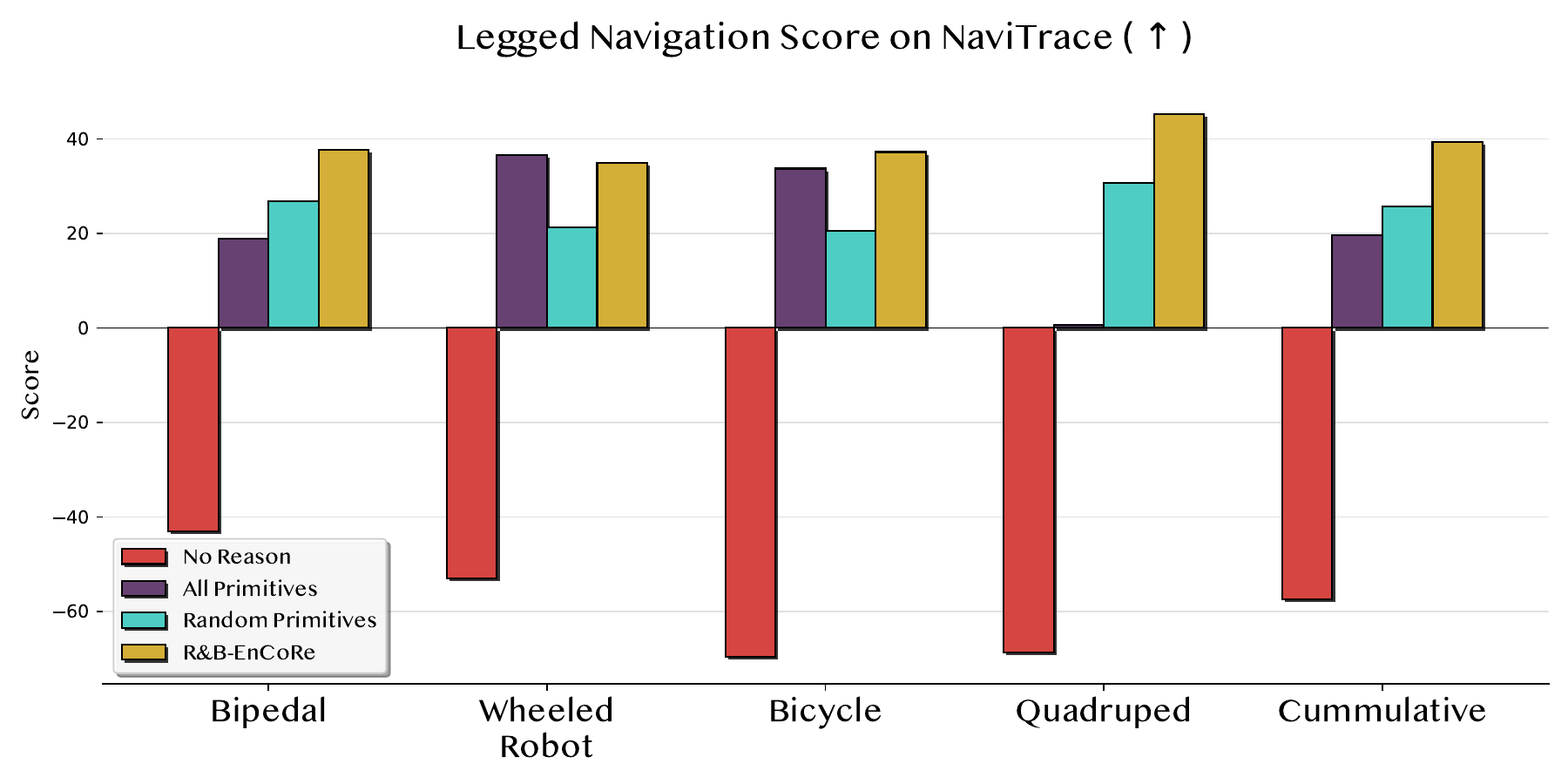}

    %\vspace{-8pt}
    \caption{\small The NaviTrace score is a normalized metric so 100 is perfect path alignment with expert, and 0 is the score for the na\"ive straight line path down the middle of the scene. \algname generally performs best across embodiments on navigation metric, improving the cumulative score from 19.6 (all primitives) to 39.4 (\algname). For comparison, native CoT of the Qwen3-VL-30b model (\ie querying zero-shot) achieves a score of $-$260.}
    \label{fig:navitrace_nogarbage_score}
\vspace{-10pt}
\end{figure}

%% file: sections/experiment_content/quadruped.tex
\vspace{-5pt}
\subsection{Legged Robots Navigation}
\label{sec:NavigationResults}

We evaluate the capabilities of \algname in the context of legged robot navigation using the NaviTrace dataset~\cite{windecker2025navitrace}, which contains approximately 1000 tasks across 500 unique scenes for four embodiments: bipedal, wheeled, bicycle, and quadruped robots. Evaluation is performed on a holdout subset of these tasks and scenes using a downstream task metric from~\cite{windecker2025navitrace} that incorporates Dynamic Time Warping distance, goal endpoint error, and semantic penalties correlating with human preferences. We finetune a Qwen3-VL-30B-A3B-Instruct~\cite{bai2025qwen3vltechnicalreport} Mixture-of-Experts model to process scene images and task specifications, outputting 2D waypoint coordinates. To enable fully self-reliant learning, we construct the reasoning dataset by querying the VLM itself via VQA (so $\text{FM}=\mathcal{M}$ in Alg.~\ref{alg:rbencore_train}) to generate traces for seven hypothesized reasoning primitives, \eg terrain, affordances, social norms.% for the ground truth paths.

\begin{hypothesis}
\textbf{Q5} How does \algname self-bootstrap legged locomotion navigation VLA performance when refining and learning from reasoning generated by its base VLM?
\end{hypothesis}

\noindent
By applying \algname with dropout $d=0.5$ to these self-generated traces, the model refines its own reasoning distribution to improve the downstream navigation score (Fig.~\ref{fig:navitrace_nogarbage_score}). As in Fig.~\ref{fig:mainresult}b, we observe that the refined distribution heavily prioritizes reasoning about actionable capabilities of the environment (Affordances) and proposed movements (Move), identifying these as critical for successful navigation. We also observe that counterfactual reasoning is not always necessary, dropping significantly in frequency, confirming recent findings~\cite{glossop2025cast, peng2025counterfactual, liu2023reflect, huang2025thinkact} that counterfactual reasoning is only reserved for specific decision moments. These results provide evidence for significant promise that frontier VLMs can self-bootstrap performance by synthetically generating VQAs on embodied reasoning and self-refining their data.

\begin{hypothesis}
\textbf{Q6} How effective is \algname in pruning patently irrelevant reasoning primitives like subjective descriptions of weather for legged navigation?
\end{hypothesis}

\input{floats/figures/navitrace_distrib_wgarbage}

\noindent
We introduce an uninformative reasoning primitive, subjective description of the weather, to the training data. Analyzing the reasoning distribution produced by \algname in Fig.~\ref{fig:garbage_distrib}, we see our approach successfully identifies ``weather" descriptions as largely irrelevant to the navigation task, pruning this primitive significantly more aggressively than others. Conversely, the model retains reasoning related to obstacles, affordances, and social norms. This refined primitive distribution provides evidence that \algname can effectively filter out patently irrelevant information, generating high-quality relevant reasoning strategies tailored to the specific needs of various embodiments of legged locomotion navigation.

%% file: floats/figures/navitrace_distrib_wgarbage.tex
\begin{figure}[t]
    \centering
    \includegraphics[width=\linewidth]{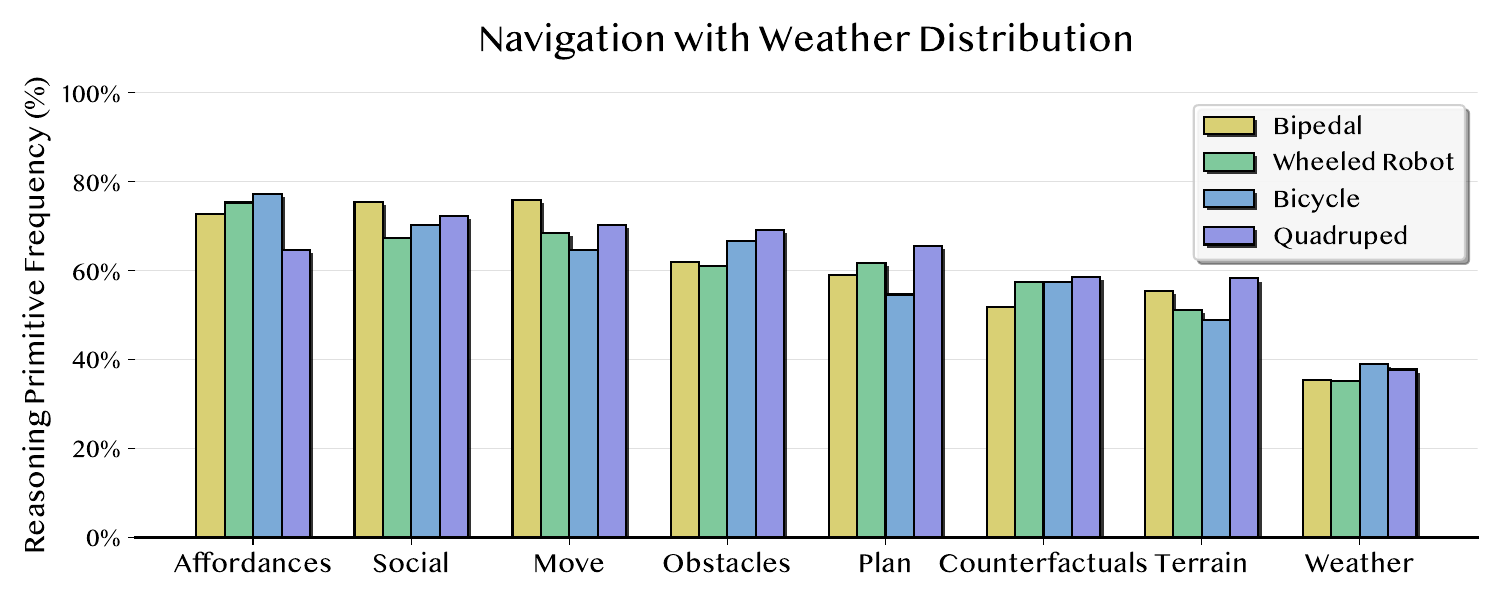}

    %\vspace{-8pt}
    \caption{\small \algname prunes uninformative subjective weather reasoning from refined traces ($\sim$36.7\%; lower than other primitives).}
    \label{fig:garbage_distrib}
\vspace{-10pt}
\end{figure}

%% file: sections/experiment_content/driving.tex
\vspace{-5pt}
\subsection{Autonomous Vehicles}
\label{sec:AVResults}

We extend \algname to the autonomous vehicle (AV) nuScenes dataset~\cite{caesar2020nuscenes} and study how reasoning traces for driving VLAs changes under our method. We leverage traces from pioneering LLM agent-based planners~\cite{mao2023gpt, maolanguage}, repurposed for training the reasoning component of a driving VLA. Specifically, we finetune a Qwen3-VL-4B-Instruct Dense Model~\cite{bai2025qwen3vltechnicalreport} to take in the front camera image, and output the ego-vehicle's planning trajectory over 3 seconds. We report evaluation using the UniAD~\cite{hu2023planning} planning metrics on the held-out validation set (refer to Appendix~\ref{sec:app_av_details} for evaluation details). In this section, we aim to refine these commonly used reasoning traces to improve downstream performance.

\begin{hypothesis}
\textbf{Q7} How does \algname further refine human-crafted reasoning for Autonomous Vehicle data?
\end{hypothesis}

\input{floats/tables/av_results}

\input{floats/figures/nuscenes_qualitative}

\input{floats/figures/collision_rate_scaling_law}

\noindent
To evaluate \algname with reasoning dropout of $d=0.5$ against handcrafted reasoning traces, we compare our method to using all reasoning primitives from~\citep{maolanguage}, a random reasoning strategy, and no reasoning (Table \ref{tab:av_performance_comparison}). While all reasoning components reduce collision rate, they slightly increase L2 error compared to no reasoning, suggesting irrelevant information may harm trajectory prediction performance. \algname is able to reduce irrelevant components (Fig.~\ref{fig:mainresult}c) and subsequently improve downstream trajectory prediction (Fig.~\ref{fig:nuscenes_qualitative}), suggesting that the embodied domain of autonomous vehicle VLAs can benefit from reasoning trace refinement, which has traditionally been handcrafted~\cite{hwang2024emma,nvidia2025alpamayor1bridgingreasoningaction}.

\begin{hypothesis}
\textbf{Q8} How does \algname's performance for AV scale with respect to the number of posterior samples?
\end{hypothesis}

\noindent
To understand how performance scales with the number of posterior samples, we ablate the variational inference parameter $K$ in Fig.~\ref{fig:scaling_law}. We observe that increasing the number of posterior samples improves performance on the collision rate metric, with performance saturating around 16 samples, substantiating that the multi-sampling IWAE process used by \algname yields tighter bounds with increased samples.

%% file: floats/tables/av_results.tex
\begin{table}[t]
    \centering
    \caption{\small Performance comparison of AV with UniAD metrics. \algname is able to refine reasoning traces and reduce deviation from ground truth path and collision rate. For comparison, zero-shot native chain-of-thought achieves suboptimal 10.35m average L2 error.}

    \vspace{-4pt}
    \label{tab:av_performance_comparison}
    
    \setlength{\tabcolsep}{4pt}
    \begin{tabular}{lcccc|cccc}
        \toprule
        & \multicolumn{4}{c}{\textbf{L2 Path Error (m)}} & \multicolumn{4}{c}{\textbf{Collision Rate (\%)}} \\
        \cmidrule(lr){2-5} \cmidrule(lr){6-9}
        \textbf{Method} & \textbf{1s} & \textbf{2s} & \textbf{3s} & \textbf{Avg} & \textbf{1s} & \textbf{2s} & \textbf{3s} & \textbf{Avg} \\
        \midrule
        No Reasoning & 0.22 & 0.62 & 1.34 & 0.72 & 0.10 & 0.25 & 1.11 & 0.49 \\
        Full Reasoning & 0.22 & 0.64 & 1.33 & 0.73 & 0.07 & \textbf{0.17} & 0.91 & 0.38 \\
        Random & 0.22 & 0.62 & 1.31 & 0.72 & \textbf{0.05} & 0.20 & 0.81 & 0.35 \\
        \algname & \textbf{0.21} & \textbf{0.59} & \textbf{1.25} & \textbf{0.68} & \textbf{0.05} & \textbf{0.17} & \textbf{0.70} & \textbf{0.30} \\
        \bottomrule
    \end{tabular}
%\vspace{-15pt}
\end{table}

%% file: floats/figures/nuscenes_qualitative.tex
\begin{figure}[t]
%\vspace{-2pt}
    \centering
    \includegraphics[width=\linewidth]{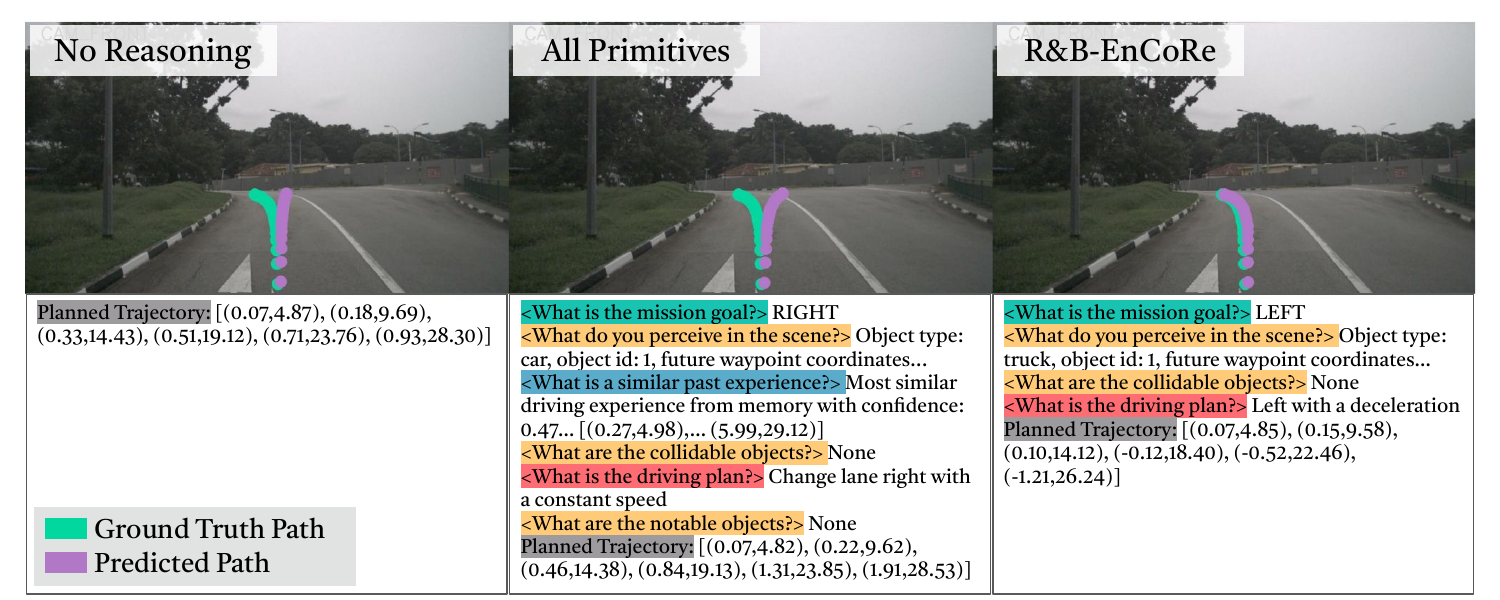}

    %\vspace{-4pt}
    \caption{\small Planned trajectories comparing driving VLAs using reasoning by \algname's model and a model producing a full list.}
    \label{fig:nuscenes_qualitative}
\vspace{-8pt}
\end{figure}

%% file: floats/figures/collision_rate_scaling_law.tex
\begin{figure}[t]
    \centering
    \includegraphics[width=\linewidth]{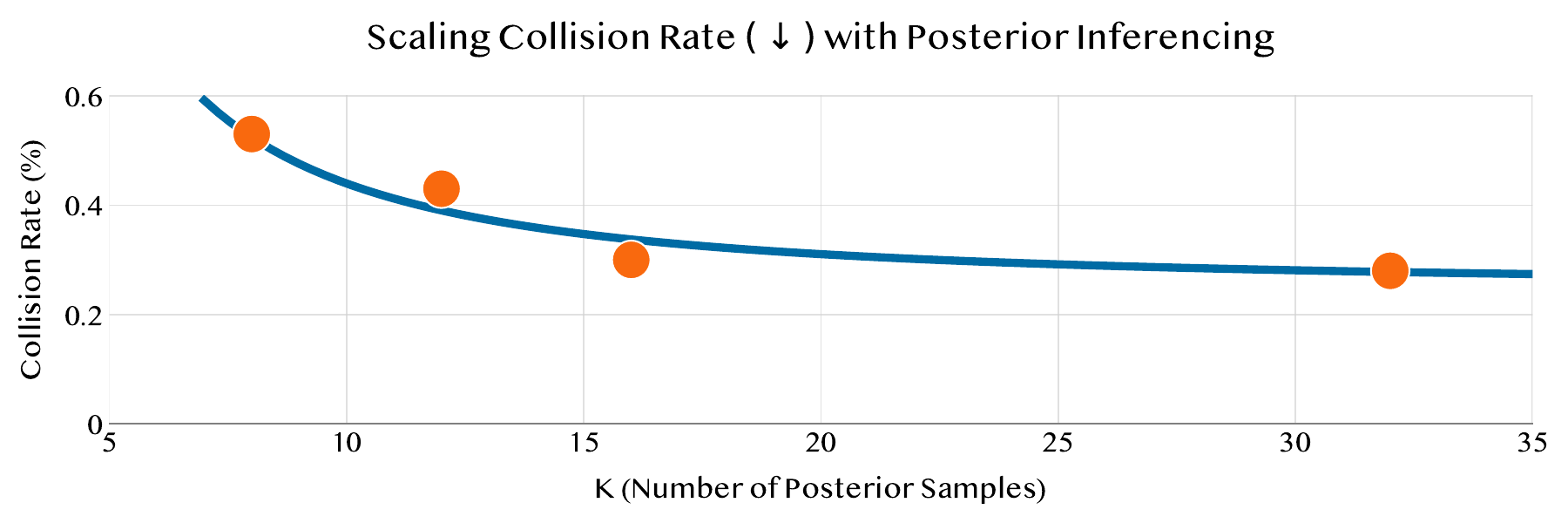}\vspace{-8pt}
    \caption{\small Collision Rate scaling with posterior inference. More samples $K$ from posterior distribution results in improved action prediction estimate, and ultimately lower collision rate. Scaling curve fitted with $\text{Collision Rate}=(3.65/K)^{1.65}+0.25$.}
    \label{fig:scaling_law}
\vspace{-10pt}
\end{figure}

%% file: sections/discussion_conclusion.tex
\section{Discussion and Conclusion}

We introduced \algname, which treats embodied reasoning as a latent variable to ground internet-scale priors in physical control. Using importance-weighted variational inference, \algname autonomously filters chain-of-thought traces to identify expert action-predictive, embodiment-specific reasoning without external supervision. Across manipulation, legged navigation, and autonomous driving, our method significantly outperforms fixed baselines. Notably, \algname's effectiveness improves with the diversity and coverage of warmstart primitives rather than their individual quality: action-uninformative primitives are automatically down-weighted by the importance-sampling step, and in the limit where no primitive carries predictive signal the variational objective (see Appendix~\ref{sec:prop} and~\ref{sec:iwae_theory}) converges to an empty reasoning trace. While our approach trains an additional prior-posterior distribution model, the posterior effectively functions as a high-quality synthetic data generator, allowing for controlled, reward-free exploration of reasoning strategies that can further enhance policy robustness. Future research can investigate applying this framework to continual learning settings, as well as prompting mechanisms that are more native to the base VLM architecture, viewing the posterior as an answer explainer and the prior as a solution reasoning-answer generator. Such approach could further improve \algname scalability for larger models where warmstart training may be computationally prohibitive.

%% file: sections/acknowledgments.tex
\section*{Acknowledgments}

This work was supported by NASA University Leadership Initiative, Toyota Research Institute (TRI), Defense Advanced Research Projects Agency (DARPA), Stanford Center for Automated Reasoning, Stanford Marlowe GPU Cluster~\cite{kapfer2025marlowe}, Stanford Google Cloud Platform, DoD High Performance Computing Modernization Program, Thinking Machines Lab, the National Artificial Intelligence Research Resource Pilot and the Anvil supercomputer (award NSF-OAC 2005632). This article solely reflects the opinions and conclusions of its authors and not of the aforementioned supporting entities.

%% file: appendix/appendix.tex
\begin{appendices}

\onecolumn
\iftoggle{appendix_only}{}{
\textcolor{violet}{\tableofcontents}
}

\input{appendix/notations}

\input{appendix/theorem}

\input{appendix/alg_details}

\input{appendix/implementation}

\input{appendix/main_exp}

\input{appendix/qualitative_results}

\input{appendix/more_ablation_experiments}

\end{appendices}

%% file: appendix/notations.tex
\section{Notation Table}

\textcolor{white}{ }

\begin{table}[h]
\centering
\caption{\small Notation and Symbols used in paper}
\label{tab:notation}
\small
\begin{tabular}{@{}cl@{}}
\toprule
\textbf{Symbol} & \textbf{Description} \\
\midrule
\multicolumn{2}{@{}l}{\textit{Core Variables}} \\
$C$ & Context (observation image and task description) \\
$A$ & Action (discrete action tokens, 2D waypoints, etc) \\
$Z$ & Latent reasoning trace (unobserved strategy) \\
$R$ & Reasoning primitve (\eg Affordance reasoning) \\
$\rho$ & Number of reasoning primitive types \\
$\mathcal{R}$ & Set of reasoning primitives $\{R_1, \ldots, R_\rho\}$ \\
$\mathbf{R}$ & Reasoning Strategy, subset of reasoning primitives \\
$z_R$ & Textual content for reasoning primitive $R\in\mathcal{R}$ \\
\midrule
\multicolumn{2}{@{}l}{\textit{Probabilistic Models}} \\
$p(A \mid  C)$ & Marginal action distribution given context \\
$p(Z, A \mid  C)$ & Prior distribution (joint over reasoning and action) \\
$p(A \mid  C, Z)$ & Action likelihood given context and reasoning \\
$p(Z \mid  C)$ & Prior reasoning distribution \\
$q(Z \mid  C, A)$ & Posterior distribution (reasoning given action) \\
$p_{\text{data}}(A\mid C)$ & Expert action distribution \\
\midrule
\multicolumn{2}{@{}l}{\textit{Variational Inference}} \\
$K$ & Number of posterior samples in IWAE \\
$Z_1, \ldots, Z_K$ & $K$ i.i.d. samples from posterior $q(Z\mid C,A)$ \\
$w_k,w(Z_k)$ & Importance weight $\frac{p(Z_k, A\mid C)}{q(Z_k\mid C,A)}$ \\
$\mathcal{L}_K$ & $K$-sample importance-weighted lower bound \\
$\text{ELBO}_{\text{VAE}}$ & Evidence lower bound for standard VAE \\
$D_{\text{KL}}(\cdot \|  \cdot)$ & Kullback-Leibler divergence \\
\midrule
\multicolumn{2}{@{}l}{\textit{Datasets and Training}} \\
$\mathcal{D}$ & Original dataset: $\{(C^i, A^i)\}_{i=1}^N$ \\
$N$ & Number of demonstrations in dataset \\
$\mathcal{D}_{\text{warm}}$ & Warmstart dataset with diverse reasoning traces \\
$\mathcal{D}_{\text{refined}}$ & Refined dataset with action-predictive reasoning \\
$M$ & Number of warmstarting traces per demonstration \\
$d$ & Reasoning dropout rate \\
$Z^i_j$ & $j$-th synthetic reasoning trace for demonstration $i$ \\
$Z^{i*}$ & Refined, importance-sampled reasoning trace  \\
\midrule
\multicolumn{2}{@{}l}{\textit{Models}} \\
$\mathcal{M}$ & Base vision-language model \\
$\mathcal{M}_{pq}$ & Jointly trained prior-posterior model \\
$\mathcal{M}_{\text{VLA}}$ & Final vision-language-action model \\
$\text{FM}$ & Foundation model for generated reasoning content \\
\midrule
\multicolumn{2}{@{}l}{\textit{Information Benefit}} \\
$\Delta I_R$ & Information benefit of reasoning strategy $R$ \\
$\mathcal{Z}_R$ & Set of reasoning traces containing strategy $R$ \\
$\mathcal{Z}_{\cancel{R}}$ & Set of reasoning traces not containing strategy $R$ \\
$\E_{Z \sim q}[w(Z) \mid Z \in \mathcal{Z}_{\mathbf{R}}]$ & Expected importance weight of reasoning traces in set $\mathcal{Z}_R$ \\
\bottomrule
\end{tabular}
\end{table}

%% file: appendix/theorem.tex
\newpage
\section{Proposition}
\label{sec:prop}

\begin{prop}[Importance Weights Capture Information Gain] Let $A \sim p_{\text{data}}(A\mid C)$ be the observed ground-truth action and $Z$ be a latent reasoning variable. We partition the reasoning space into disjoint sets $\mathcal{Z}_{\mathbf{R}}$ and $\mathcal{Z}_{\cancel{\mathbf{R}}}$.

We assume that we construct our latent reasoning dataset $\mathcal{Z}=\mathcal{Z}_{\mathbf{R}} \cup \mathcal{Z}_{\cancel{\mathbf{R}}}$ so the latents $Z_{\cancel{\mathbf{R}}}$ without reasoning primitive or more generally strategy $\mathbf{R}$ are sampled with constant ``dropout probability" $d$ and each latent $Z_{\mathbf{R}}$ with reasoning $\mathbf{R}$ with constant probability $1-d$, %and then we train both the prior latent generation $p(Z)$ and posterior $q(Z\mid A)$ on the identical distribution over $(\mathcal{Z},A)$ so at convergence $p(\mathcal{Z}_{\mathbf{R}}) = (1-\pi_R) = q(\mathcal{Z}_{\mathbf{R}}|A)$ and likewise $p(\mathcal{Z}_{\cancel{\mathbf{R}}}) = \pi_R = q(\mathcal{Z}_{\cancel{\mathbf{R}}}|A)$ for all $A$.
and then we train both the prior latent generation $p(Z\mid C)$ and posterior $q(Z\mid C,A)$ on the constructed dataset. Crucially, because Alg.~\ref{alg:rbencore_train} applies reasoning dropout independently of the action $A$, the conditional probability of the reasoning partition in the ground-truth data is constant. Therefore, assuming the models $p$ and $q$ converge to the underlying data distribution during training in Alg.~\ref{alg:rbencore_refine}, we have the point-wise equality for all $A$:$$p(\mathcal{Z}_{\mathbf{R}}\mid C) = q(\mathcal{Z}_{\mathbf{R}}\mid C,A) = (1-d),$$
$$p(\mathcal{Z}_{\cancel{\mathbf{R}}}\mid C) = q(\mathcal{Z}_{\cancel{\mathbf{R}}}\mid C,A) = d.$$

We define the information benefit of of a reasoning strategy
$\mathbf{R}$ as how much it reduces the divergence between our model’s action distribution and the expert’s distribution $p_{data}(A\mid C)$:
$$\Delta \mathcal{I}_{\mathbf{R}} \doteq D_{KL}(p_{\text{data}} \| p(A\mid C,\mathcal{Z}_{\cancel{\mathbf{R}}})) - D_{KL}(p_{\text{data}} \| p(A\mid C,\mathcal{Z}_{\mathbf{R}}))$$

Then, the expected log-ratio of the importance weights (defined as $w(Z) \doteq \frac{p(A, Z\mid C)}{q(Z\mid C,A)}$) assigned to reasoning traces $\mathcal{Z}_{\mathbf{R}}$ with the strategy versus reasoning traces $\mathcal{Z}_{\cancel{\mathbf{R}}}$ without the strategy is:
\begin{equation*}
\E_{A \sim p_{\text{data}}}  \log \frac{\E_{Z \sim q}[w(Z) \mid Z \in \mathcal{Z}_{\mathbf{R}}]}{\E_{Z \sim q}[w(Z) \mid Z \in \mathcal{Z}_{\cancel{\mathbf{R}}}]}  = \Delta \mathcal{I}_{\mathbf{R}}.
\end{equation*}
\end{prop}

%\vspace{-10pt}

\begin{mdframed}[style=ThmFrame]

\begin{proof} The expected importance weight for reasoning traces from the partition $\mathcal{Z}_{\mathbf{R}}$ is:
\begin{align*}
W_{\mathbf{R}} &\doteq \E_{Z \sim q}[w(Z) \mid Z \in \mathcal{Z}_{\mathbf{R}}] \\
&= \int_{\mathcal{Z}_{\mathbf{R}}} \frac{p(A \mid C,Z)p(Z\mid C)}{q(Z\mid C,A)} q(Z \mid C,A, \mathcal{Z}_{\mathbf{R}} ) dZ
\end{align*}
Using identity $q(Z\mid C, A) = q(Z \mid C, A,\mathcal{Z}_{\mathbf{R}}) q(\mathcal{Z}_{\mathbf{R}}\mid C,A)$ for $Z \in \mathcal{Z}_{\mathbf{R}}$:
\begin{align*}
W_{\mathbf{R}} & = \int_{\mathcal{Z}_{\mathbf{R}}} \frac{p(A \mid C, Z)p(Z\mid C)}{q(Z \mid C,A,\mathcal{Z}_{\mathbf{R}}) q(\mathcal{Z}_{\mathbf{R}}\mid C,A)} q(Z \mid C,A,\mathcal{Z}_{\mathbf{R}}) dZ \\
&= \frac{1}{q(\mathcal{Z}_{\mathbf{R}}\mid C,A)} \int_{\mathcal{Z}_{\mathbf{R}}} p(A \mid C,Z) p(Z\mid C) dZ
\end{align*}
This integral is the joint probability of the data and the partition, $p(A, \mathcal{Z}_{\mathbf{R}}\mid C)$.
$$
W_{\mathbf{R}} = \frac{p(A, \mathcal{Z}_{\mathbf{R}}\mid C)}{q(\mathcal{Z}_{\mathbf{R}}\mid C,A)} = \frac{p(A \mid C,\mathcal{Z}_{\mathbf{R}}) p(\mathcal{Z}_{\mathbf{R}}\mid C)}{q(\mathcal{Z}_{\mathbf{R}}\mid C,A)}
$$
Similarly for the $\mathcal{Z}_{\cancel{\mathbf{R}}}$ partition:
$$
W_{\cancel{\mathbf{R}}} = \frac{p(A \mid C,\mathcal{Z}_{\cancel{\mathbf{R}}}) p(\mathcal{Z}_{\cancel{\mathbf{R}}}\mid C)}{q(\mathcal{Z}_{\cancel{\mathbf{R}}}\mid C,A)}
$$

Now we look at the ratio $W_{\mathbf{R}} / W_{\cancel{\mathbf{R}}}$. This allows us to separate the Likelihood term (predictive power) from the Prior/Posterior reasoning distribution.
$$
\frac{W_{\mathbf{R}}}{W_{\cancel{\mathbf{R}}}} = \underbrace{\frac{p(A \mid C,\mathcal{Z}_{\mathbf{R}})}{p(A \mid C,\mathcal{Z}_{\cancel{\mathbf{R}}})}}_{\text{Likelihood Ratio}} \times \underbrace{\left( \frac{p(\mathcal{Z}_{\mathbf{R}}\mid C) / q(\mathcal{Z}_{\mathbf{R}}\mid C,A)}{p(\mathcal{Z}_{\cancel{\mathbf{R}}}\mid C) / q(\mathcal{Z}_{\cancel{\mathbf{R}}}\mid C,A)} \right)}_{\text{$=1$ (from assumption)}}
$$

Taking the log of the ratio and expectation over $A \sim p_{\text{data}}$, and recalling the definition of cross-entropy and KL $\E_{p}[\log q] = -H(p) - D_{KL}(p \| q)$, where $H(\cdot)$ measures entropy of a distribution:

\begin{align*}
\E_{A \sim p_{\text{data}}}  \log \frac{\E_{Z \sim q}[w(Z) \mid Z \in \mathcal{Z}_{\mathbf{R}}]}{\E_{Z \sim q}[w(Z) \mid Z \in \mathcal{Z}_{\cancel{\mathbf{R}}}]} &= \E_{A \sim p_{\text{data}}} \left[ \log \frac{W_{\mathbf{R}}}{W_{\cancel{\mathbf{R}}}} \right]\\
&= \E_{A \sim p_{\text{data}}} \left[ \log \frac{p(A \mid C,\mathcal{Z}_{\mathbf{R}})}{p(A \mid C,\mathcal{Z}_{\cancel{\mathbf{R}}})} \right]  \\
&= \E_{A \sim p_{\text{data}}} [\log p(A \mid C,\mathcal{Z}_{\mathbf{R}})] - \E_{A \sim p_{\text{data}}} [\log p(A \mid C,\mathcal{Z}_{\cancel{\mathbf{R}}})]  \\
&= \left( -H(p_{\text{data}})-D_{KL}(p_{\text{data}} \| p(A \mid C,\mathcal{Z}_{\mathbf{R}})) \right) \\
&\text{ }\text{ }\text{ }\text{ }- \left( -H(p_{\text{data}})-D_{KL}(p_{\text{data}} \| p(A \mid C,\mathcal{Z}_{\cancel{\mathbf{R}}})) \right)  \\
&= D_{KL}(p_{\text{data}} \| p(A \mid C,\mathcal{Z}_{\cancel{\mathbf{R}}})) - D_{KL}(p_{\text{data}} \| p(A \mid C,\mathcal{Z}_{\mathbf{R}}))\\
&= \Delta \mathcal{I}_{\mathbf{R}}
\end{align*}

\end{proof}

\end{mdframed}

\subsection{Empirical Support for Assumption}

In the proposition, one assumption we make is that the probability of a reasoning traces with strategy $\mathbf{R}$ from the trained prior $p(\mathcal{Z}_{\mathbf{R}}\mid C)$ and posterior $ q(\mathcal{Z}_{\mathbf{R}}\mid C,A)$ is $(1-d)$, which is 1 minus the dropout rate. We make this assumption because of how we train in Alg.~\ref{alg:rbencore_train} the prior and posterior models: specifically they both are trained with the same underlying distribution of reasoning data where the primitives are dropped out with probability $d$. We empirically validate this assumption by investigating the reasoning primitive frequencies from the raw prior and posterior distributions in Fig.~\ref{fig:raw_posterior_prior} for the LIBERO-90 reasoning setting where the dropout rate $d=0.2$. Overall we observe that the reasoning primitive distributions are quite close to $1-d=0.8$.

\input{floats/figures/raw_prior_posterior}

%% file: floats/figures/raw_prior_posterior.tex
\begin{figure}[h]
    \centering
    \includegraphics[width=\linewidth]{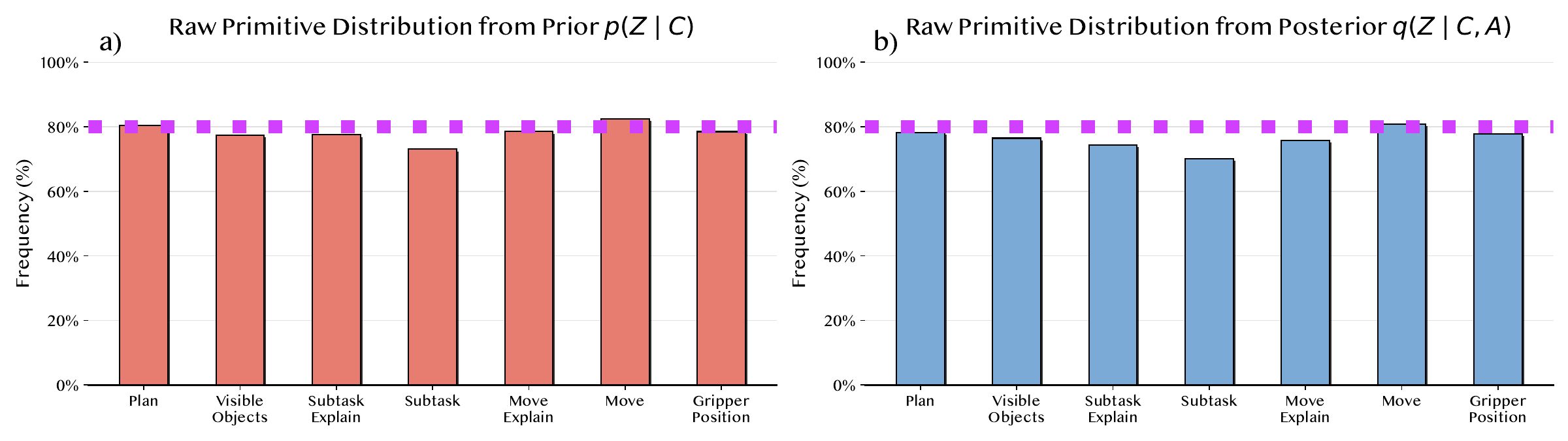}
    \caption{\small Reasoning primitive distributions from the \textbf{raw poseterior and prior} distributions. Note this is before the reweighting and importance sampling step (minor gaps due to warmstarting sampling noise and potential base model prior bias).}
    \label{fig:raw_posterior_prior}
\end{figure}

%% file: appendix/alg_details.tex
\iffalse

\section{Details on Algorithm}

pseudo code for implementing algorithm

\fi

\input{appendix/SIR_c}

%% file: appendix/SIR_c.tex
\section{Importance-Weighted Variational Inference with Categorical Resampling}
\label{sec:iwae_theory}

In the main text, we introduced the Importance Weighted Autoencoder (IWAE) framework~\cite{burda16importance}, which provides a tighter lower bound on the log-evidence through multiple importance-weighted samples. Our algorithm \algname employs an improved variant of this approach that uses categorical resampling from importance weights, following the theoretical developments in \citet{cremer2017reinterpreting}. This section provides the theoretical justification for this approach and relates it to our method based on sampling-importance-resampling (SIR).

\subsection{From IWAE to SIR IWAE}

Recall from Section~\ref{sec:variationalinference} that the $K$-sample IWAE bound is defined as:
\begin{equation}
    \mathcal{L}_K \doteq \mathbb{E}_{Z_1,\ldots,Z_K \sim q(Z|A)} \left[ \log \left( \frac{1}{K} \sum_{k=1}^K w_k \right) \right],
\end{equation}
where $w_k = \frac{p(Z_k, A)}{q(Z_k | A)}$ are the importance weights.

A closer look at this bound shows we can apply sampling-importance-resampling (SIR)~\cite{bishop2006pattern}. Consider the following procedure:
\begin{enumerate}
    \item Sample $K$ candidates $\{Z_k\}_{k=1}^K$ from the proposal distribution $q(Z|A)$
    \item Compute importance weights $w_k$ for each candidate
    \item Resample a single latent $Z^*$ from a categorical distribution proportional to the weights: $Z^* \sim \text{Cat}(\{w_k\}_{k=1}^K)$
\end{enumerate}

This resampling procedure induces an improved proposal distribution $\tilde{q}(Z|A)$ that is defined implicitly through the SIR process. Critically, \citet{cremer2017reinterpreting} showed that this categorical resampling can be understood as sampling from a distribution that achieves a tighter IWAE bound:

\begin{prop}[Categorical Resampling Bound Adapted from~\cite{cremer2017reinterpreting}]
\label{prop:categorical_equivalence}
Let $\tilde{q}(Z\mid A)$ denote the distribution induced by the categorical resampling procedure described above. Then, as $K \to \infty$, we have that $\tilde{q}(Z\mid A) \to p(Z\mid A)$, the true posterior. Furthermore, for the importance weighted lower bound:
\begin{equation*}
    \tilde{\mathcal{L}}_K \doteq \mathbb{E}_{Z \sim \tilde{q}(Z\mid A)} \left[ \log \frac{p(Z, A)}{\tilde{q}(Z\mid A)} \right],
\end{equation*}
and for $K\geq1$, the bounds satisfy $\log p(A) \geq\tilde{\mathcal{L}}_{K} \geq \mathcal{L}_K \geq \mathcal{L}_1=\text{ELBO}_{\text{VAE}}$.
\end{prop}

This interpretation reveals that the IWAE bound can be viewed as training with an improved proposal distribution $\tilde{q}(Z|A)$ that is obtained through the resampling procedure. Importantly, this improved proposal $\tilde q(Z\mid A)$ is provably closer to the true posterior $p(Z\mid A)$ than the original $q(Z\mid A)$.

\subsection{Practical Advantages of SIR IWAE} 

The categorical resampling interpretation provides several practical advantages for our embodied reasoning application:

\noindent\textbf{Policy Improvement Guarantee} The resampling procedure acts as a non-parametric policy improvement operator. Each resampled trace $Z^*$ is theoretically guaranteed in expectation to come from a distribution that provides a theoretically tighter bound on $\log p(A)$ than the original proposal $q(Z|A)$.

\noindent \textbf{Interpretable Refined Training Data} Rather than computing expectations over the continuous distribution defined by the log-sum-exp of importance weights, we obtain samples from the improved posterior estimate $\tilde{q}(Z|A)$. This is important for our setting where we generate concrete reasoning traces $Z^*$ to refining the original reasoning distribution and augment the training data. These reasoning traces can provide better interpretability and insight into what action-predictive embodied reasoning is.

\noindent\textbf{Computational Efficiency} For synthetic data generation at scale, generating $K$ samples and resampling once is more efficient than computing gradients through the log-sum-exp operation required for standard IWAE training.

\noindent\textbf{Data Augmentation Compatibility} The discrete resampling naturally produces a refined dataset where each demonstration $(C, A)$ is paired with a single high-quality reasoning trace $Z^*$, making it directly compatible with standard reasoning VLA model training pipelines.

\subsection{Adapting SIR IWAE in \algname}

In our embodied reasoning setting, we apply this categorical resampling framework as follows:

\subsubsection{The Surrogate Posterior as Proposal Distribution}

During the Warmstarting training in Alg.~\ref{alg:rbencore_train}, we have trained a model $\mathcal{M}_{pq}$ that parameterizes both:
\begin{itemize}
    \item The joint distribution: $p(Z, A | C)$
    \item The approximate posterior: $q(Z | C, A)$
\end{itemize}

For each demonstration $(C^i, A^i)$ in our dataset, we use $q(\cdot | C^i, A^i)$ as the proposal distribution to generate $K$ candidate reasoning traces $\{Z^i_k\}_{k=1}^K$ (Alg.~\ref{alg:rbencore_train} Line 3).

\subsubsection{Importance Weight Computation}

For each candidate $Z^i_k$, we compute the importance weight (Alg.~\ref{alg:rbencore_train} Line 4-6):
\begin{equation}
    w(Z^i_k) = \frac{p(Z^i_k, A^i | C^i)}{q(Z^i_k | C^i, A^i)}.
\end{equation}

These weights can be computed efficiently using the model's log-probabilities for the respective conditioning patterns.

\subsubsection{Categorical Resampling for Data Refinement}

We then resample a single trace $Z_i^*$ from the categorical distribution (Alg.~\ref{alg:rbencore_train} Line 7):
\begin{equation}
    Z^{i*} \sim \text{Cat}\left(\left\{\frac{w(Z^i_k)}{\sum_{k'=1}^K w(Z^i_{k'})}\right\}_{k=1}^K\right).
\end{equation}

By Proposition~\ref{prop:categorical_equivalence}, this trace $Z^{i*}$ is sampled from the improved distribution $\tilde{q}(Z | C^i, A^i)$ that provides a tighter bound on the true posterior than the original $q(Z | C^i, A^i)$.

\subsection{Contextualizing Information Benefit in SIR IWAE}

Our algorithm leverages the sampling-importance-resampling (SIR) procedure to provide a rigorous framework for quantifying the value of reasoning in terms of its information benefit for expert action prediction. To understand this connection, we relate the improved posterior distribution $\tilde{q}(Z|A)$ induced by SIR IWAE to the specific reasoning strategies selected during resampling. We demonstrate that this categorical resampling step naturally shifts the distribution of reasoning strategies by exactly the amount they improve action prediction; specifically, the model's ``Refined Preference" for a strategy mathematically decomposes into its ``Information Benefit" plus its initial ``Warmstart Preference." Consequently, reasoning strategies that reduce divergence from the expert action distribution are automatically amplified, while distracting strategies are suppressed, allowing the model to autonomously filter for action-predictive reasoning without external supervision.

Recall our definition of information benefit for a reasoning strategy $\mathbf{R}$:
\begin{equation*}
    \Delta \mathcal{I}_{\mathbf{R}} \doteq D_{\text{KL}}(p_{\text{data}} \| p(A|C, \mathcal{Z}_{\cancel{\mathbf{R}}})) - D_{\text{KL}}(p_{\text{data}} \| p(A|C, \mathcal{Z}_{\mathbf{R}})),
\end{equation*}
where $\mathcal{Z}_{\mathbf{R}}$ denotes the set of reasoning traces that include strategy $\mathbf{R}$, and $\mathcal{Z}_{\cancel{\mathbf{R}}}$ denotes traces that exclude it. A positive $\Delta \mathcal{I}_{\mathbf{R}}$ indicates that reasoning with strategy $\mathbf{R}$ brings the model's action distribution closer to the expert's, while a negative value indicates that the strategy is distracting.

\subsubsection{How SIR Reweights Reasoning Strategies}

The improved posterior $\tilde{q}(Z|A)$ induced by categorical resampling does not uniformly upweight all reasoning traces. Instead, it selectively amplifies traces based on their importance weights $w(Z) = \frac{p(Z,A|C)}{q(Z|C,A)}$, which measure how well each trace aligns with both the model's generative distribution and the expert action.

To formalize this, consider partitioning the reasoning space into $\mathcal{Z}_{\mathbf{R}}$ and $\mathcal{Z}_{\cancel{\mathbf{R}}}$. When we draw $K$ candidate samples $\{Z_k\}_{k=1}^K \sim q(Z|C,A)$ and resample according to normalized importance weights, the probability that the resampled trace $Z^* \sim \tilde{q}(Z|C,A)$ belongs to $\mathcal{Z}_{\mathbf{R}}$ is:
\begin{align}
    \tilde{q}(\mathcal{Z}_{\mathbf{R}} \mid C,A) &= \E_{Z_1,\ldots,Z_K \sim q(Z|C,A)} \left[ \sum_{k=1}^K \frac{w(Z_k)}{\sum_{k'=1}^K w(Z_{k'})} \cdot \mathbf{1}_{Z_k \in \mathcal{Z}_{\mathbf{R}}} \right]\nonumber \\
    &= \E_{Z_1,\ldots,Z_K \sim q(Z|C,A)} \left[  \frac{\frac{1}{K}\sum_{k=1}^Kw(Z_k)\cdot\mathbf{1}_{Z_k \in \mathcal{Z}_{\mathbf{R}}}}{\frac{1}{K}\sum_{k'=1}^K w(Z_{k'})}  \right].
\end{align}

This expression captures the stochastic resampling process: each candidate $Z_k$ is selected with probability proportional to its importance weight $w(Z_k)$, so this is the probability that the selected candidate belongs to $\mathcal{Z}_{\mathbf{R}}$.

\subsubsection{The Large-$K$ Limit: Relating SIR to Expected Importance Weights}

To understand how SIR affects strategy selection, we analyze the large-$K$ limit where the law of large numbers provides intuition. As $K \to \infty$, the sample averages converge to expectations:
\begin{equation}
    \frac{1}{K}\sum_{k=1}^K w(Z_k) \cdot \mathbf{1}_{Z_k \in \mathcal{Z}_{\mathbf{R}}} \xrightarrow{K \to \infty} \E_{Z \sim q}[w(Z) \cdot \mathbf{1}_{Z \in \mathcal{Z}_{\mathbf{R}}}].
\end{equation}

We can decompose this expectation by conditioning on the set membership:
\begin{align}
    \E_{Z \sim q}[w(Z) \cdot \mathbf{1}_{Z \in \mathcal{Z}_{\mathbf{R}}}]&= \int_{\mathcal{Z}_{\mathbf{R}}} w(z) \, q(z|C,A) \, dz \nonumber\\
    &= \underbrace{\left(\int_{\mathcal{Z}_{\mathbf{R}}} q(z|C,A) \, dz\right)}_{q(\mathcal{Z}_{\mathbf{R}}|C,A)} \cdot \underbrace{\frac{\int_{\mathcal{Z}_{\mathbf{R}}} w(z) \, q(z\mid C,A) \, dz}{\int_{\mathcal{Z}_{\mathbf{R}}} q(z|C,A) \, dz}}_{\E_{Z \sim q}[w(Z) | Z \in \mathcal{Z}_{\mathbf{R}}]} \nonumber \\
    &= q(\mathcal{Z}_{\mathbf{R}}|C,A) \cdot \E_{Z \sim q}[w(Z) | Z \in \mathcal{Z}_{\mathbf{R}}].
\end{align}

Similarly, the denominator converges:
\begin{equation}
    \frac{1}{K}\sum_{k'=1}^K w(Z_{k'}) \xrightarrow{K \to \infty} \E_{Z \sim q}[w(Z)].
\end{equation}

Therefore, in the large-$K$ limit, the improved posterior assigns the following probability to $\mathcal{Z}_{\mathbf{R}}$:
\begin{equation}
    \tilde{q}(\mathcal{Z}_{\mathbf{R}} | C,A) = \frac{q(\mathcal{Z}_{\mathbf{R}}|C,A) \cdot \E_{Z \sim q}[w(Z) | Z \in \mathcal{Z}_{\mathbf{R}}]}{\E_{Z \sim q}[w(Z)]}.
\end{equation}

This demonstrates the reweighting mechanism in the new categorical importance sampled distribution: the improved posterior $\tilde{q}$ modifies the original probability $q(\mathcal{Z}_{\mathbf{R}}|C,A)$ by a factor equal to the ratio of the conditional expected importance weight (within $\mathcal{Z}_{\mathbf{R}}$) to the overall expected importance weight.

\subsubsection{Quantifying Strategy Preference Through Probability Ratios}

To understand how SIR shifts preference between reasoning strategies, we examine the ratio of probabilities assigned to $\mathcal{Z}_{\mathbf{R}}$ versus $\mathcal{Z}_{\cancel{\mathbf{R}}}$. For the improved posterior:
\begin{align}
    \frac{\tilde{q}(\mathcal{Z}_{\mathbf{R}} | C,A)}{\tilde{q}(\mathcal{Z}_{\cancel{\mathbf{R}}} | C,A)} 
    &= \frac{\E_{Z \sim q}[w(Z) | Z \in \mathcal{Z}_{\mathbf{R}}]}{\E_{Z \sim q}[w(Z) | Z \in \mathcal{Z}_{\cancel{\mathbf{R}}}]} \cdot \frac{q(\mathcal{Z}_{\mathbf{R}} | C,A)}{q(\mathcal{Z}_{\cancel{\mathbf{R}}} | C,A)} \label{eq:ratio_decomposition}
\end{align}

We consider two components:
\begin{itemize}
    \item \textbf{Warmstart Preference for R:} $\log \frac{q(\mathcal{Z}_{\mathbf{R}} | C,A)}{q(\mathcal{Z}_{\cancel{\mathbf{R}}} | C,A)}$ represents the (log) prior preference for strategy $\mathbf{R}$ in the original proposal distribution $q$, which is induced by our warmstart training with reasoning dropout rate $d$.
    
    \item \textbf{Importance Weight Ratio:} The term $\frac{\E_{Z \sim q}[w(Z) | Z \in \mathcal{Z}_{\mathbf{R}}]}{\E_{Z \sim q}[w(Z) | Z \in \mathcal{Z}_{\cancel{\mathbf{R}}}]}$ quantifies how much more (or less) the model values traces with strategy $\mathbf{R}$ in terms of their alignment with the joint distribution $p(Z,A|C)$.
\end{itemize}

From our warmstart training procedure (Alg.~\ref{alg:rbencore_train}), we construct the latent dataset with reasoning dropout such that each strategy $\mathbf{R}$ is independently included with probability $1-d$ and excluded with probability $d$. So, as in the assumptions from Section~\ref{sec:prop}, this induces:
\begin{equation}
    \frac{q(\mathcal{Z}_{\mathbf{R}} | C,A)}{q(\mathcal{Z}_{\cancel{\mathbf{R}}} | C,A)} = \frac{1-d}{d}.
    \label{eq:warmstart_ratio}
\end{equation}

Substituting Equation~\eqref{eq:warmstart_ratio} into Equation~\eqref{eq:ratio_decomposition}, we obtain:
\begin{equation}
    \frac{\tilde{q}(\mathcal{Z}_{\mathbf{R}} | C,A)}{\tilde{q}(\mathcal{Z}_{\cancel{\mathbf{R}}} | C,A)} = \frac{\E_{Z \sim q}[w(Z) | Z \in \mathcal{Z}_{\mathbf{R}}]}{\E_{Z \sim q}[w(Z) | Z \in \mathcal{Z}_{\cancel{\mathbf{R}}}]} \cdot \frac{1-d}{d}.
\end{equation}

We define the \textbf{Refined Preference for $\mathbf{R}$} as the log-ratio of probabilities of reasoning with strategy $\mathbf{R}$ and without strategy $\mathbf{R}$ in the improved posterior, averaged over the expert action distribution:
\begin{equation*}
    \textbf{Refined Preference for } \mathbf{R} := \E_{A\sim p_{\text{data}}} \left[ \log  \frac{\tilde{q}(\mathcal{Z}_{\mathbf{R}} | C,A)}{\tilde{q}(\mathcal{Z}_{\cancel{\mathbf{R}}} | C,A)} \right].
\end{equation*}
Note that the improved posterior distrution $\tilde q(\cdot\mid C,A)$ is the source for the refined reasoning training data for the VLA model in Alg.~\ref{alg:rbencore_refine}.

Expanding this expectation:
\begin{align}
    \E_{A\sim p_{\text{data}}} \left[ \log  \frac{\tilde{q}(\mathcal{Z}_{\mathbf{R}} | C,A)}{\tilde{q}(\mathcal{Z}_{\cancel{\mathbf{R}}} | C,A)} \right]  &= \E_{A\sim p_{\text{data}}} \left[ \log \frac{\E_{Z \sim q}[w(Z) | Z \in \mathcal{Z}_{\mathbf{R}}]}{\E_{Z \sim q}[w(Z) | Z \in \mathcal{Z}_{\cancel{\mathbf{R}}}]} + \log \frac{1-d}{d} \right] \nonumber \\
    &= \underbrace{\E_{A\sim p_{\text{data}}} \left[ \log \frac{\E_{Z \sim q}[w(Z) | Z \in \mathcal{Z}_{\mathbf{R}}]}{\E_{Z \sim q}[w(Z) | Z \in \mathcal{Z}_{\cancel{\mathbf{R}}}]} \right]}_{\Delta \mathcal{I}_{\mathbf{R}} \text{ (by Proposition of Sec~\ref{prop:info_benefit})}} + \underbrace{\log \frac{1-d}{d}}_{\text{Warmstart Preference}} \nonumber \\
    &= \Delta \mathcal{I}_{\mathbf{R}} + \log \frac{1-d}{d}. \label{eq:refined_preference}
\end{align}

\subsubsection{Information Benefit Drives Strategy Selection}

Equation~\eqref{eq:refined_preference} reveals the fundamental insight that the refined preference for a reasoning strategy $\mathbf{R}$, which characterizes the new refined reasoning training dataset, decomposes into two terms:
\begin{equation*}
   \textbf{Refined Preference for } \mathbf{R} =\textbf{Information Benefit of } \mathbf{R} + \textbf{Warmstart Preference for } \mathbf{R}
\end{equation*}

This decomposition has a clear interpretation:
\begin{itemize}
    \item The \textbf{Warmstart Preference} ($\log \frac{1-d}{d}$) is a constant baseline that reflects the prior probability of including strategy $\mathbf{R}$ during warmstart training. 
    
    \item The \textbf{Information Benefit} ($\Delta \mathcal{I}_{\mathbf{R}}$) represents the additional preference shift induced by categorical resampling based on importance weights. This term is strategy-specific and captures how much the strategy improves action prediction.
\end{itemize}

This demonstrates that the \algname's SIR IWAE procedure systematically adjusts the distribution of reasoning strategies based on their utility for action prediction:

\paragraph{Positive Information Benefit ($\Delta \mathcal{I}_{\mathbf{R}} > 0$)} When strategy $\mathbf{R}$ improves action prediction, traces in $\mathcal{Z}_{\mathbf{R}}$ receive higher importance weights on average than traces in $\mathcal{Z}_{\cancel{\mathbf{R}}}$. Categorical resampling amplifies the probability of selecting traces with $\mathbf{R}$, increasing its representation in the refined dataset beyond the warmstart baseline.

\paragraph{Negative Information Benefit ($\Delta \mathcal{I}_{\mathbf{R}} < 0$)} When strategy $\mathbf{R}$ distracts from action prediction, traces in $\mathcal{Z}_{\mathbf{R}}$ receive lower importance weights on average. Categorical resampling suppresses the probability of selecting traces with $\mathbf{R}$, decreasing its representation below the warmstart baseline.

\paragraph{Zero Information Benefit ($\Delta \mathcal{I}_{\mathbf{R}} = 0$)} When strategy $\mathbf{R}$ provides no information about actions, the importance weight ratio is unity (in expectation), and the refined preference equals the warmstart preference. The strategy's representation remains unchanged.

\subsubsection{Relating to \algname}
By employing SIR IWAE to generate refined reasoning traces, \algname implements a principled, self-supervised filtering process that:
\begin{enumerate}
    \item Starts with diverse reasoning strategies sampled uniformly via reasoning dropout (warmstart)
    \item Measures each strategy's alignment with expert actions through importance weights
    \item Automatically concentrates probability mass on action-predictive strategies through categorical resampling
    \item Produces a refined training dataset where strategy prevalence correlates with information benefit
\end{enumerate}

Importantly, this entire process requires no external rewards, verifiers, or heuristics---only the generative probabilities of the model itself. The improved posterior $\tilde{q}(Z|C,A)$ discovers embodiment-specific reasoning distributions (Fig.~\ref{fig:mainresult}) by maximizing information benefit for action prediction. Through this mechanism, \algname can identify distinct reasoning strategies across manipulation, legged navigation, and autonomous driving domains, with each embodiment benefiting from a unique distribution of reasoning primitives tailored to its specific embodiment and task challenges.

\subsection{Scaling with Inference Compute}

A key advantage of the categorical resampling approach is that performance improves monotonically with the number of samples $K$. From the IWAE literature~\cite{burda16importance, cremer2017reinterpreting}, we know that:
\begin{equation}
    \tilde{\mathcal{L}}_{K+1} \geq \tilde{\mathcal{L}}_K ,
\end{equation}
for $K\geq 1$, with the bound approaching $\log p(A)$ as $K \to \infty$ when $\tilde{q}(Z|A)$ converges to $p(Z|A)$.

In our experiments on scaling sampling for AV VLA models (Fig.~\ref{fig:scaling_law}), we empirically validate this theoretical prediction, showing that increasing $K$ from 8, 12, 16 to 32 consistently improves downstream performance on metrics like collision rate. This demonstrates a practical avenue for scaling inference compute to improve data efficiency during pretraining---by investing more compute in the posterior sampling step to generate higher-quality reasoning traces, we empirically observe better model performance.

%% file: appendix/implementation.tex
%\section{Implementation Details for each embodiment}

%\subsection{architecture for pq}

%\subsection{Training, Inference, Prompting (training and generation, inferencing, what are the reasoning primitives)}

\input{floats/figures/prior_posterior_arch}

\input{floats/tables/embodiment_details}

\section{Experimental Details}
\label{sec:experimental_details}

This section provides comprehensive details on the training procedures, inference protocols, and reasoning primitive specifications for all experimental domains evaluated in this work.

\subsection{LIBERO-90 Manipulation}
\label{sec:app_libero_details}

\subsubsection{Model Architecture and Training}
We employ the MiniVLA architecture~\cite{belkhale2024minivla} with a 1B-parameter configuration, consisting of a 0.5B Qwen2.5 language model backbone~\cite{qwen2025qwen25technicalreport} combined with DINOv2~\cite{oquab2024dinov} and SigLIP~\cite{zhai2023sigmoid} vision encoders. Actions are represented as seven discrete tokens encoding a 10-step action chunk~\cite{zitkovich2023rt2, belkhale2024minivla} via VQ-VAE tokenization~\cite{van2017neural} with codebook size of 256.

\noindent\textbf{Implementation.} We adapt the codebase from~\cite{zawalski2025Robotic} to implement MiniVLA with VQ-VAE action tokenization. Note that while we follow the architectural principles from~\cite{chen2025training}, their code, models, and detailed hyperparameters were not publicly available at the time of our experiments, so we could not directly verify or reproduce their specific implementation details.

\noindent\textbf{Warmstart Training.} The prior-posterior model $\mathcal{M}_{pq}$ is trained on 64 NVIDIA A100 GPUs (40GB) with a total batch size of 512. Training continues until the model achieves 95\% action token prediction accuracy on the training set. We apply reasoning dropout with rate $d = 0.2$ to generate diverse reasoning strategy combinations as described in Alg.~\ref{alg:rbencore_train}.

\noindent\textbf{Reasoning VLA Training.} Following the refinement procedure (Alg.~\ref{alg:rbencore_refine}), we train the final VLA model $\mathcal{M}_{\text{VLA}}$ from the base model initialization to ensure fair comparison with baseline methods. This model is trained on the refined dataset $\mathcal{D}_{\text{refined}}$ with the same hardware configuration and stopping criterion (95\% action token accuracy).

\noindent\textbf{Posterior Sampling.} During the refinement stage, we sample $K=8$ reasoning candidates from the posterior distribution for each demonstration. 

\subsubsection{Reasoning Primitives}
We utilize seven reasoning primitives for manipulation tasks, with content generated from the foundation models Llama 2~\cite{touvron2023llama}, Molmo~\cite{deitke2025molmo} following the data from previous work~\cite{chen2025training}:

\begin{itemize}
    \item \textbf{Plan}: High-level task decomposition into sequential subtasks.
    
    \item \textbf{Visible Objects}: Enumeration of objects detected in the scene using 2D bounding boxes.
    
    \item \textbf{Subtask}: Current subtask or immediate goal the robot should accomplish.
    
    \item \textbf{Subtask Explain}: Detailed justification for why the current subtask is necessary. In~\cite{chen2025training} this was called Subtask Reasoning, and we trained using the phrase Subtask Reasoning, but for the sake of clarity and differentiating the usage of the word Reasoning we call this Subtask Explain in the main paper.
    
    \item \textbf{Move}: Meta-action language description of the intended motion (e.g., ``move gripper to object").
    
    \item \textbf{Move Explain}: Detailed reasoning explaining why the proposed movement is appropriate. Like Subtask Explain, this primitive was called Move Reasoning in the~\cite{ chen2025training} and we trained our traces using the phrase Move Reasoning.
    
    \item \textbf{Gripper Position}: 2D waypoint representation of the end-effector's spatial coordinates.
\end{itemize}

\subsubsection{Inferencing}
At test time, the model receives a scene image and task description as context $C$. The model generates reasoning traces at each action step before predicting the action tokens. This step-by-step reasoning allows the policy to adapt its chain-of-thought to the evolving scene state throughout task execution.

\noindent\textbf{Evaluation Protocol.} We perform 20 rollout trials per task. Success is determined by task-specific completion criteria defined in the LIBERO-90 benchmark.

\subsection{Bridgev2 WidowX Hardware Manipulation}
\label{sec:app_bridge_details}

\subsubsection{Model Architecture and Training}
We employ the OpenVLA architecture~\cite{kim2025openvla} with a 7B-parameter Llama 2~\cite{touvron2023llama} language backbone with DINOv2~\cite{oquab2024dinov} and SigLIP~\cite{zhai2023sigmoid} vision encoders. Actions are represented as seven discrete tokens following the tokenization scheme from~\cite{zitkovich2023rt2, kim2025openvla}.

\noindent\textbf{Implementation.} We adapt the codebase from~\cite{zawalski2025Robotic} for training the OpenVLA model with out algorithm on the Bridgev2 dataset~\cite{walke2023bridge}.

\noindent\textbf{Warmstart Training.} The prior-posterior model $\mathcal{M}_{pq}$ is trained on 64 NVIDIA A100 GPUs (40GB) with a global batch size of 640, continuing until 95\% action token accuracy is achieved. Reasoning dropout rate is set to $d = 0.2$.

\noindent\textbf{Reasoning VLA Training.} The final VLA model $\mathcal{M}_{\text{VLA}}$ is trained from base model initialization on the refined dataset $\mathcal{D}_{\text{refined}}$ with identical hardware and convergence criteria.

\noindent\textbf{Posterior Sampling.} During refinement, we sample $K=4$ reasoning candidates from the posterior distribution for each demonstration. We adapted SGLang-VLA~\cite{kwok2025robomonkey, zheng2024sglang} to support high-throughput batch inferencing of posterior distribution VLA architecture along with extraction of log probabilities of prior VLA architecture (needed for computation of importance weights).

\noindent\textbf{Post-Training All Primitives Model for Action Forcing.} For the baseline model trained with all reasoning primitives, we perform additional post-training epochs with reasoning-free examples to enable the action-forcing capability, following~\cite{chen2025training}. Models trained with \algname and even random primitives acquire this capability naturally through dropout.

\subsubsection{Reasoning Primitives}
We utilize the same seven reasoning primitives like in LIBERO-90 (Plan, Visible Objects, Subtask, Subtask Explain, Move, Move Explain, Gripper Position), with reasoning content generated by the Gemini 1.0 model~\cite{gemini2023family} from the data of~\cite{zawalski2025Robotic}. The demonstration data is sourced from the Bridgev2 dataset~\cite{walke2023bridge}.

\subsubsection{Evaluation}

During evaluation, we use a SGLang-VLA~\cite{kwok2025robomonkey, zheng2024sglang} serving engine to rapidly inference the VLAs during test-time.

\noindent\textbf{Evaluation with Suppressed Test-Time Reasoning (\ie Action Forcing).} In this evaluation mode (Fig.~\ref{fig:hwresult} and Table~\ref{tab:hw_af}), we suppress reasoning generation at test time appending the phrase \texttt{Action:} right at the end of the standard openvla prompt. This forces the model to directly produce action tokens without generating intermediate reasoning traces. We evaluate across 9 tasks (3 per category) totaling 468 trials, categorized into:
\begin{itemize}
    \item \textbf{In-Distribution}: Tasks similar to training data
    \item \textbf{OOD Target}: Tasks with novel target grasp objects not seen during training
    \item \textbf{OOD Scene with Distractions}: Cluttered environments containing task-irrelevant objects
\end{itemize}

This evaluation protocol tests whether reasoning during training improves the learned policy representation, even when reasoning is not explicitly generated at deployment.

\noindent\textbf{Evaluation with Test-Time Reasoning Enabled.} In this evaluation mode (Fig.~\ref{fig:slowhwresult} and Table~\ref{tab:hw_ttr}), models generate complete chain-of-thought reasoning traces before predicting actions. We evaluate on an in-distribution task and OOD scenes with distractions task to measure both task success and inference latency. Models trained on all primitives exhibit $>5$ second per-step latency, while R\&B-EnCoRe models achieve $\sim$3 second latency due to more concise reasoning traces. 

\subsection{Legged Robot Navigation}

\subsubsection{Model Architecture and Training}
We finetune a Qwen3-VL-30B-A3B-Instruct~\cite{bai2025qwen3vltechnicalreport} Mixture-of-Experts (MoE) model to process scene images and task specifications, outputting 2D waypoint coordinates for navigation. The model uses the NaviTrace dataset~\cite{windecker2025navitrace} containing approximately 1,000 tasks across 500 unique scenes for four embodiments: bipedal, wheeled robot, bicycle, and quadruped. Evaluation is performed on a 20\% holdout subset of the scenes (to ensure no evaluation image data contamination in training data).

\noindent\textbf{Training Infrastructure.} We utilize the Thinking Machines Lab Tinker API for both inference and finetuning of the Qwen3-VL-30B MoE model. Training is performed with a batch size of 16 for 8 epochs.

\noindent\textbf{Reasoning Dropout.} Training employs reasoning dropout with rate $d = 0.5$ to generate diverse reasoning strategy combinations. We additionally ablate the dropout values as seen in Fig.~\ref{fig:dropout_ablation}.

\noindent\textbf{Posterior Sampling.} During refinement, we sample $K=4$ reasoning candidates from the posterior distribution for each demonstration.

\subsubsection{Self-Generated Reasoning Primitives}
Unlike manipulation tasks that rely on foundation models external to/different from the base VLM, we generate reasoning primitives by querying the base Qwen3-VL-30B model itself via visual question answering (VQA). For each demonstration, we provide the scene image \textit{with the ground truth waypoint path overlaid} along with the task description.

\noindent\textbf{Data Curation Protocol.} For each scene in the NaviTrace dataset, we query the base Qwen3-VL-30B model with:
\begin{itemize}
    \item The scene image with ground truth waypoint trajectory overlaid
    \item The task description
    \item Individual prompts for each reasoning primitive
\end{itemize}

An example image and prompt for generating content for the eight reasoning primitives (Terrain, Obstacles, Affordances, Plan, Move, Social, Counterfactuals, and Weather) is:

\input{appendix/legged_prompt}

and the output from the Qwen3-VL-30b Instruct model was:

\input{appendix/legged_output}

We parsed out the individual contents for reasoning primitives from the corresponding XML synthetic textual data to get the corresponding reasoning.

\subsubsection{Inference Protocol}
At test time, the model receives the scene image and task description, generating 2D waypoint coordinates. Evaluation uses the NaviTrace metric~\cite{windecker2025navitrace} incorporating Dynamic Time Warping distance, goal endpoint error, and semantic penalties aligned with human preferences.

\subsection{Autonomous Driving}
\label{sec:app_av_details}

%%%%%%%%%%%%%%%%%%%%%%%%%%%%%%%% AV START %%%%%%%%%%%%%%%%%%%%%%%%%%%%%%%%
\subsubsection{Model Architecture and Representation}
We finetune a Qwen3-VL-4B-Instruct~\cite{bai2025qwen3vltechnicalreport} dense (non-MoE) model on the nuScenes dataset~\cite{caesar2020nuscenes} which consists of almost 1000 driving scenes of 20 seconds each, consisting of images and trajectory data sampled at 2Hz, collected from Boston and Singapore and we use the official training and validation split. 
We follow \cite{mao2023gpt,maolanguage} and use 2 seconds of history to predict 3-second of future trajectories at a 2Hz frequency.
We rely on a natural language representation of the trajectory, formatted as pairs of ($x$,$y$) in nuScenes ego-centric coordinates, for a total of 6 numbers. See an example below.
% Similar to \cite{maolanguage}, we output the ego-vehicle's planned trajectory over a 3-second horizon at a 2hz frequency. 
Different from prior works, we process the front-camera RGB image only in order to remain within the computational budget. 
We find that empirically, we perform similarly to methods using a surround-view setup \cite{mao2023gpt,zhou2025opendrive}, suggesting that future work can improve VLMs' understanding of multiple frames.
% We find that empirically, we are able to perform similarly to methods using a surround-view setup \cite{mao2023gpt,zhou2025opendrive}, suggesting that future work can be done to improve VLM's understanding of multiple frames.
% The model processes front-camera RGB image \textit{only} and outputs the ego-vehicle's planned trajectory over a 3-second horizon.

\noindent\textbf{Implementation.} We employed the Qwen3-VL 4B \cite{bai2025qwen3vltechnicalreport} dense model to initialize the autonomous driving VLA. Training employs reasoning dropout with rate $d = 0.5$.

\noindent\textbf{Posterior Sampling.} During refinement, we sample $K=16$ reasoning candidates from the posterior distribution for each demonstration (results reported in Table~III use this configuration). We ablate the effect of $K$ in Fig.~12, finding that performance saturates around $K=16$ samples.

\subsubsection{Reasoning Primitives}
We utilize reasoning traces from human-annotated data for the nuScenes dataset, originally curated for agent-based planning methods~\cite{mao2023gpt,maolanguage}. For this work, we re-distributed them into reasoning primitives that more evenly break down the original annotations. The new primitives, based on the annotations, include:

\begin{itemize}
    \item \textbf{Mission Goal}: High-level navigation objective (e.g., ``turn left", ``change lane right"). This information is extracted from the ``ego" field of the annotation from \cite{maolanguage}, repurposed for driving VLA reasoning.

    \item \textbf{Perception}: Detailed scene understanding including detected objects, their types, IDs, and predicted future waypoints. This field corresponds to the ``perception" annotations.

    \item \textbf{Notable Objects}: Salient objects requiring attention. Notable objects are extracted from original ``reasoning" field of the annotations, and are objects that a language model identified as notable.

    \item \textbf{Collidable Objects}: Objects with collision risk. These correspond to the notable objects in the ``chain\_of\_thoughts" field of the annotations, and are objects the ego-vehicle will collide with if the path is kept.

    \item \textbf{Driving Plan}: Concrete action plan with speed modulation (e.g., ``change lane right with constant speed"). This is extracted from the driving plan in the ``chain\_of\_thoughts" field of the annotations, and is heuristically generated.

    \item \textbf{Experience}: Retrieved similar past driving scenarios from memory with confidence scores. This corresponds exactly to the ``experiences" field of the annotations.

\end{itemize}

\subsubsection{Data Transformation and Evaluation}
% \noindent\textbf{Data transformation details.} 
The original nuScenes dataset was matched to the reasoning annotations from \cite{maolanguage}, and details corresponding to the ego-vehicle were extracted to provide to the driving VLA. This information corresponds to the ``ego" field in the annotations, and include current state (velocities, acceleration, can bus information, \etc) as well as the historical trajectory (2-second past information). The mission goal is not included, and is instead used for the reasoning. All reasoning is headed by a question, formatted such that the reasoning type is the answer, and in tags, \eg:
\begin{prompt}
\begin{code}
<What is the mission goal?> FORWARD

<What do you perceive in the scene?>

Distance to both sides of road shoulders of current ego-vehicle location:
Current ego-vehicle's distance to left shoulder is 5.5m and right shoulder is 1.5m

<What is a similar past experience?>
Most similar driving experience from memory with confidence score: 1.00:
The planned trajectory in this experience for your reference:
[(-0.00,0.00), (-0.00,0.00), (-0.00,-0.00), (-0.00,0.16), (-0.01,0.60), (0.01,1.40)]
<What are the collidable objects?>
 - Notable Objects: None
   Potential Effects: None<What is the driving plan?> STOP
<What are the notable objects?>
 - Notable Objects: None
   Potential Effects: None

Planned Trajectory:
[(0.05,4.55), (0.09,9.14), (0.12,12.84), (0.19,17.51), (0.22,22.14), (0.25,26.60)]
\end{code}
\end{prompt}
% \begin{itemize}
%     \item $<$What is the mission goal?$>$
%     \item $<$What do you perceive in the scene?$>$
%     \item $<$What is a similar past experience?$>$
%     \item $<$What are the collidable objects?$>$
%     \item $<$What is the driving plan?$>$
%     \item $<$What are the notable objects?$>$
% \end{itemize}

\noindent\textbf{Training Data Details.} During training, the model takes as input the front camera image, ego annotation, and the task. The model optionally predicts a reasoning trace in the above format, and a planned trajectory. We train on 4 NVIDIA H100 GPUs with a global batch size of $32$, for a fixed $30$ epoch training run, using the provided finetuning infrastructure.

\noindent\textbf{Evaluation Details.} We follow the UniAD evaluation \cite{hu2023planning} for open-loop evaluation of the planned trajectory. The predicted trajectory is parsed into the evaluation format. If there is an error in parsing---which occurs $<$0.5\% of the time---it is replaced with a stationary trajectory to not bias the evaluation results.

%% file: floats/figures/prior_posterior_arch.tex
\begin{figure*}[t]
    \centering
    \includegraphics[width=\textwidth]{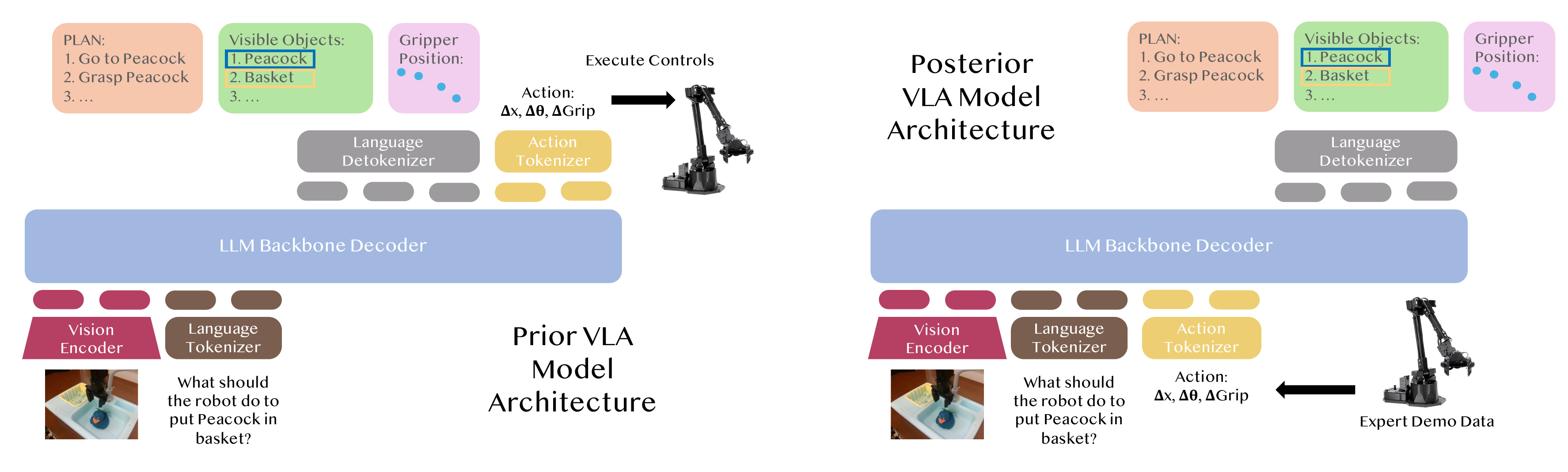}
    \caption{\small \textbf{Prior and Posterior architecture.} The prior architecture is the same as the standard generative VLA that takes as input the task context (scene and task) and outputs textual reasoning followed by action tokens. The posterior architecture takes as input the context \emph{and action} and outputs only the reasoning tokens.}
    \label{fig:pq_arch}
\end{figure*}

%% file: floats/tables/embodiment_details.tex
\begin{table*}[t]
\centering
\caption{\small Experimental Configuration Details Across Embodiment Domains}
\label{tab:experimental_configs}
\resizebox{\textwidth}{!}{%
\begin{tabular}{l|cccc}
\toprule
\textbf{Configuration} & \textbf{LIBERO-90} & \textbf{WidowX Hardware} & \textbf{Legged Robot} & \textbf{Autonomous Vehicle} \\
\midrule
$M$ (warmstart traces) & 10 & 4 & 8 & 32 \\
$K$ (posterior samples) & 8 & 4 & 4 & 8,12,16,32 \\
$\rho$ (reasoning primitives) & 7 & 7 & 7,8 (w/ or w/o Weather) & 6 \\
Dropout rate $d$ & 0.2 & 0.2 & 0.1,0.3,0.5.0.7,0.9 & 0.5 \\
Posterior sampling temp. & 1.0 & 1.0 & 0.7 & 1.0 \\
\midrule
Expert demo source & LIBERO-90~\cite{liu2023libero} & Bridgev2~\cite{walke2023bridge} & NaviTrace~\cite{windecker2025navitrace} & nuScenes~\cite{caesar2020nuscenes} \\
\midrule
Primary reasoning & Llama 2~\cite{touvron2023llama}, Molmo~\cite{deitke2025molmo} & Gemini 1.0~\cite{gemini2023family} & Qwen3-VL 30B MoE~\cite{bai2025qwen3vltechnicalreport} & Human Annotations~\cite{mao2023gpt, maolanguage} \\
primitive content source\\
\midrule
 & MiniVLA~\cite{belkhale2024minivla}: Qwen2.5 0.5B LLM~\cite{qwen2025qwen25technicalreport}, & OpenVLA~\cite{kim2025openvla}: Llama 2 7B LLM~\cite{touvron2023llama}, & Qwen3-VL 30B MoE~\cite{bai2025qwen3vltechnicalreport} & Qwen3-VL 4B Dense~\cite{bai2025qwen3vltechnicalreport} \\
VLA architecture & DINOv2~\cite{oquab2024dinov}+SigLIP~\cite{zhai2023sigmoid},VQ-VAE~\cite{van2017neural} &  DINOv2~\cite{oquab2024dinov}+SigLIP~\cite{zhai2023sigmoid} & & \\
\bottomrule
\end{tabular}%
}
\end{table*}

%% file: appendix/legged_prompt.tex
\input{floats/figures/wheeled_with_traj}
\newpage
\begin{prompt}
\begin{code}
You are a navigation expert for various embodiments including robots and humans. Given a first-person image view of the current scenario with a planned path overlaid on the image, the coordinates of these 2D points in normalized image coordinates (ranging from 0 to 1, where [0,0] is the top-left and [1,1] is the bottom-right), a specified embodiment (e.g., legged robot, wheeled robot, human, or bike), and a navigation task (e.g., "Go down the road"), you will provide a comprehensive, step-by-step, detailed reasoning trace in full sentences about how the overlaid path on the image solves the given task. Do not explicitly refer to the trajectory during the reasoning as future readers of your reasoning trace will not see the overlaid path.
Your reasoning must be provided within the specified XML-style tags for each distinct category.

###Navigation Reasoning Components
Terrain: A detailed analysis of the physical characteristics of the ground surfaces in the scene, including the types of stepping, rolling, or contact considerations required (e.g., grass vs. sidewalk, uneven cobblestones, smooth asphalt).
Obstacles: An identification of specific objects, obstacles, or areas anywhere in the scene that the embodiment must actively avoid colliding with or entering.
Affordances: Description of specific action capabilities a region of the environment offers to the robot--such as steppable, jumpable, or gap-crossable--determined by the interplay between the terrain's physical properties and the robot's kinematic and dynamic limits.
Plan: High-level language description of the overall navigation plan.
Move: Description of primary direction and type of motion required to follow the path (e.g., move forward, turn left 45 degrees, slow down).
Social: Description of expected or required social behaviors and conventions that the depicted trajectory does to maintain safe and cooperative movement around other entities (e.g., move left for bikers, right for humans or yield to oncoming traffic).
Counterfactuals: Listing of potential alternative navigation plans and their likely consequences, presented as pairs of Action and Consequence.
Weather: An analysis of the weather in the environment using subjective language.

###Output Format
You must format your final output using the following tags for each corresponding type of reasoning:
<Terrain>
Description of physical properties of surface contact patches.
</Terrain>

<Obstacles>
Identify objects in scene to not hit.
</Obstacles>

<Affordances>
Region: [Environment region/terrain 1]
Affordances: [Action capabilities of region/terrain 1]

Region: [Environment region/terrain 2]
Affordances: [Action capabilities of region/terrain 2]
...
</Affordances>

<Plan>
Step 1: [First step in plan]
Step 2: [Second step in plan]
...
</Plan>

<Move>
Directions of motions.
</Move>

<Social>
Applicable social rules for movement.
</Social>

<Counterfactuals>
Action: [Alternative action 1]
Consequence: [Likely result of alternative action 1]

Action: [Alternative action 2]
Consequence: [Likely result of alternative action 2]
...
</Counterfactuals>

<Weather>
Description of the weather.
</Weather>

###Additional Notes about the Image:
  - The image shows a first-person view of the navigation scenario
  - The trajectory starts near the bottom center of the image, which corresponds approximately to normalized coordinate [0.5, 0.95] (representing the current position of the embodiment)
  - The trajectory is adapted to the embodiment's abilities and limitations
  - The trajectory extends all the way to the goal if the path is visible. If the path is occluded, the trajectory ends where the path becomes fully obscured, unless the path can be reasonably inferred from the visible context.

### Information on General Embodiment Movement Characteristics
  - **Human**: A standard pedestrian. Can navigate stairs and ramps but cannot climb tall obstacles.
  - **Wheeled Robot**: A wheeled delivery robot. Behaves like a wheelchair, preferring smooth surfaces such as walkways and ramps. It cannot use stairs or escalators.
  - **Bicycle**: A standard cyclist. Follows traffic regulations and prefers bike lanes or streets. Cannot navigate stairs.
  - **Legged Robot**: A quadruped like ANYmal. Behaves similarly to a human, but it is shorter. It can handle stairs and escalators.

### Your Task Specific Information
**Trajectory**: [[0.50, 0.95], [0.60, 0.85], [0.69, 0.72],[0.72, 0.55],[0.73, 0.40], [0.68, 0.35], [0.61, 0.33], [0.57, 0.31], [0.52, 0.30], [0.47, 0.30], [0.41, 0.29], [0.34, 0.29], [0.26, 0.28], [0.21, 0.28], [0.13, 0.27], [0.05 , 0.27]]
**Your current Embodiment**: Wheeled Robot
**Task**: Follow the road to the left
\end{code}
\end{prompt}

%% file: floats/figures/wheeled_with_traj.tex
\begin{figure*}[ht]
    \centering
    \includegraphics[width=0.5\textwidth]{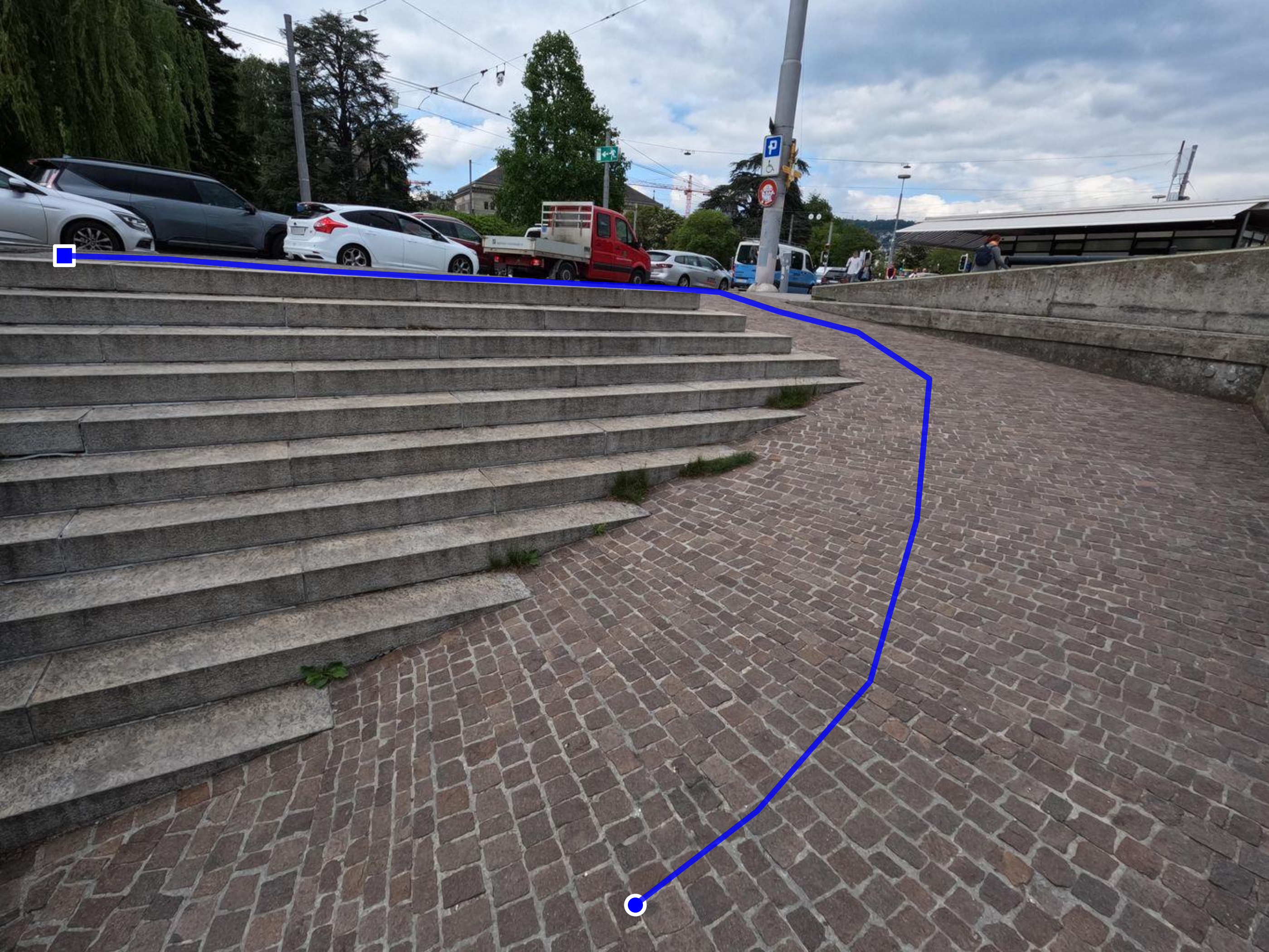}
\end{figure*}

%% file: appendix/legged_output.tex
\begin{prompt}
\begin{code}
<Terrain>
The terrain consists of two main surfaces: a set of wide, stone steps on the left and a cobblestone path on the right. The steps are made of large, flat stone slabs and are designed for walking. The cobblestone path is composed of small, irregularly shaped stones set in a pattern, creating a textured but generally stable surface. For a wheeled robot, the cobblestone path is navigable but may require careful control due to the unevenness, while the steps are completely impassable.
</Terrain>

<Obstacles>
The primary obstacles are the stone steps on the left, which the robot cannot traverse. Additionally, there are parked cars and a red truck in the background, which are static obstacles that must be avoided. The presence of a person in the distance on the right side of the path is also noted, although they appear to be stationary and not in the immediate path of the robot.
</Obstacles>

<Affordances>
Region: Cobblestone path
Affordances: The cobblestone path is suitable for a wheeled robot to move across, although it may require adjustments in speed and steering to handle the uneven surface.

Region: Stone steps
Affordances: The stone steps are not suitable for a wheeled robot to traverse and must be avoided.
</Affordances>

<Plan>
Step 1: Begin moving forward along the cobblestone path, staying to the right of the steps to avoid them.
Step 2: Continue moving forward, maintaining a steady course along the path as it curves slightly to the left.
Step 3: Ensure to navigate around any potential obstacles, such as the parked cars and the person, by adjusting the path slightly if necessary.
Step 4: Continue until the path is fully traversed or the goal is reached.
</Plan>

<Move>
Move forward along the cobblestone path, navigating the slight curve to the left. Maintain a steady pace and ensure to stay clear of the steps on the left and any obstacles on the path.
</Move>

<Social>
The robot should maintain a safe distance from the parked cars and the person on the right side of the path. It should not obstruct the path of any pedestrians and should move at a pace that allows for smooth and safe navigation.
</Social>

<Counterfactuals>
Action: Attempt to move up the stone steps
Consequence: The wheeled robot cannot traverse the steps and would become stuck or damaged.

Action: Move directly towards the parked cars
Consequence: The robot would collide with the vehicles, causing damage and potentially obstructing traffic.
</Counterfactuals>

<Weather>
The weather appears to be overcast, with a cloudy sky and no direct sunlight visible. The clouds are thick and gray, suggesting a cool and possibly damp environment. There are no signs of rain or strong winds, and the overall lighting is diffused, creating a soft and even illumination across the scene.
</Weather>
\end{code}
\end{prompt}

%% file: appendix/main_exp.tex
%\newpage
\section{Hardware Details with Results}
\label{sec:app_hardware_details}

%\subsection{Hardware Experiment setup and detailed plots}

%\input{floats/tables/hardware_af}

\input{floats/figures/hardware_results}

\input{floats/tables/hardware_ttr}

%\subsection{Example Traces/Trajectories for Each Embodiment}

%\section{Detailed Hardware Results, Example Reasoning Traces, Additional Ablation Experiments}

\newpage

\section{Additional Ablation Experiments}

\input{floats/figures/dropout_ablation}

\input{floats/figures/legged_sampling_bound}

\input{floats/figures/navitrace_wgarbage_perf}

%% file: floats/figures/hardware_results.tex
\begin{figure}[h] 
    \centering
    \includegraphics[width=\linewidth]{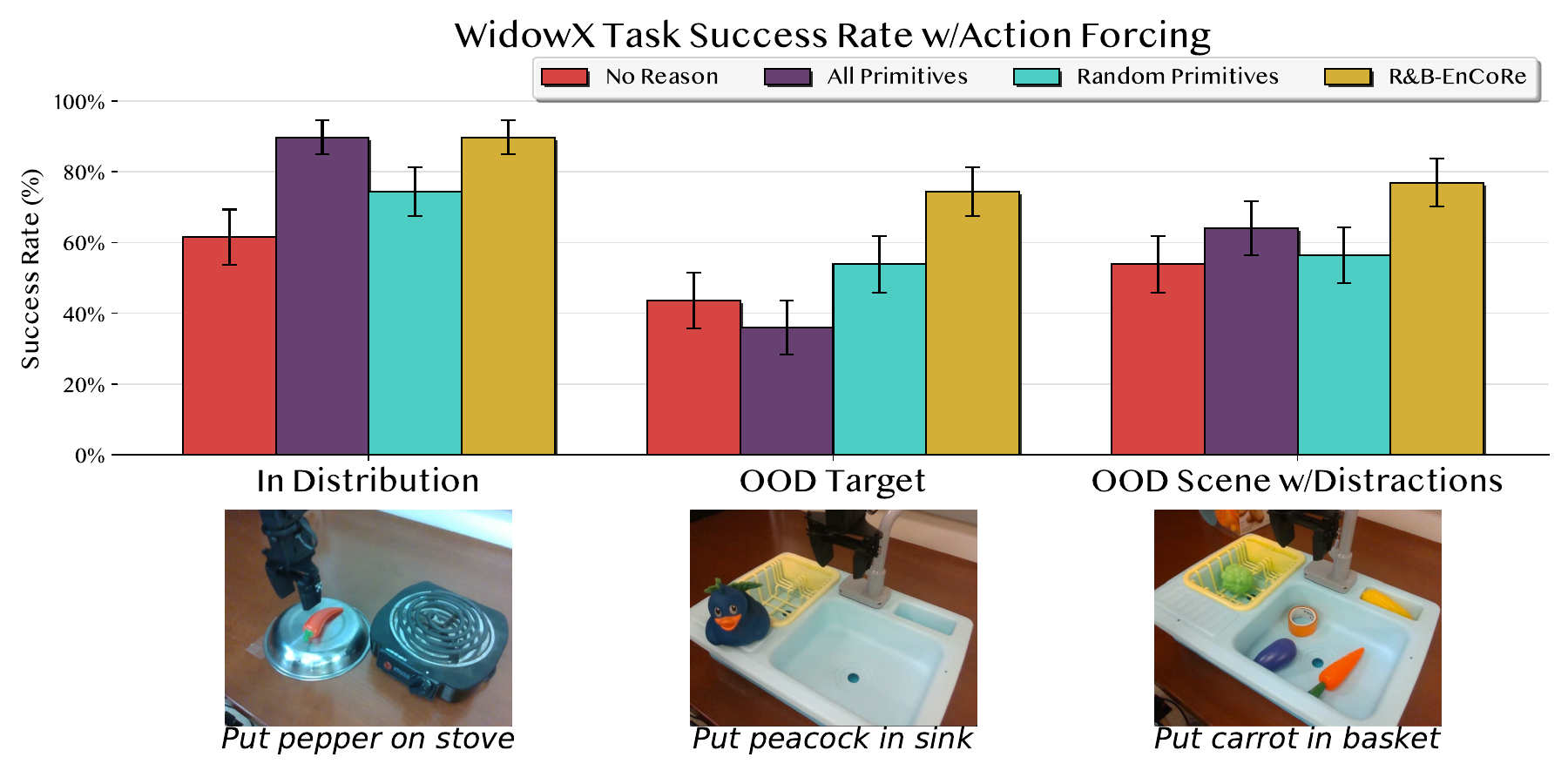}

    %\vspace{-8pt}
    \caption{\small Success rates on WidowX hardware Bridgev2 setup. Reasoning Models are prompted and/or trained with Action forcing for reduced latency~\cite{chen2025training}. Evaluated on a total of 9 different tasks (3 per category), and 468 total trials. We show 1 example task per category.}
    \label{fig:hwresult}
%\vspace{-5pt}
\end{figure}

%% file: floats/tables/hardware_ttr.tex
\begin{table*}[h]
\centering
\caption{\small WidowX Task Success Rate with Test-Time Reasoning Enabled. Evaluation includes 13 trials per task and model to compare performance when the model generates full chain-of-thought traces.}
\label{tab:hw_ttr}
\resizebox{0.8\textwidth}{!}{%
\begin{tabular}{llcc}
\hline
\textbf{Category} & \textbf{Task} & \textbf{All Primitives} & \textbf{\algname} \\ \hline
In Distrib. & put red pepper in yellow basket & 84.6\% & \textbf{92.3\%} \\ \hline
\multirow{2}{*}{OOD Scene w/Distr.} & put red pepper in yellow basket & \multirow{2}{*}{53.8\%} & \multirow{2}{*}{\textbf{84.6\%}} \\ 
 & \emph{(include distraction objects in basket, sink)} & & \\ \hline
\end{tabular}
}
\end{table*}

%% file: floats/figures/dropout_ablation.tex
\begin{figure*}[h]
    \centering
    \includegraphics[width=0.55\textwidth]{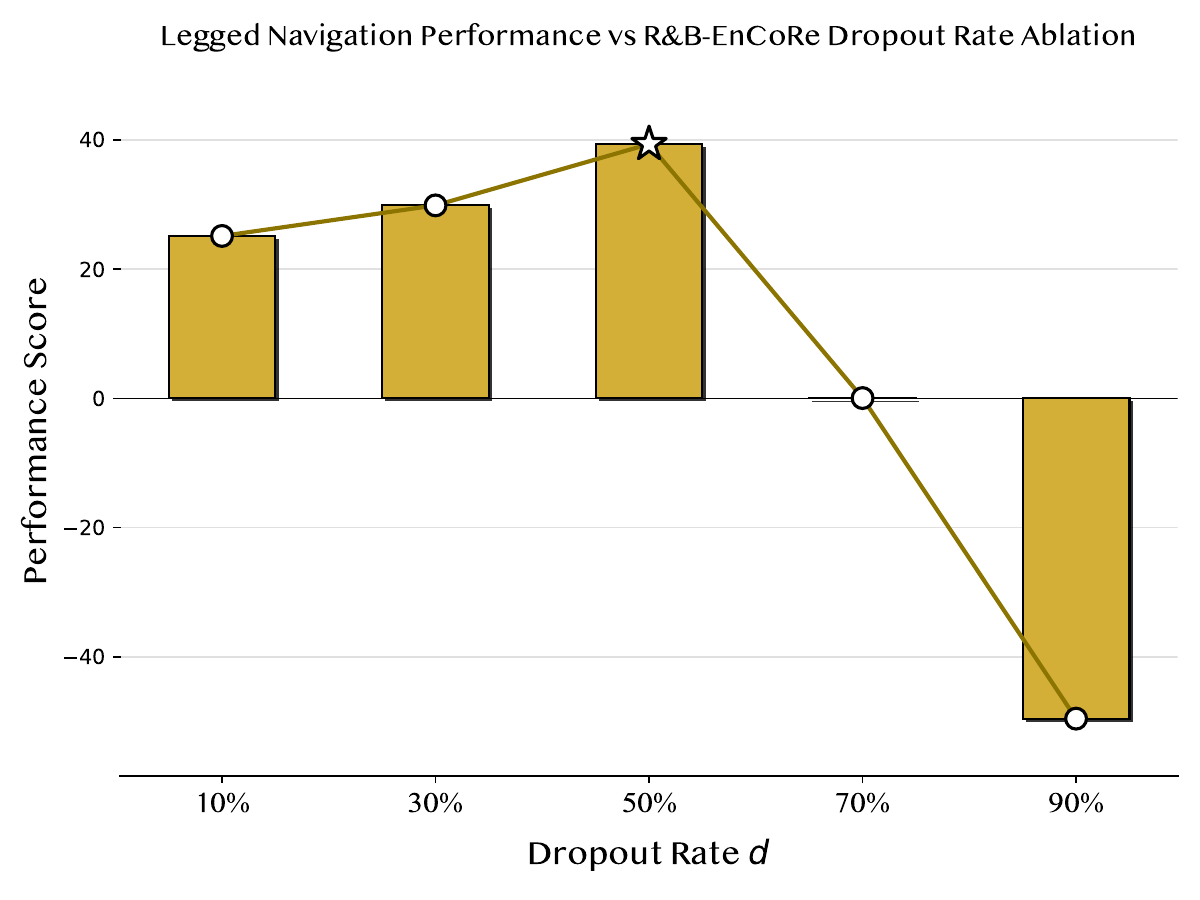}
    \caption{\small For our \algname algorithm applied to the Legged navigation embodiments, we perform an ablation study on varying the dropout rate parameter $d$ that affects the initial warmstart reasoning strategy training distribution. We find that $50\%$ dropout provides best downstream performance. This dropout rate encourages the prior and posterior model to see a diverse set of reasoning strategies with overall minimal warmstarting preference bias for any strategy (\ie $\log\frac{1-d}{d}=\log\frac{1-0.5}{0.5}=0$). }
    \label{fig:dropout_ablation}
\end{figure*}

%% file: floats/figures/legged_sampling_bound.tex
\begin{figure*}[h]
    \centering
    \includegraphics[width=0.6\textwidth]{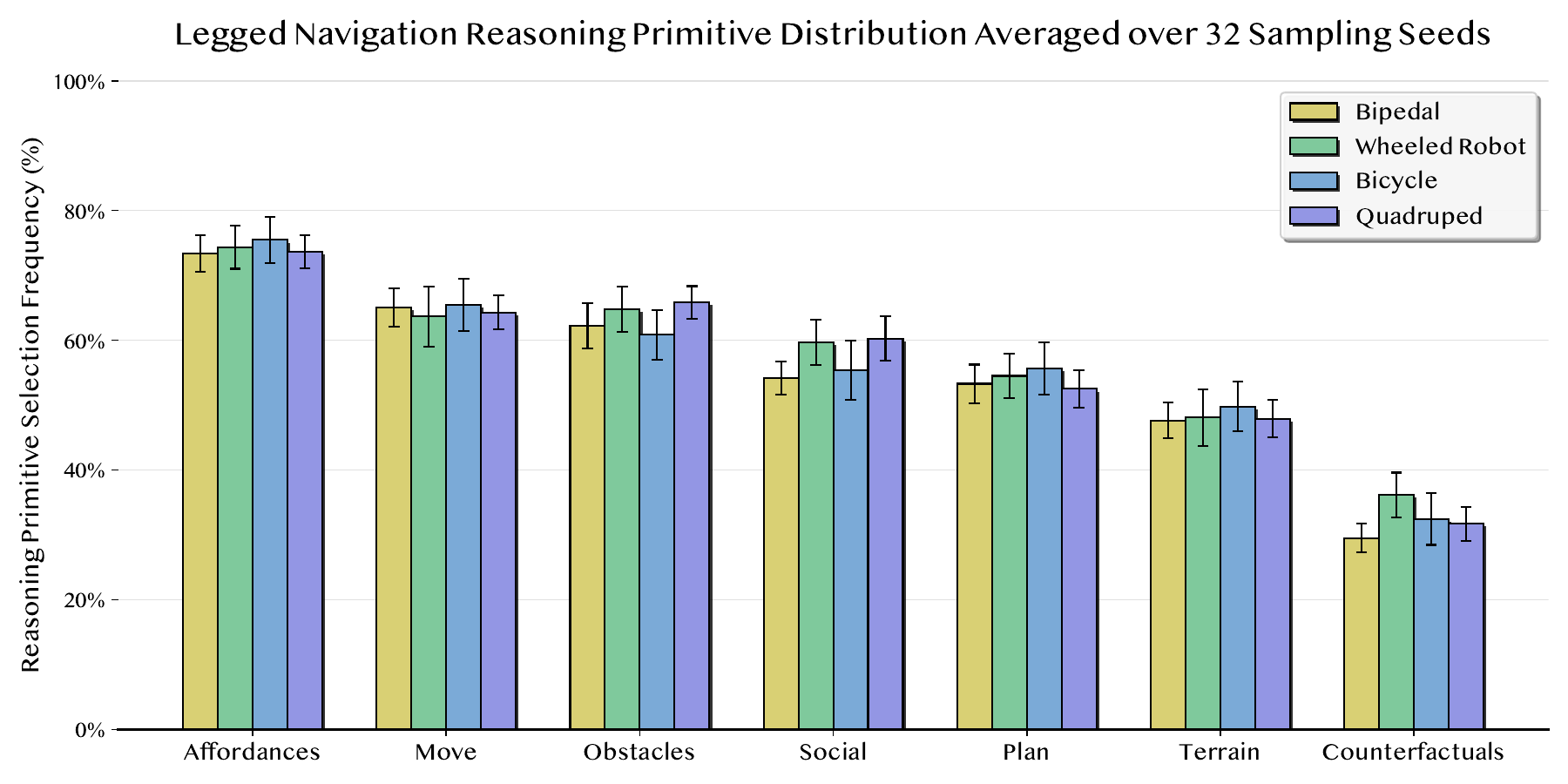}
    \caption{\small For the Legged Navigation Dataset we perform an ablation on performing posterior sampling (from Alg.~\ref{alg:rbencore_refine}) across 32 different sampling seed to validate whether the refined reasoning primitive distribution remains consistent. This plot confirms the generally consistency of the reasoning primitive frequencies (note the error bars and compare with the result of a single sample seed of Fig.~\ref{fig:mainresult}b).}
    \label{fig:sampling_ablation}
\end{figure*}

%% file: floats/figures/navitrace_wgarbage_perf.tex
\begin{figure}[h]
    \centering
    \includegraphics[width=0.7\linewidth]{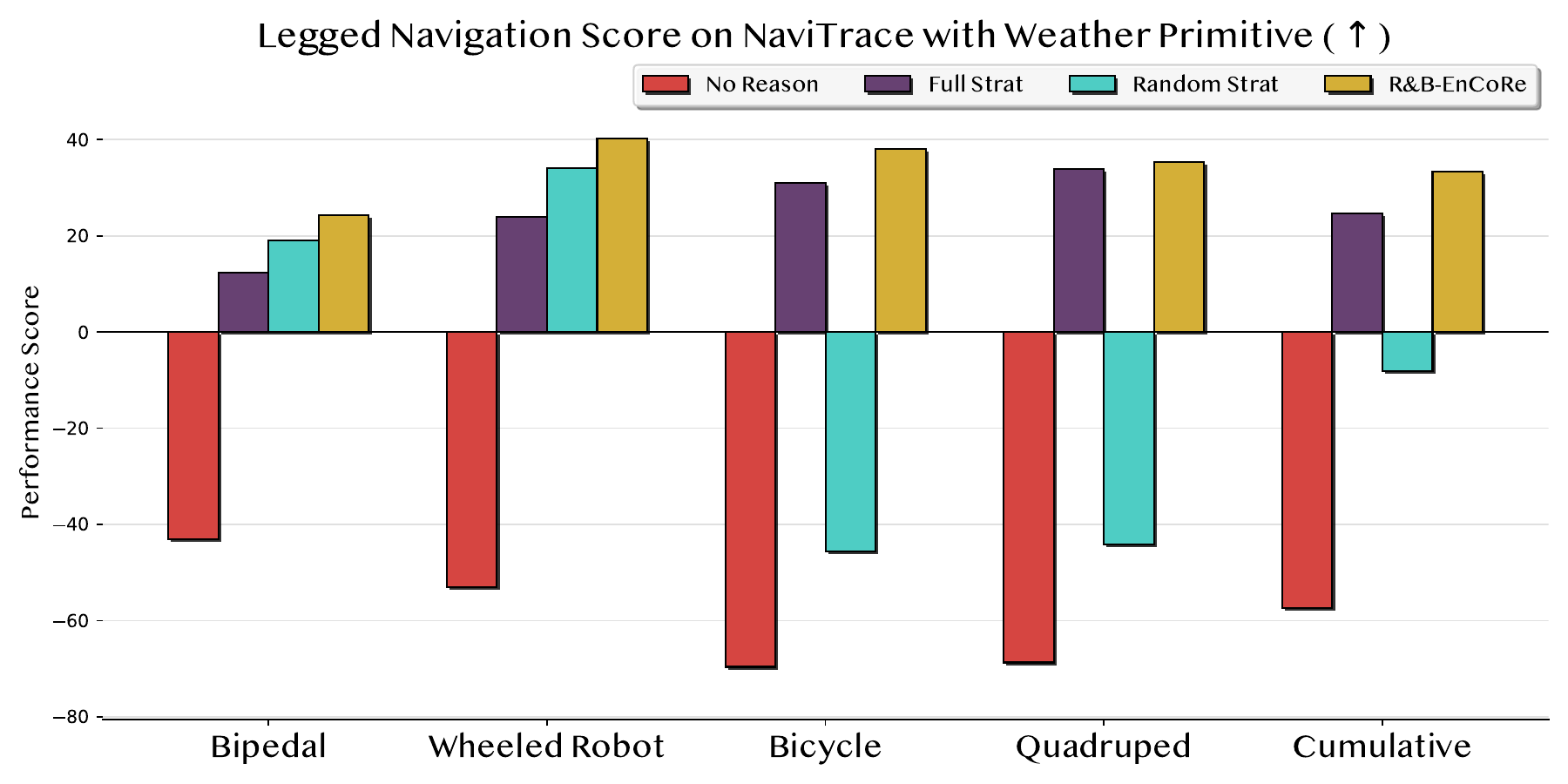}
    \caption{\small NaviTrace scores on the various VLA models with the additional Weather Reasoning primitive. \algname refines the traces to remove irrelevant Weather reasoning primitive scores as seen in Fig.~\ref{fig:garbage_distrib}, and results in best performance across all embodiments on the navigation metric from ~\cite{windecker2025navitrace}. Recalls the NaviTrace score is a normalized metric so 100 is perfect path alignment with expert, and 0 is the score for the na\"ive straight line path down the middle of the scene.}
    \label{fig:navitrace_wgarbage_score}
\end{figure}

%% file: appendix/qualitative_results.tex
%\section{Detailed Hardware Results, Example Reasoning Traces, Additional Ablation Experiments}

\newpage
\textcolor{white}{ }
\newpage
\section{Example Reasoning Traces Across Embodiments}
\label{sec:supplementary_qualitative}

\input{floats/figures/libero_visible_object_qualitative}

%\subsection{LIBERO Visible Objects}
%Fig.~\ref{fig:visobj-time-qual}

\input{floats/figures/libero_all_subtypes_qualitative}
%\subsection{LIBERO Qualitative Results}
%Fig.~\ref{fig:libero-all-subtypes-qual1}, %Fig.~\ref{fig:libero-all-subtypes-qual2}

\input{floats/figures/bridge_qualitative}
%\subsection{Bridge-v2 Qual}
%Fig.~\ref{fig:bridge-all-subtypes-qual1}, %Fig.~\ref{fig:bridge-all-subtypes-qual2}

\input{floats/figures/quadruped_reason}

\input{floats/figures/av_qualitative_all}

%% file: floats/figures/libero_visible_object_qualitative.tex
\begin{figure*}[h] % 'p' places the figure on its own page
    \centering
    
    % First figure
    \includegraphics[width=\textwidth]{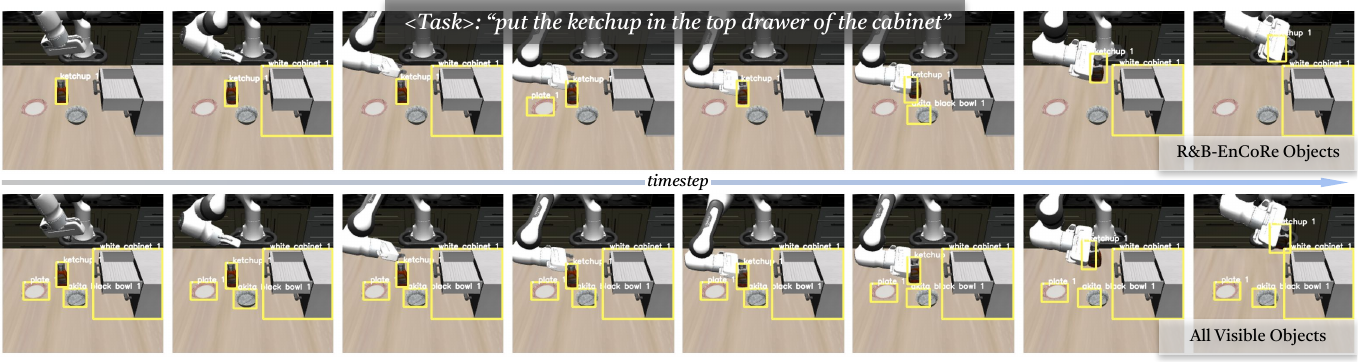}
    
    \vspace{12pt} % Add some vertical space between figures
    
    % Second figure
    \includegraphics[width=\textwidth]{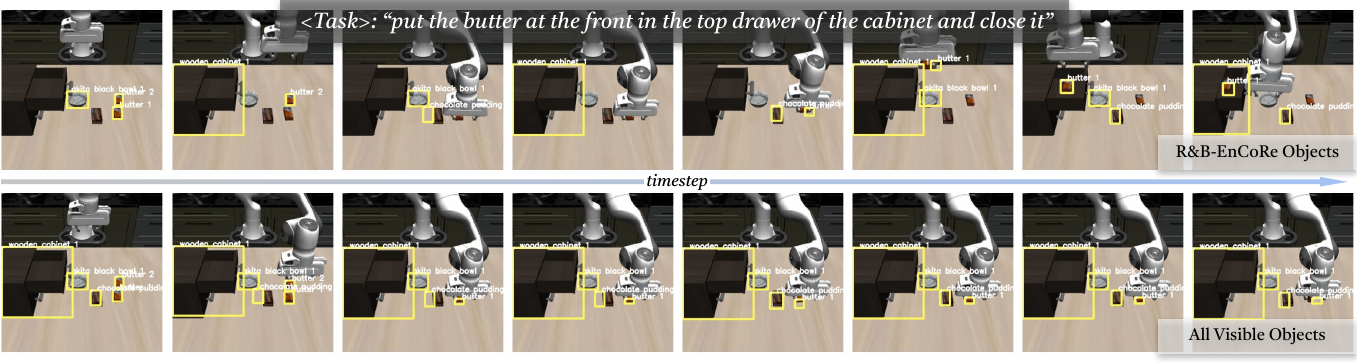}
    
    \caption{\small Visible Object-only reasoning (Section~\ref{sec:Libero90Results}-Q1) in LIBERO-90 across steps in episode. Notice that the generated object bounding boxes for \algname's reasoning model generally attend to primarily task salient objects, while reasoning with all visible objects attends to all objects (including distracting/irrelevant ones) at every frame.}
    \label{fig:visobj-time-qual}
\end{figure*}

%% file: floats/figures/libero_all_subtypes_qualitative.tex
\begin{figure*}[h]
    \centering
        \includegraphics[width=\textwidth]{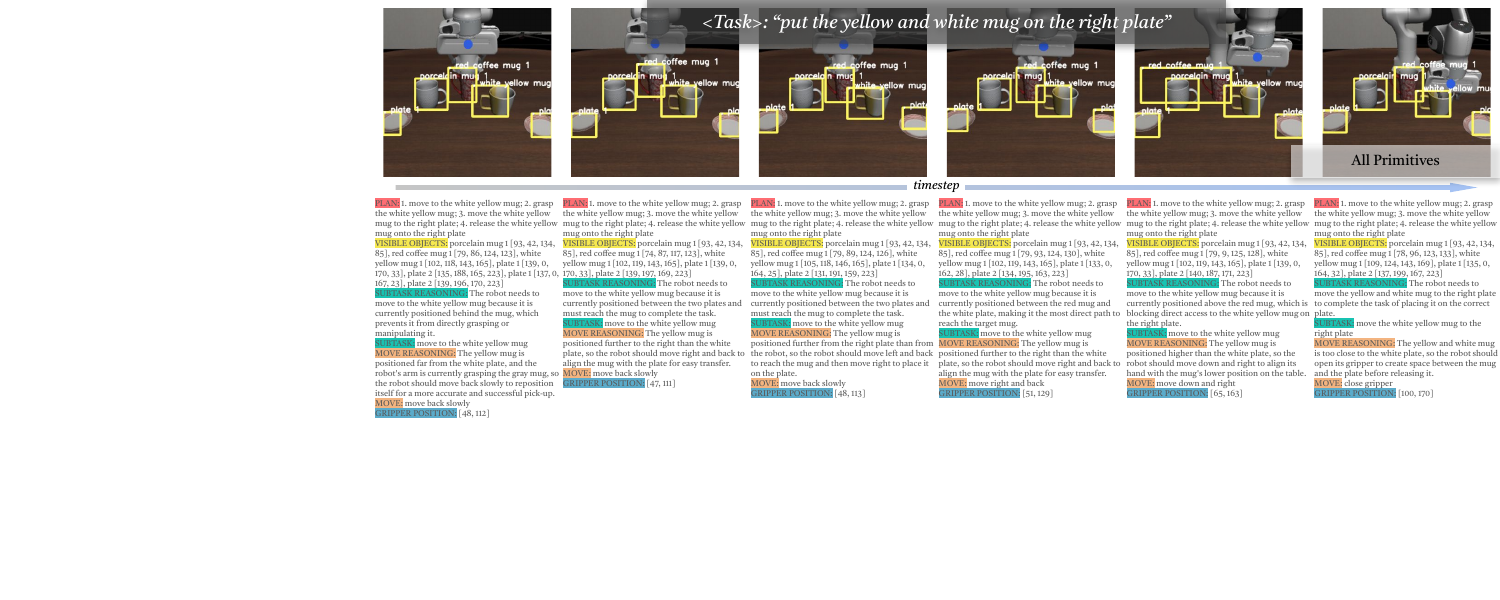}
        
        % \vspace{4pt}
        
        \includegraphics[width=\textwidth]{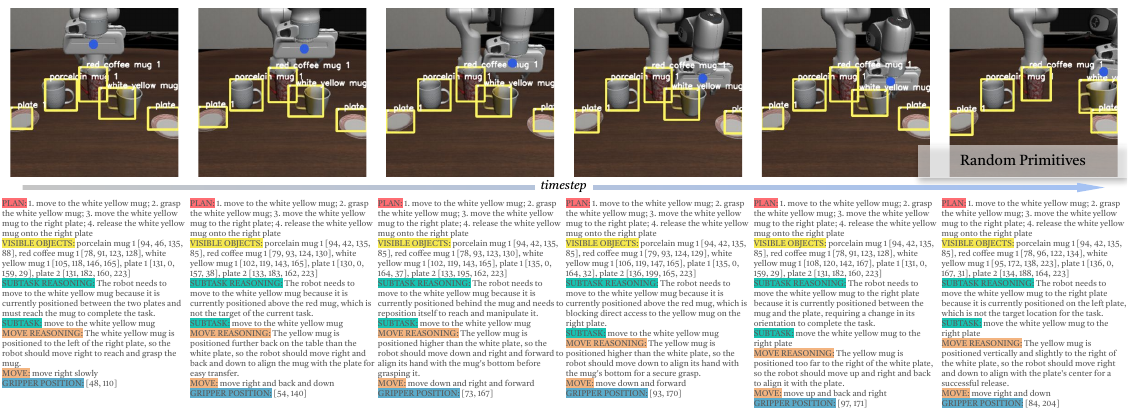}
        
        % \vspace{4pt}
        
        \includegraphics[width=\textwidth]{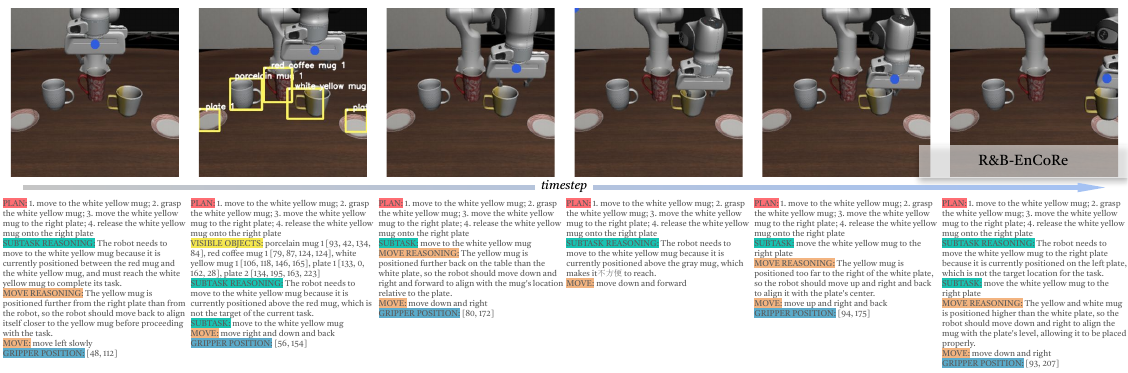}
        
        \caption{\small Reasoning Traces with wider set of primitives (Section~\ref{sec:Libero90Results}-Q2) from the reasoning VLAs for LIBERO-90 across episode. Notice how \algname reasons less frequently about visible objects compared with the other two models.}
    % Second subfigure with 3 stacked figures
    \label{fig:libero-all-subtypes-qual1}
\end{figure*}

\begin{figure*}[h]
    \centering
    \includegraphics[width=\textwidth]{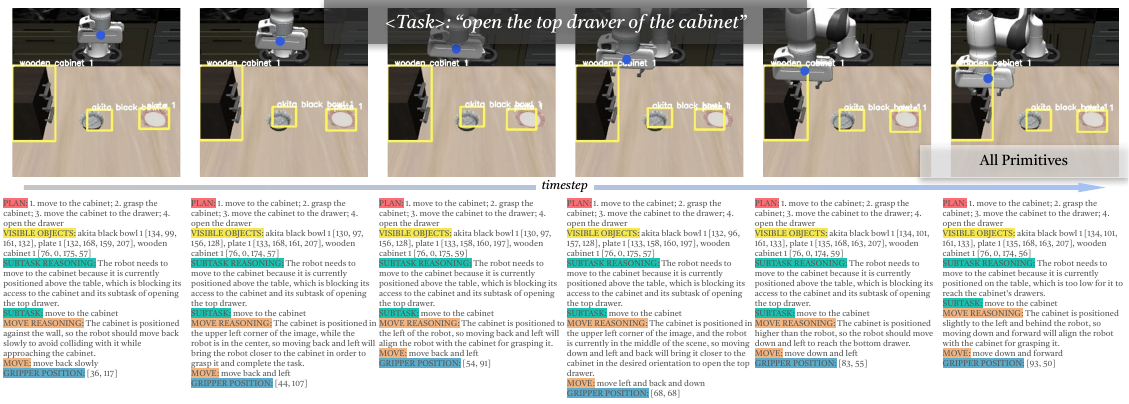}
    
    % \vspace{4pt}
    
    \includegraphics[width=\textwidth]{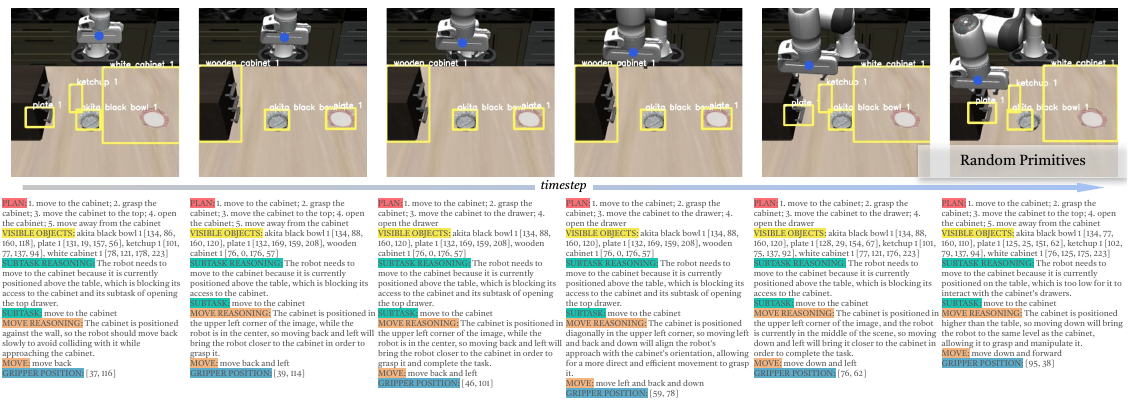}
    
    % \vspace{4pt}
    
    \includegraphics[width=\textwidth]{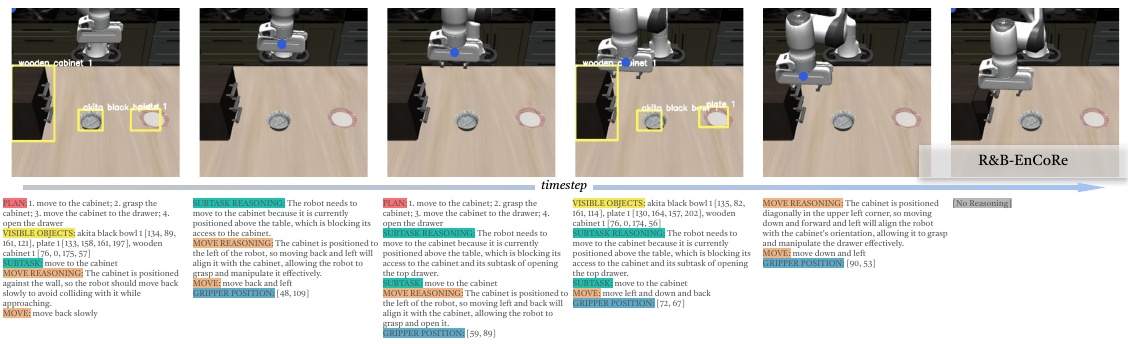}
    
    \caption{\small Reasoning Traces with wider set of primitives (Section~\ref{sec:Libero90Results}-Q2) from the reasoning VLAs for LIBERO-90 across episode. Notice how \algname reasons less frequently about visible objects compared with the other two models.}
    \label{fig:libero-all-subtypes-qual2}
\end{figure*}

%% file: floats/figures/bridge_qualitative.tex
\begin{figure*}[t]
    \centering
        \includegraphics[width=\textwidth]{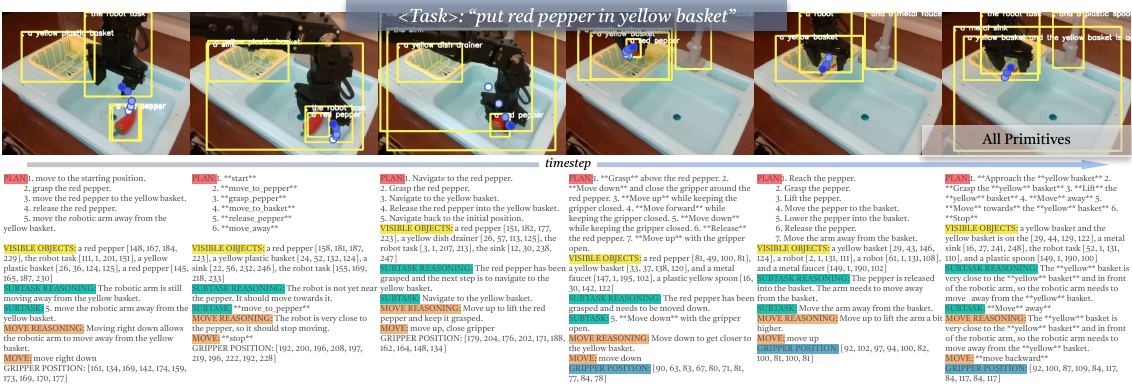}

        \includegraphics[width=\textwidth]{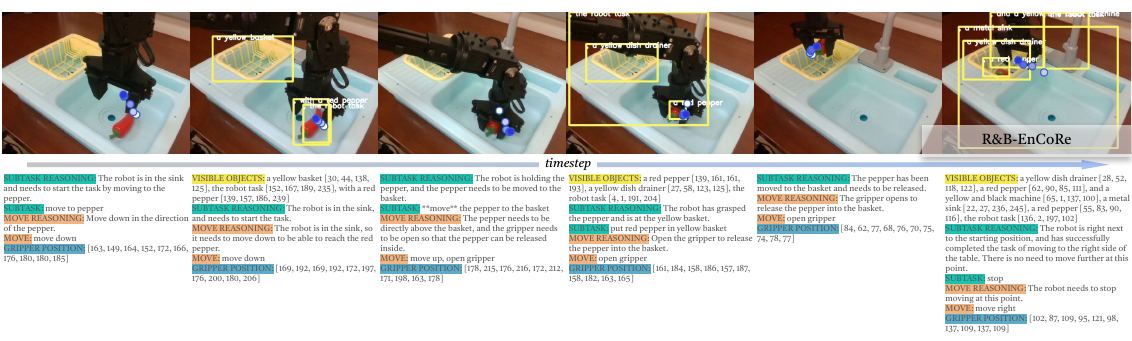}
        
        \caption{\small Reasoning Traces (Section~\ref{sec:BridgeResults}-Q2) from the reasoning VLAs in Bridgev2 setup on WidowX hardware with Test-Time Reasoning enabled across episode.}
    % Second subfigure with 3 stacked figures
    \label{fig:bridge-all-subtypes-qual1}
\end{figure*}

\begin{figure*}[t]
    \centering
    \includegraphics[width=\textwidth]{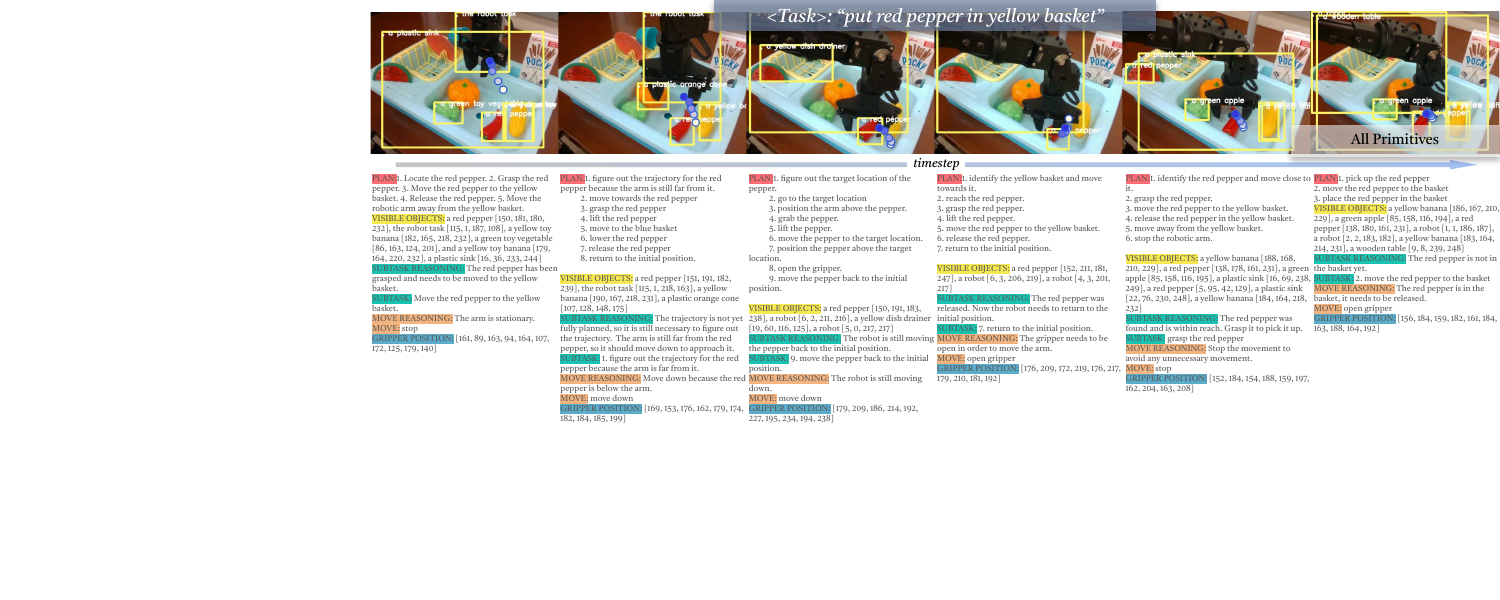}
    
    % \vspace{4pt}
    
    \includegraphics[width=\textwidth]{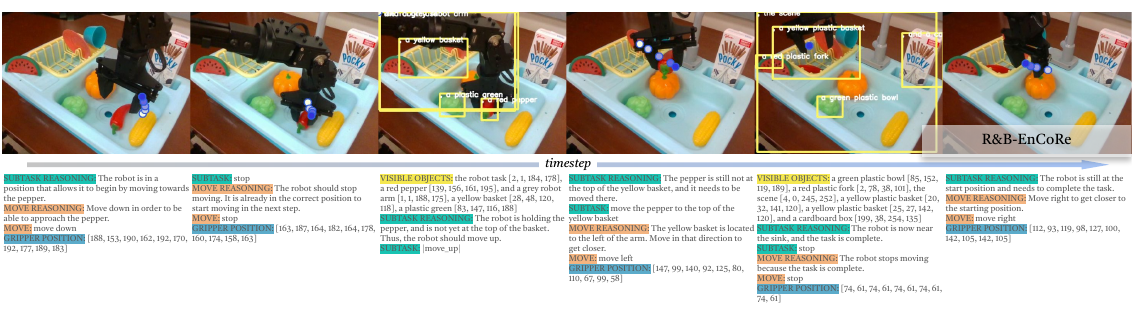}
    
    \caption{\small Reasoning Traces (Section~\ref{sec:BridgeResults}-Q2) from the reasoning VLAs in Bridgev2 setup on WidowX hardware with Test-Time Reasoning enabled across episode.}
    \label{fig:bridge-all-subtypes-qual2}
\end{figure*}

%% file: floats/figures/quadruped_reason.tex
\begin{figure*}[ht]
    \centering
    \includegraphics[width=\textwidth]{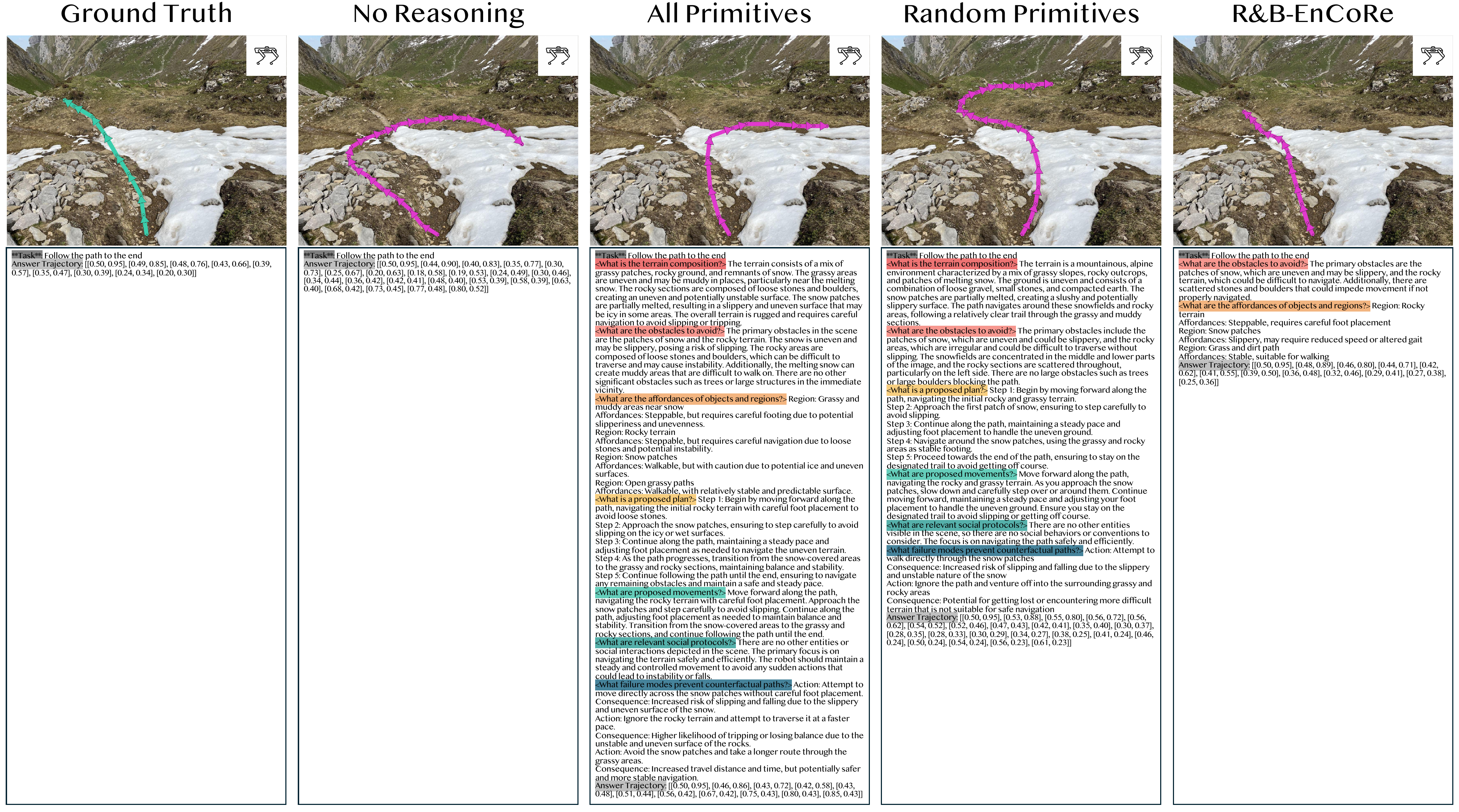}
    \caption{\small Reasoning Traces for NaviTrace dataset with Quadruped embodiment (expanded version of Fig.~\ref{fig:navitrace_waypoints}).}
    \label{fig:quadruped_reason}
\end{figure*}

%% file: floats/figures/av_qualitative_all.tex
\begin{figure*}[t] % 'p' places the figure on its own page
    \centering
    
    % First figure
    \includegraphics[width=\textwidth]{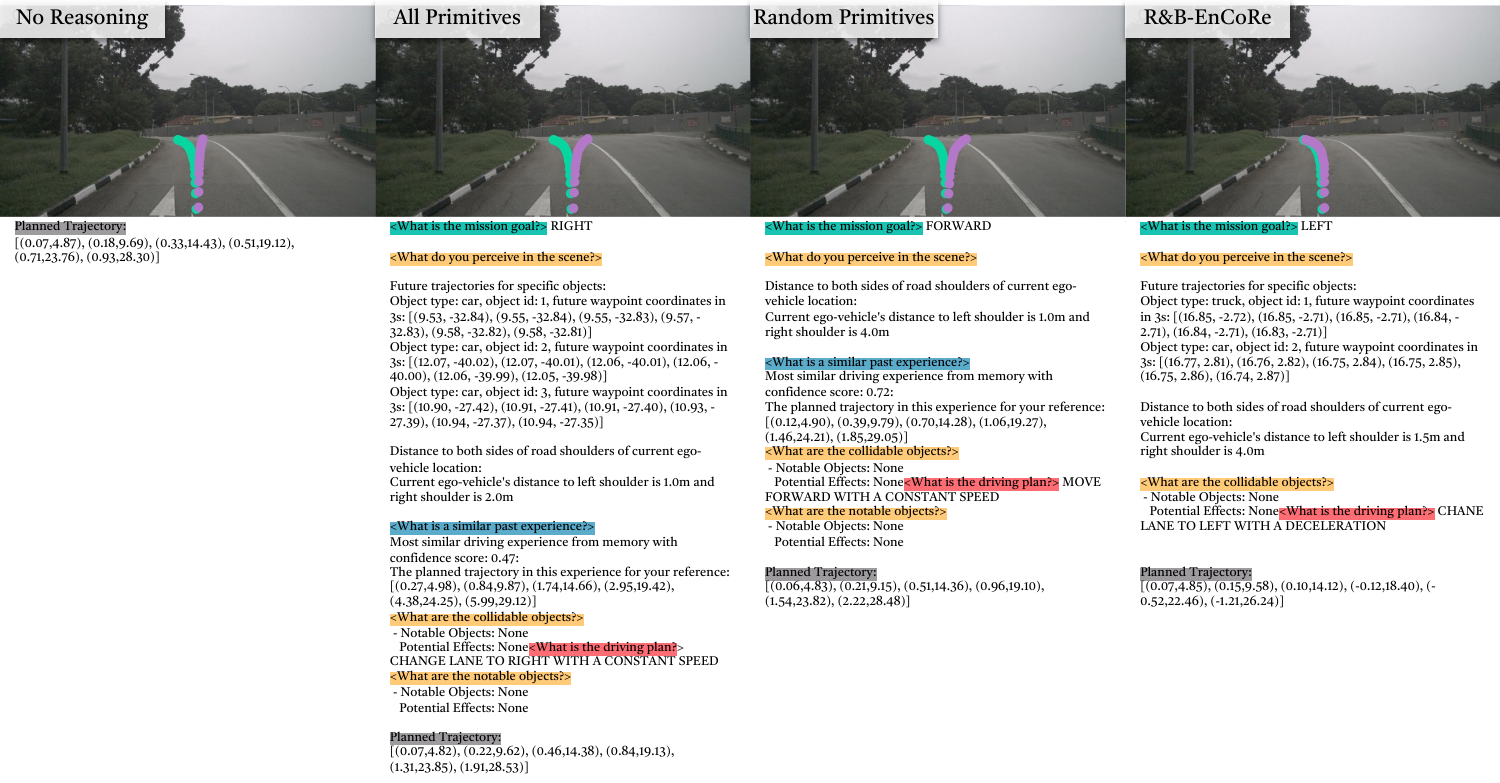}
    
    \vspace{2pt} % Add some vertical space between figures
    
    % Second figure
    \includegraphics[width=\textwidth]{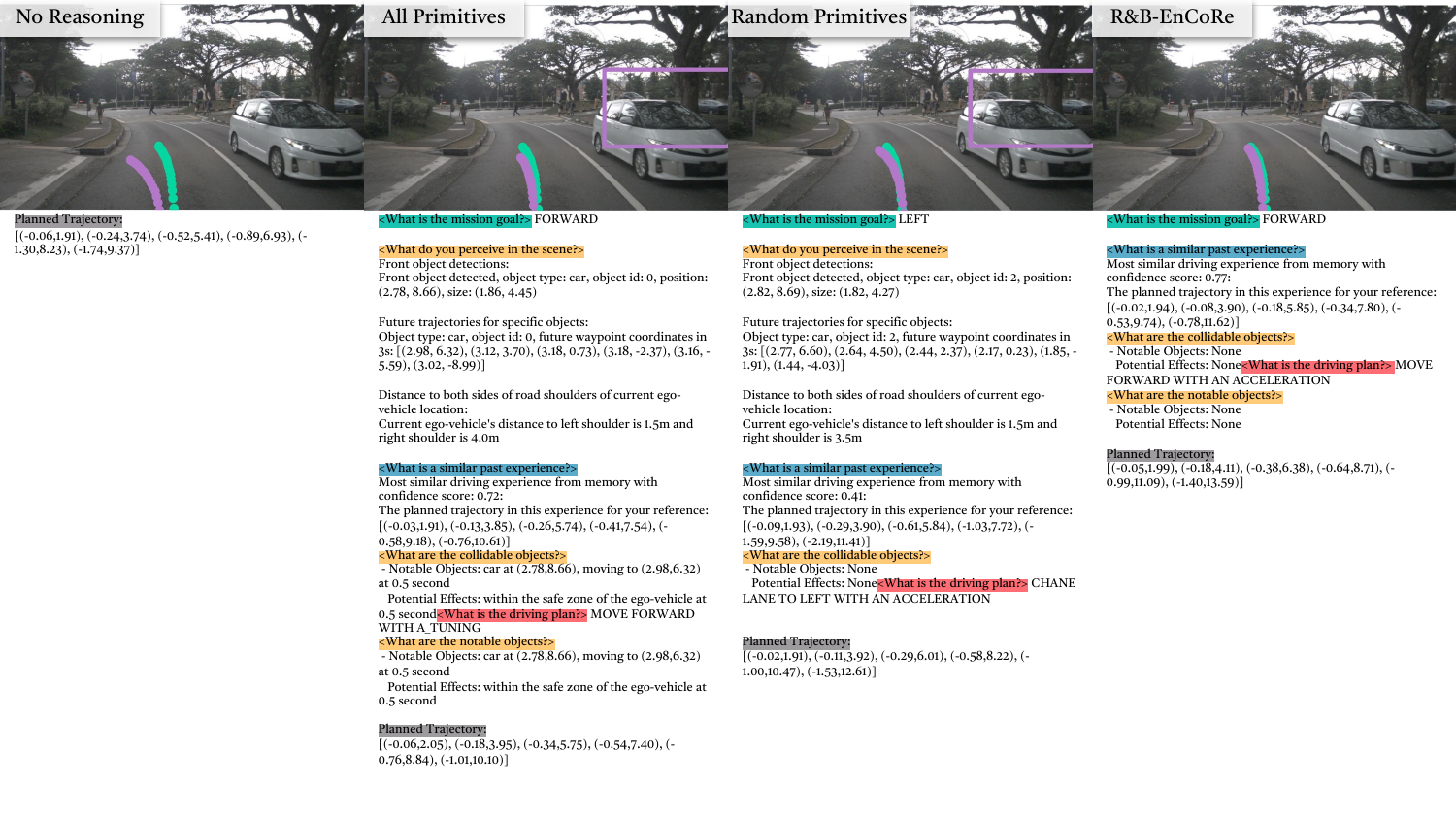}
    
    \caption{\small Reasoning Traces (Section~\ref{sec:AVResults}-Q7) from the driving VLAs. We visualize predictions across models on two samples from the nuScenes dataset. Observe that using \algname{} improves performance and yields concise reasoning traces that is more informative than not reasoning at all. Reasoning types are colored for visualization purposes.}
    \label{fig:av-all-qual}
\end{figure*}

\begin{figure*}[t] % 'p' places the figure on its own page
    \centering
    
    % First figure
    \includegraphics[width=\textwidth]{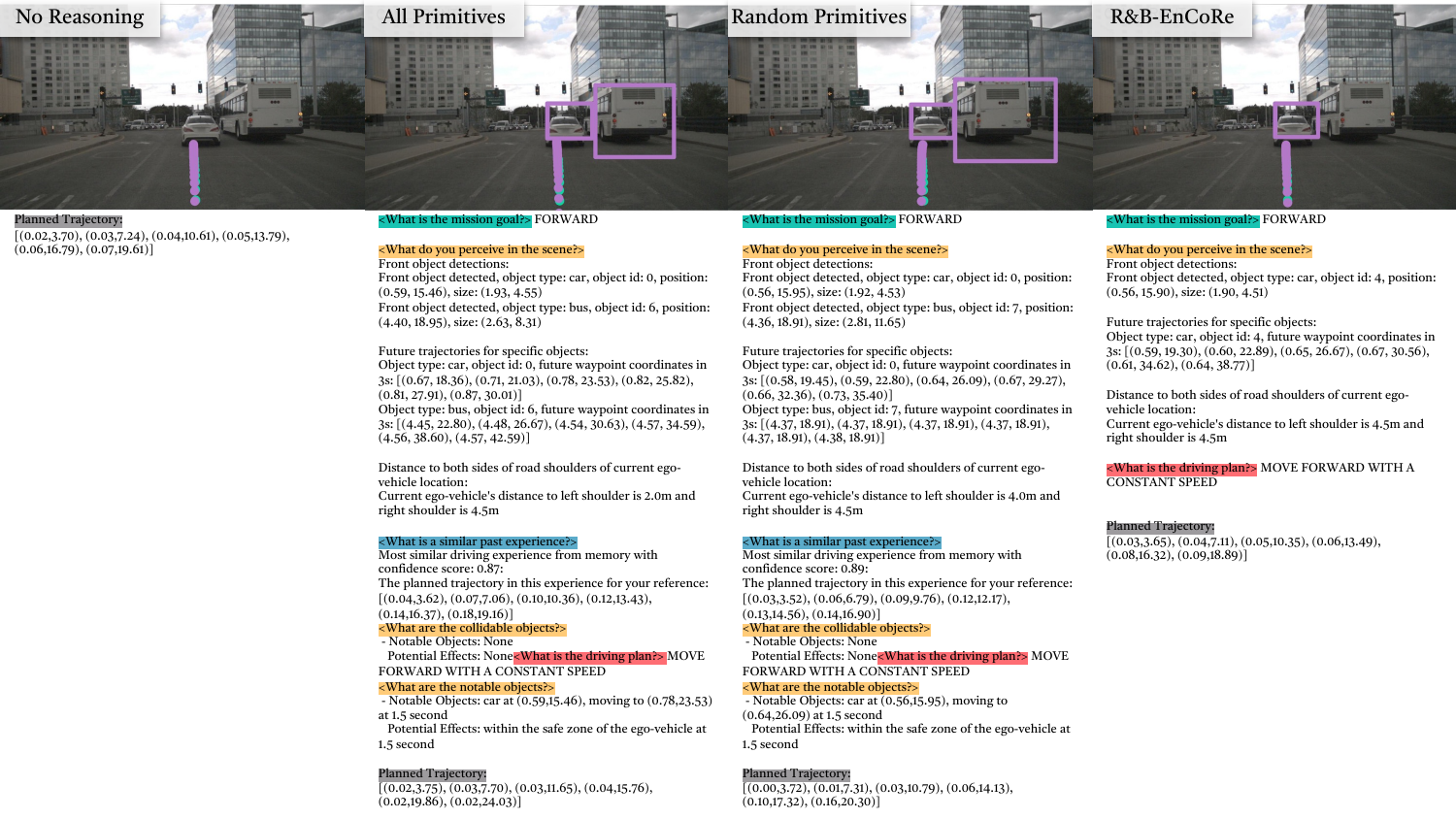}
    
    \vspace{2pt} % Add some vertical space between figures
    
    % Second figure
    \includegraphics[width=\textwidth]{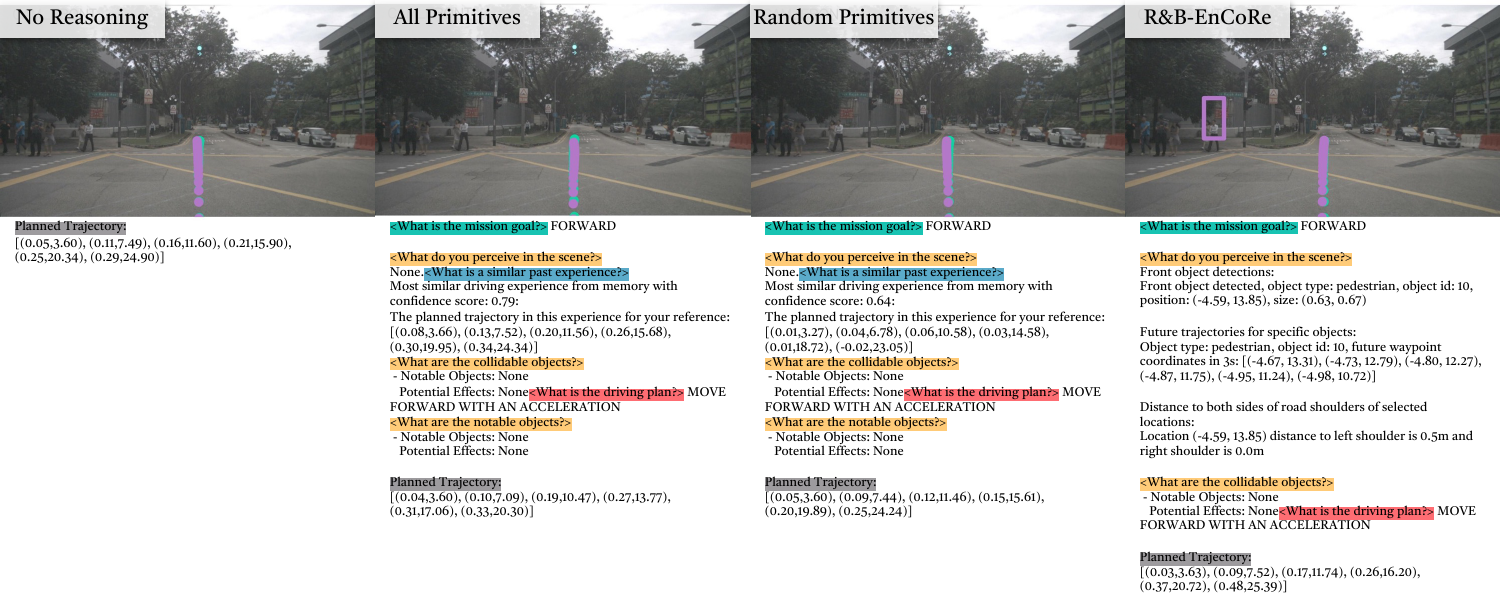}
    
    \caption{\small Reasoning Traces (Section~\ref{sec:AVResults}-Q7) from the driving VLAs. We visualize results on two more samples from the nuScenes dataset.}
    \label{fig:av-all-qual2}
\end{figure*}

%% file: appendix/more_ablation_experiments.tex
%\section{Additional Ablation Studies and analyses}